\documentclass{article}
\PassOptionsToPackage{numbers, compress}{natbib}

\usepackage[main, final]{neurips_2026}              

\usepackage[utf8]{inputenc}
\usepackage[T1]{fontenc}
\usepackage{times}
\usepackage{microtype}

\usepackage{amsmath, amssymb, amsfonts, amsthm, mathtools}
\usepackage{bm}
\usepackage{nicefrac}

\usepackage{booktabs}
\usepackage{graphicx}
\usepackage{subcaption}
\usepackage{wrapfig}
\usepackage{caption}

\usepackage{algorithm}
\usepackage{algpseudocode}

\usepackage{enumitem}
\usepackage{xcolor}
\usepackage{soul}
\usepackage{hyperref}
\usepackage{url}
\usepackage{titletoc}

\usepackage{tikz}

\usepackage{amsmath,amsfonts,bm}

\def\1{\bm{1}}

\DeclareMathAlphabet{\mathsfit}{\encodingdefault}{\sfdefault}{m}{sl}
\SetMathAlphabet{\mathsfit}{bold}{\encodingdefault}{\sfdefault}{bx}{n}

\newcommand{\E}{\mathbb{E}}

\newcommand{\R}{\mathbb{R}}

\newcommand{\Cov}{\mathrm{Cov}}

\newcommand{\calF}{\mathcal{F}}

\newcommand{\calE}{\mathcal{E}}
\newcommand{\calM}{\mathcal{M}}

\graphicspath{{./images/}}

\newcommand{\calG}{\mathcal{K}}
\newcommand{\calH}{\mathcal{H}}
\newcommand{\calX}{\mathcal{X}}
\newcommand{\xx}{\mathbf{x}}

\def\subbundleidx{d}
\def\ttD{\mathsf{D}}
\def\ttE{\mathsf{E}}

\newtheoremstyle{tightplain}%
  {1.5pt}{1.5pt}{\itshape}{}{\bfseries}{.}{0.5em}{}

\theoremstyle{tightplain}
\newtheorem{theorem}{Theorem}
\newtheorem{corollary}{Corollary}[theorem]
\newtheorem{lemma}{Lemma}
\newtheorem{proposition}{Proposition}
\newtheorem{definition}{Definition}
\newtheorem{remark}{Remark}
\newtheorem{example}{Example}

\newtheorem*{theorem*}{Theorem}
\newtheorem*{corollary*}{Corollary}
\newtheorem*{lemma*}{Lemma}
\newtheorem*{proposition*}{Proposition}
\newtheorem*{definition*}{Definition}

\DeclareRobustCommand{\cameraready}[1]{#1}
\renewcommand{\acksection}{\section*{\cameraready{Acknowledgments and Disclosure of Funding}}}
\newcommand{\coderepo}{\url{https://github.com/clabat9/hilbnet}} 

\title{\Large Consistent Geometric Deep Learning \\
  via Hilbert Bundles and Cellular Sheaves}

\author{
  Kartik Tandon$^{1,*,\dagger}$ \quad
  Julian Gould$^{2,*}$ \quad
  Tanishq Bhatia$^{3}$ \\[0.3em]
  \textbf{Francesca Dominici}$^{4}$ \quad
  \textbf{Alejandro Ribeiro}$^{1}$ \quad
  \textbf{Claudio Battiloro}$^{5,4,*,\dagger}$ \\[0.5em]
  {\small
    $^{1}$University of Pennsylvania 
    $^{2}$Sakana AI 
    $^{3}$Northeastern University 
    $^{4}$Harvard University
    $^{5}$Brown University} \\[0.3em]
  {\small $^{*}$Equal contribution} \\[0.1em]
  {\small $^{\dagger}$Corresponding authors:
    \texttt{ktandon@sas.upenn.edu, claudio\_battiloro@brown.edu}}
}

\begin{document}

\maketitle

\begin{abstract}
Modern deep learning architectures increasingly contend with sophisticated signals that are natively infinite-dimensional, such as time series, probability distributions, or operators, and are defined over irregular domains. Yet, a unified learning theory for these settings has been lacking. To start addressing this gap, we introduce a novel convolutional learning framework for possibly infinite-dimensional signals supported on a manifold. Namely, we use the connection Laplacian associated with a Hilbert bundle as a convolutional operator, and we derive filters and neural networks, dubbed as \textit{HilbNets}. We make HilbNets and, more generally, the convolution operation, implementable via a two-stage sampling procedure. First, we show that sampling the manifold induces a Hilbert Cellular Sheaf, a generalized graph structure with Hilbert feature spaces and edge-wise coupling rules, and we prove that its sheaf Laplacian converges in probability to the underlying connection Laplacian as the sampling density increases. Notably, this result is a generalization to the infinite-dimensional bundle setting of the Belkin \& Niyogi \cite{BELKIN20081289} convergence result for the graph Laplacian to the manifold Laplacian, a theoretical cornerstone of geometric learning methods. Second, we discretize the signals and prove that the discretized (implementable) HilbNets converge to the underlying continuous architectures and are transferable across different samplings of the same bundle, providing consistency for learning. Finally, we validate our framework on synthetic and real-world tasks. Overall, our results broaden the scope of geometric learning as a whole by lifting classical Laplacian-based frameworks to settings where the signal at each point lives in its own Hilbert space.
\end{abstract}

\vspace{-.2cm}\section{Introduction}\vspace{-.2cm}
Over the past few years, advances in deep learning have delivered state-of-the-art performance across many areas, driven by increasingly expressive architectures and corresponding gains in both theory and practice. A major contributor to this success, though not the only one, has been the rise of Convolutional Neural Networks (CNNs) \cite{lecun1998gradient}. CNNs have shown outstanding results in settings ranging from image recognition \cite{alexnet2012} to speech processing \cite{hamid2012cnnspeech}. At their core, CNNs rely on filters leveraging the regular (often metric) organization of common signal types, such as spatial grids. In contrast, many modern datasets live on irregular, non-Euclidean domains, including social networks for detection and recommendation \cite{aggarwal2020machine} or point clouds for shape segmentation \cite{xie2020linking}, to name only a few. Such structured data can be represented by richer mathematical objects, among which networks and manifolds are prominent. Motivated by this, the intuition behind CNNs has been generalized to graph convolutional neural networks (GCNs) \cite{scarselli2008graph, gama2018convolutional, kipf2016semi} and extended to many other settings,~e.g.~simplicial complexes \cite{battiloro2022san,bodnar2021weisfeiler,barbarossa2020topological}, cell complexes \cite{battiloro2022can,bodnarcwnet, papillon2025topotune}, order lattices \cite{riessmultidimensional}, and manifolds \cite{wang2021stability, cohen2019gauge, schonsheck2018parallel, battiloro2023tangent, cohen2019general}. Nevertheless, existing works do not address convolutional filtering of infinite-dimensional signals on manifolds, despite such data being ubiquitous in practice, from time series and spatiotemporal fields arising in sensing, robotics, and climate science to distributional and measure-valued representations common in modern learning systems \cite{hu2024amortizing}.

 To address this gap, we adopt a bundle viewpoint. Informally, a bundle $\mathcal{E}$ over a base manifold $\mathcal{M}$ is a consistent assignment to each point $x\in \mathcal{M}$ of a space $\mathcal{E}_x$, called a fiber.  A \emph{section} is a map that picks an element $S(x)\in \mathcal{E}_x$ at every point. In other words, signals supported on manifolds can be seen as sections, e.g., scalar manifold signals correspond to $\mathcal{E}_x\simeq\mathbb{R}$ \cite{wang2021stability}, or tangent bundle signals correspond to $\mathcal{E}_x\simeq T_x\mathcal M$ \cite{battiloro2023tangent}. In this work, we develop a convolutional learning framework operating over \textit{Hilbert bundles}, i.e., bundles whose fibers are (possibly) infinite-dimensional Hilbert spaces.  We design a bundle-theoretic convolutional learning framework and, to make it implementable, we draw the first rigorous connection between Hilbert bundles and \textit{Hilbert Cellular Sheaves},  generalized graph structures whose nodes and edges carry infinite-dimensional signals along with consistency rules.

\begin{figure}[t]
  \centering
   \includegraphics[width=1.0\linewidth]{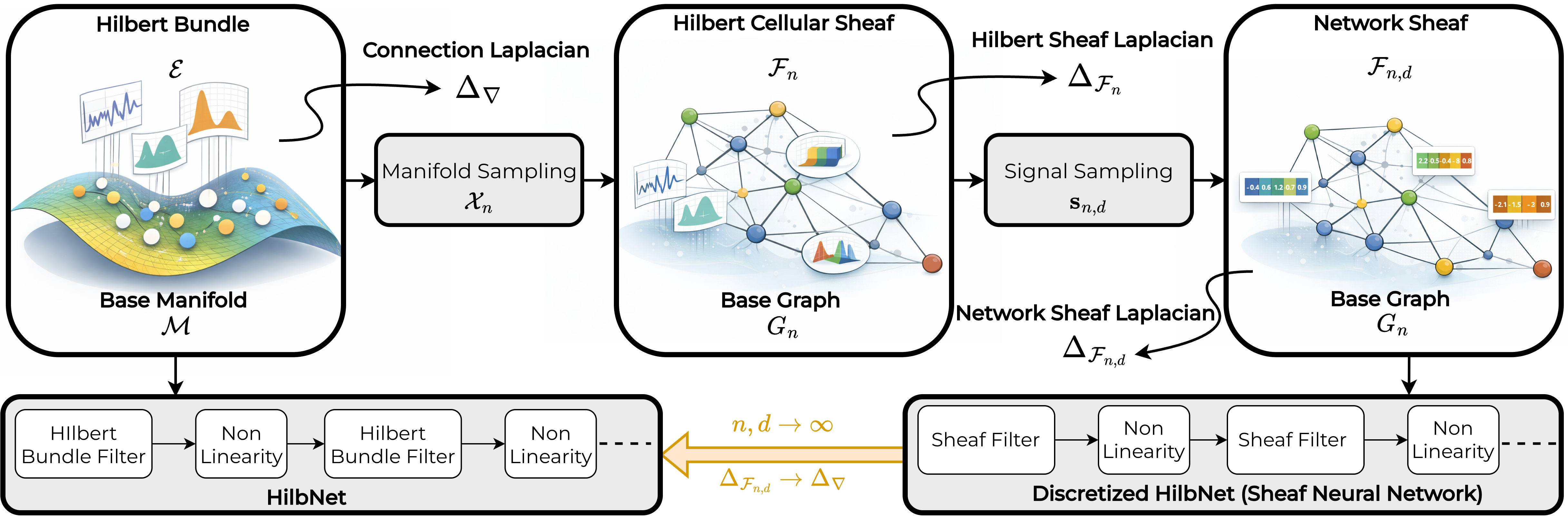}
    \caption{\textbf{Overview of the HilbNets framework.} A HilbNet is a convolutional neural network processing infinite-dimensional signals supported on $\mathcal{M}$ (e.g., time-series or distributions over curved domains). The convolutional operator is the connection Laplacian $\Delta_\nabla$. To make HilbNets implementable, we first sample $n$ points $\mathcal{X}_n$ from $\mathcal{M}$ to obtain a Hilbert Cellular Sheaf $\mathcal{F}_n$ with associated Hilbert Sheaf Laplacian $\Delta_{\mathcal{F}_n}$. We then take $d$ samples of the signals to get vectors $\mathbf{s}_{n,d} \in \mathbb{R}^{nd}$ living on a network sheaf $\mathcal{F}_{\mathcal{X}_{n,d}}$,  i.e., a generalized matrix-weighted graph, with associated network Sheaf Laplacian $\Delta_{\mathcal{F}_{n,d}}\in \mathbb{R}^{nd \times nd}$. Discretized HilbNets are then Sheaf Neural Networks that take as inputs $\mathbf{S}_{n,d}$ and $\Delta_{\mathcal{F}_{n,d}}$. We prove that Discretized HilbNets converge to the underlying HilbNet as the number of manifold and signal samples goes to infinity.}
  \label{fig:figure_abstract}
\end{figure}

\noindent\textbf{Related Works.} 
The connection between continuous domains, such as manifolds and bundles, and discrete structures, such as graphs and cellular sheaves, first emerged in pioneering investigations on the so-called manifold hypothesis. This hypothesis posits that, although data may live in a high-dimensional ambient space, they are effectively generated by sampling from one or several low-dimensional Riemannian manifolds \cite{fefferman2016testing}. The manifold hypothesis underpins several modern spectral graph methods, e.g., nonlinear dimensionality-reduction, clustering, interpretability, and learning algorithms that exploit latent geometric structures. The renowned work of Belkin and Niyogi \cite{BELKIN20081289} proved that, assuming access to a finite point cloud sampled from the underlying manifold, it is possible to build a weighted undirected graph whose Laplacian converges to the Laplace-Beltrami operator of the underlying manifold in probability as the number of samples goes to infinity.

The work in \cite{BELKIN20081289}, and related results, e.g., \cite{singer2012vdm,singer2017spectral}, have been used, directly or indirectly, to design  learning systems over manifolds and networks \cite{wang2022convolution,levie2019transferability, borovitskiy2020matern,mostowsky2024geometrickernels, battiloro2023tangent, borovitskiy2021vector}. Despite the diversity of such systems, these models all assume finite-dimensional fibers and therefore do not directly address learning with infinite-dimensional manifold signals. The main technical reason behind this gap is the lack of an extension of the convergence result in \cite{BELKIN20081289} to bundles with infinite-dimensional fibers.

A  line of related works of interest comes from cellular sheaf theory. Cellular sheaves are combinatorial instances of sheaves introduced in \cite{shepard1985cellular} and later rediscovered in \cite{curry2014sheaves}.  In \cite{bodnar2022sheafdiff,hansen2020sheafnn,barbero2022sheafnnconn,fiorini2025sheaves,duta2023sheaf,peng2026sheaf}, neural networks operating on finite-dimensional cellular sheaves over graphs, referred to as \textit{network sheaves}, are presented, generalizing graph neural networks by, intuitively, replacing scalar edge weights with learned or structured matrix weights. Recently, the works in \cite{battiloro2022tangent,battiloro2023tangent} showed that neural networks for tangent bundle signals can be implemented as certain sheaf neural networks operating on network sheaves built from manifold samples. For an extended treatment of related work, see Appendix~\ref{appsec:relwork}.

\noindent\textbf{Contribution.} Our contributions are as follows:
\cameraready{%
\begin{description}[leftmargin=*]
\item[C1.] We first define a \textbf{\textit{convolution operation over a Hilbert bundle}} through its associated \textit{connection Laplacian}. This convolution extends Laplacian-based convolutions on tangent bundles \cite{battiloro2023tangent}, manifolds \cite{wang2022convolution}, and graphs \cite{shuman2013emerging,gama2020gcnmagazine}, as well as standard time convolutions. Using the Borel functional calculus, we then define \textbf{\textit{Hilbert bundle convolutional filters}} for infinite-dimensional manifold signals. These filters are general and expressive, and can be instantiated through suitable spectral responses.
\item[C2] We then introduce \textbf{\textit{HilbNets}}, deep convolutional architectures whose layers stack Hilbert bundle filters and pointwise nonlinearities. HilbNets are continuous models and are therefore not directly implementable. To address this, we provide a principled discretization of the manifold domain by sampling points and showing that the induced structure is a \textit{Hilbert cellular sheaf} over an undirected graph. The corresponding sheaf Laplacian combines scalar edge weights, obtained from the sampled base manifold, with parallel transport maps associated with the bundle geometry or learned from data. We prove that this sheaf Laplacian converges in probability to the connection Laplacian as the sampling density increases, yielding the first extension of the classical convergence result of \cite{BELKIN20081289} to the infinite-dimensional bundle setting. We then discretize the signals themselves to obtain an implementable architecture, show that discretized HilbNets are novel instances of network sheaf neural networks, and prove that they converge to the corresponding continuous architectures as both the manifold and signal sampling densities increase. Moreover, we show that discretized HilbNets are \textit{transferable} across different samplings of the same underlying bundle, providing resolution consistency guarantees for learning.
\item[C3] Finally, we show the effectiveness of HilbNets on a synthetic transport recovery task and on real-world traffic forecasting tasks, comparing them against baselines with different inductive biases in order to isolate the benefits of the bundle formulation.
\end{description}}%
The potential impact of this work extends well beyond  the HilbNet architecture. See Appendix~\ref{app:broader_impact} for a detailed discussion of broader impact and future directions, and Fig.~\ref{fig:figure_abstract} for an overview.

\vspace{-.2cm}\section{Preliminaries}\vspace{-.2cm}\label{sec:preliminaries}

\textbf{Signals on Manifolds.} Given a manifold $\mathcal{M}$, a vector-valued signal is a square-integrable function $S \in L^2(\mathcal{M},\mathbb{R}^n)$. Certain vector-valued signals on $\mathcal{M}$ may possess the richer structure of a vector field, i.e., they are \textit{sections} of the tangent bundle $T\mathcal{M}$ of $\mathcal{M}$ and thus elements of $L^2(\mathcal{M},T\mathcal{M})$. More generally, we may consider signals that are $L^2$-sections of an arbitrary \textit{bundle} $\mathcal{E}$. A bundle is called trivial when, for a generic fiber $\mathcal{V}$ , it can be written as a product $\mathcal{E}=\mathcal{M}\times \mathcal{V}$. In this setting, $L^2(\mathcal{M},\mathbb{R}^n)$ may be understood as the space of sections of the \textit{trivial bundle} $\mathcal{E} = \mathcal{M} \times \mathbb{R}^n$.
\begin{wrapfigure}{r}{0.45\linewidth} \includegraphics[width=.85\linewidth]{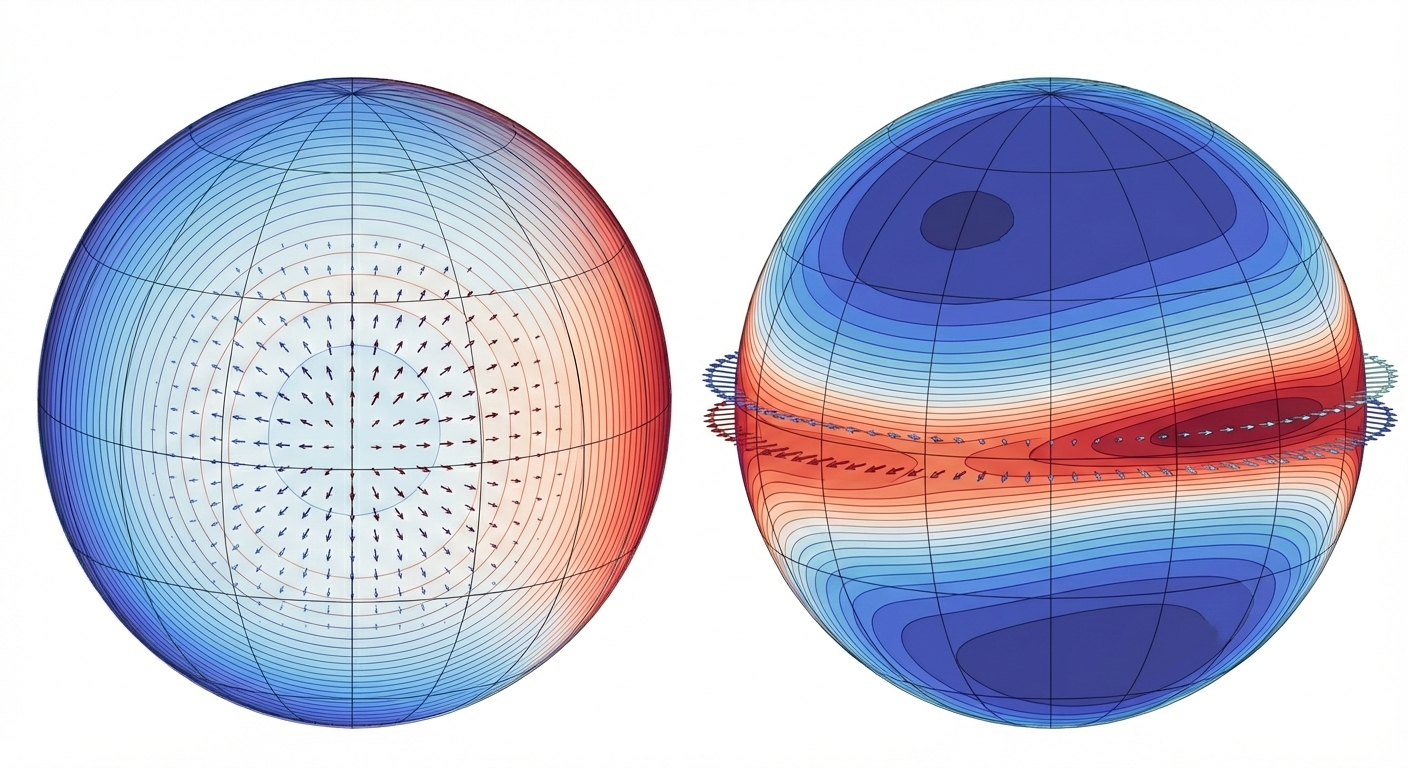}
 \centering
 \caption{Visualization of the effect of the choice of underlying connection for generating heat flows of vector fields on the sphere $\mathbb{S}^2$. Left: Generated by considering the standard Levi-Civita connection on $T\mathbb{S}^2$, and corresponding Laplacian $\Delta_\nabla$. Right: Generated by allowing a more general connection $\nabla$ that allows for torsion anisotropy and considering the corresponding connection Laplacian $\Delta_\nabla$.}\label{fig:connection}\vspace{-.1cm}
\end{wrapfigure}
Consider now the case where the signal is `infinite-dimensional', for instance, representing a time series recorded at each point $x \in \mathcal{M}$. While this is usually considered as a function $S: \mathcal{M} \times \mathbb{R} \to \mathbb{R}^n$, it may instead be more richly understood as a section of a $\textit{Hilbert bundle}$, i.e., a bundle whose fibers are Hilbert spaces. As we will see, Hilbert bundles provide a principled and versatile approach to incorporating structural properties of infinite-dimensional data. 
\begin{example} \label{example:chern-simons}
    In physics, Hilbert bundles often arise naturally when considering global geometric properties of quantum mechanical systems \cite{AxelrodDellaPietraWitten1991}.
\end{example}
\begin{example} \label{example:info-geo-bundle}
    In information geometry, the key objects of study are manifolds $\mathcal{M}$ given by the underlying parameters of some family of data distributions. This manifold is then equipped with a Riemannian structure by either the Otto-Wasserstein or Fisher-Rao metric, the latter of which locally recovers KL divergence. The proper analogue of the tangent bundle in this setting is  a Hilbert bundle  \cite{Malago2018}. 
\end{example}
\textbf{Convolution, Heat Equation, and Connection Laplacian.} 
Geometric signal processing and deep learning \cite{leus2023graph, bronstein2021geometric} traditionally aim to develop convolutional filters and neural networks designed to respect the underlying geometry of the signals of interest. The relevant convolutional operators can usually be realized as a \textit{connection Laplacian} $\Delta_\nabla: L^2(\mathcal{M},\mathcal{E}) \to L^2(\mathcal{M},\mathcal{E})$ operator realized from a connection $\nabla$. For instance, for the tangent bundle over the circle $T\mathbb{S}^1$, the eigenfunctions of $\Delta_\nabla$, with $\nabla$ the Levi-Civita connection, recover the usual Fourier basis. Similarly, the eigenfunctions of $\Delta_\nabla$ for the tangent bundle of the sphere $T\mathbb{S}^2$ recover spherical harmonics. Thus, convolutions with the connection Laplacian may be understood as generalized Fourier transforms in the spectral domain. In the spatial domain, it can be seen as performing a geometry-aware ‘local averaging’ of a signal over fibers. Formally, the connection Laplacian is the generator of the heat equation in $L^2(\mathcal{M},\mathcal{E})$, 
\begin{equation}
    \frac{\partial \mathbf{U}(x,t)}{\partial t} = -\Delta_\nabla \mathbf{U}(x,t), 
\end{equation} 
where $\mathbf{U}(x,t)$ is the distribution of heat at $\mathcal{E}_x$ for $x\in \mathcal{M}$ at time $t \in \mathbb{R}_+$.
A subtlety of note is that for non-Euclidean spaces, there is typically no canonical identification between fibers $\mathcal{E}_x$ and  $\mathcal{E}_y$. Intuitively, the connection $\nabla$ precisely encodes a globally coherent notion of transport between fibers. That is, inducing \textit{parallel transport} maps $P_{\gamma}: \mathcal{E}_{\gamma(0)} \to \mathcal{E}_{\gamma(1)}$ for a path $\gamma \subset \mathcal{M}$, allowing us to compare elements across fibers along this path. More formally, the connection is used to define a first-order ODE whose solution is given by parallel transport (see Appendix \ref{sec:connections}). The connection Laplacian is then the self-adjoint operator $\Delta_\nabla:=\nabla^*\nabla$, now more clearly understandable as a `local weighted average’ over fibers with `weights’ corresponding to our choice of parallel transport. A more rigorous introduction to  
the relevant mathematical background is provided in Appendix \ref{sec:mathematical-background}.  
\begin{example}
      As seen in Fig. \ref{fig:connection}, the choice of connection can be used to emphasize aspects of the geometry that may be relevant for a particular task. For instance, most PDE-based approaches to color-image regularization can be realized as heat equations for a suitable choice of connection \cite{Batard2011HeatVectorBundles}.
\end{example}
  
\vspace{-.2cm}\section{Hilbert Bundle Filters and Neural Networks}\vspace{-.2cm}

In this section, we develop a convolutional learning framework for infinite-dimensional data, such as time series or probability distributions, indexed by a manifold $\mathcal{M}$. Core objects  are \textit{Hilbert bundles}. 

\textbf{Hilbert Bundles.} Given a closed Riemannian manifold $\mathcal{M}$, a Hilbert bundle $\mathcal{E}$ over $\mathcal{M}$ is a bundle whose potentially infinite-dimensional fibers are separable Hilbert spaces over $\mathbb{R}$. The assumption of real, instead of complex, Hilbert spaces is not essential to our analysis, and is made only for the sake of exposition. As mentioned in Section~\ref{sec:preliminaries}, a \textit{Hilbert Bundle signal} is then an $L^2$-section $S: \mathcal{M} \to \mathcal{E}$. Integration of sections in this setting should be understood in the Bochner integral sense, a generalized notion of integration for functions whose values lie in a Hilbert space rather than in $\mathbb{R}^n$. In finite dimensions, it reduces to the standard Lebesgue integral. Given fibers $\mathcal{H}_x$ and $\mathcal{H}_y$ of the Hilbert bundle $\mathcal{E}$, we consider \textit{unitary} parallel transport operators $P_{x\to y}:\mathcal{H}_x \to \mathcal{H}_y$. As before, a globally compatible collection of such transport operators determines a connection $\nabla$, with the subtlety that derivatives of sections must now be understood in the Fréchet sense. Intuitively, Fréchet differentiability is the infinite-dimensional analogue of ordinary differentiability: it asks that a section admit a best linear approximation under small perturbations, but where the linear approximation acts between Hilbert spaces. We therefore refer to this construction as a \textit{Fréchet connection}, which recovers the usual notion of connection and covariant derivative when restricted to finite-dimensional bundles. As before, we obtain a self-adjoint operator $\Delta_\nabla := \nabla^*\nabla$ on $L^2(\mathcal{M},\mathcal{E})$. Unlike the finite-dimensional case, however, this operator need not be compact and thus need not possess a discrete spectrum. As such, care must be taken when adapting classical arguments that involve spectral properties of the Laplacian to the Hilbert-bundle setting. Formal definitions of Hilbert bundles and Fréchet connections are provided in Appendix~\ref{sec:mathematical-background}. For a triple $(\mathcal{M}, \mathcal{E}, \nabla)$, where $\nabla$ is a choice of Fréchet connection on the Hilbert bundle $\mathcal{E}$ over the manifold $\mathcal{M}$, we now wish to construct a general notion of a `filtering' operation using the connection Laplacian $\Delta_\nabla$. In finite-dimensional or compact settings, filters are often defined by applying a function directly to the eigenvalues of the Laplacian. In our setting, the appropriate analogue of eigenvalue-by-eigenvalue filtering is furnished by the \textit{Borel functional calculus}, which allows us to apply a filter to a self-adjoint operator by instead integrating over its spectral measure. See Appendix~\ref{sec:borel-calc} for details.
\begin{definition}[Hilbert bundle convolutional filter]
A \textbf{convolutional filter} is specified by a bounded compactly supported Borel function $g\in L_c^\infty(\mathbb{R})$. The filtering of a signal $S\in L^2(\mathcal{M},\mathcal{E})$ is then its convolution $\star_{\Delta_\nabla}$ with $g$ defined as $
        g \star_{\Delta_\nabla} S
        :=
        g(\Delta_\nabla)S, \textrm{ where } g(\Delta_\nabla):
        L^2(\mathcal{M},\mathcal{E})
        \to
        L^2(\mathcal{M},\mathcal{E})$
    is the bounded linear operator obtained by applying $g$ to $\Delta_\nabla$ through the Borel functional calculus.
\end{definition}
 In this sense, $g$ is the learnable frequency response, as in spectral graph neural filters \cite{gama2020gcnmagazine}, except we now use the spectral measure of the connection Laplacian acting on Hilbert bundle-valued signals.
\begin{definition}[Hilbert bundle convolutional neural network]\label{HilbNet-defn}
 Let  $(\mathcal{M},\mathcal{E},\nabla)$ be a Hilbert bundle. A Hilbert bundle convolutional neural network, or \textbf{HilbNet}, is specified by a filter bank $\mathcal{W}=\{g^{\ell}_{u,q}\}_{\ell,u,q}$ with $g^{\ell}_{u, q}\in L_c^\infty(\mathbb{R})$, and a Lipschitz continuous nonlinear activation $\sigma:\mathbb{R}\to\mathbb{R}$. Given input signals $S_1,..., S_{F_0}\in L^2(\mathcal{M},\mathcal{E})$, the $L$-layer network output is obtained by the recursion
    \begin{equation}
        S^{\ell+1}_u
        =
        \sigma\left(\sum_{q=1}^{F_\ell}
        g^\ell_{u,q}(\Delta_\nabla)S^\ell_q
        \right),
        \qquad
        \ell=0,\dots,L-1, \qquad  S^0_q=S_q,
    \end{equation}
where $\sigma$ is applied pointwise in each fiber. 
\end{definition}
We concisely denote a HilbNet with $\Omega(\mathcal{E},\Delta_\nabla,\mathcal{W},\sigma)$. Similarly to the finite-dimensional case, a nonlinear activation $\sigma: \mathbb{R} \to \mathbb{R}$ with $\sigma(0) = 0$ extends to an operator on the $L^2$ section by simply picking a basis and then applying $\sigma$ to each coordinate with respect to the chosen basis. It is straightforward to check that, for each $\ell \in \{0, \ldots, L\}$, the layer signal $S^\ell$ remains an $L^2$ section.

\vspace{-.2cm}\section{Discretized HilbNets via Cellular Sheaves}\vspace{-.2cm}

HilbNets are continuous architectures that cannot be implemented directly in practice. Moreover, we typically do not have access to the true bundle and connection structure, but only to a point cloud or graph sampled from the underlying manifold $\mathcal{M}$, together with samples of the signal at each point or node. In this section, we first analyze the \textit{Hilbert cellular sheaf} induced by spatial, i.e., manifold-level, sampling. We then further discretize the fibers, i.e., the signal domain itself, obtaining a finite rank \textit{network sheaf}. This two-stage sampling is the basis of our consistency theory presented in Section \ref{sec:convergence}. It allows us to prove that the fully discrete (thus, implementable) Laplacian and HilbNet converge to their infinite-dimensional counterparts, and hence that learning is consistent across scales.

\textbf{Manifold Sampling.} 
A generalized viewpoint to the theory of bundles is given by the language of \textit{sheaves}, mathematical structures initially introduced by Jean Leray while a prisoner of war \cite{Leray1946}. The \textit{functoriality} of sheaves lends them particularly well to the type of principled discretization of geometric structures that we are interested in. In particular, we consider a Hilbert-space valued version of \textit{cellular sheaves} on graphs, as introduced in \cite{gould25}. Intuitively, they can be understood as generalized graph structures with signals valued in Hilbert spaces along with edge-wise coupling rules. For a more thorough introduction to cellular sheaves, see Appendix~\ref{sec:cellular-sheaves}.

In this work, our primary interest is in Hilbert sheaves that represent discretizations of the structure of a Hilbert bundle over a manifold. In particular, we desire a spatial discretization such that we can recover an appropriate discrete analogueof the Hilbert bundle's connection Laplacian. Formally, given an iid random sample $\mathcal{X}_n \subset \mathcal{M}$ from the uniform distribution (see Def. \ref{def:uniform-distribution}), we have the following.

\begin{definition}[Hilbert Cellular Sheaf from a Hilbert Bundle]
\label{def:sheaf-from-bundle}
For a given Hilbert bundle $(\mathcal{M},\mathcal{E},\nabla)$ with sampled points $\mathcal{X}_n = \{x_1, \dots, x_n \} \subset \mathcal{M}$, fix a geodesic $\gamma_{ij}$ between $x_i$ and $x_j$, for all $i < j$. Further, let $m_{\gamma_{ij}}$ denote the midpoint of this geodesic. Consider the graph $G_{n} = (\mathcal{X}_n,E)$ with an undirected edge $e_{ij}$ between $x_i$ and $x_j$, for each $i < j$. The associated \textbf{Hilbert cellular sheaf} $\mathcal{F}_n^{t}$ on $G_n$ with bandwidth parameter $t$ is given by the following assignments:
    \begin{itemize}[leftmargin=*]
    \item The Hilbert space $\mathcal{F}_n^t(x_i) := \mathcal{E}_{x_i}$ for each $x_i \in \mathcal{X}_n$, referred to as the node stalk over $x_i \in \mathcal{X}_n$.
    \item The Hilbert space $\mathcal{F}_n^t(e_{ij}) := \mathcal{E}_{m_{\gamma_{ij}}}$ for each $e_{ij} \in E$, referred to as the edge stalk over $e_{i,j} \in E$.
    \item For each edge $e_{ij} \in E$ with bounding vertices $x_i,x_j$, a pair of bounded linear restriction maps
    \begin{align}
    &(\mathcal{F}_n^t)_{x_i \leq e_{ij}}
    :=
    \sqrt{k_{ij}^t} \,
    P_{x_i \to m_{\gamma_{ij}}}:
    \mathcal{F}_n^t(x_i)
    \to
    \mathcal{F}_n^t(e_{ij}), \nonumber \\
    &(\mathcal{F}_n^t)_{x_j \leq e_{ij}}
    :=
    \sqrt{k_{ij}^t} \,
    P_{x_j \to m_{\gamma_{ij}}}:
    \mathcal{F}_n^t(x_j)
    \to
    \mathcal{F}_n^t(e_{ij}),
    \end{align}
    where $k_{ij}^t = e^{-d_{\calM}(x_i,x_j)^2/4t}$, with $d_{\calM}$ the geodesic distance on $\calM$, and $P_{x_i \to m_{\gamma_{ij}}}$ denotes the unitary parallel transport map on $\mathcal{E}$ between $x_i$ and $m_{\gamma_{ij}}$. 
    \end{itemize}
\end{definition}
For the sake of exposition, the choice of sample and corresponding geodesic paths will often be suppressed, so our parallel transports will be denoted as $P_{x_i \to m_{ij}}$. Also, note that for $n < m$ and $\mathcal{X}_n \subset \mathcal{X}_m$, we assume each additional point is again sampled iid\ from the uniform distribution on $\mathcal{M}$. For the categorically-minded reader, we remark that our sheaf is constructed such that refining our sample then leads to a \textit{subfunctor} $\mathcal{F}^t_{n} \subset \mathcal{F}^t_{m}$. For the Hilbert sheaf $\mathcal{F}_n^t$ on the graph $G_n=(\mathcal{X}_n,E)$, a \textit{signal} is an element of the Hilbert space
\begin{equation}
\label{defn:Hilbert-Global-Section}
C^0(\mathcal{F}_n^t;G_n)
:=
\bigoplus_{x_i \in \mathcal{X}_n}
\mathcal{F}_n^t(x_i).
\end{equation}
\begin{example}
If $\mathcal{F}_n^t$ encodes univariate spatiotemporal data, each node stalk can be chosen as $\mathcal{F}_n^t(x_i)=L^2(\mathbb{R})$. Then
$
C^0(\mathcal{F}_n^t;G_n)
=
\bigoplus_{x_i\in\mathcal{X}_n}
L^2(\mathbb{R}),
$
so a signal assigns a full time series to every node, recovering the usual notion of a node-time graph signal.
\end{example}

\begin{example}
In one dimension, a probability distribution $\mu\in\mathcal{P}_2(\mathbb{R})$ can be represented by its quantile function $Q_\mu\in L^2([0,1])$, and the Wasserstein distance becomes the $L^2$ distance between quantiles \cite{villani2003topics}. Thus, by choosing node stalks
$
\mathcal{F}_n^t(x_i)=L^2([0,1]),
$
a signal assigns a full probability distribution to each graph node, recovering the distributional graph-signal setting of \cite{Ji2025GraphDistributionalSignals,zhao2026graphdistributionvaluedsignalswasserstein}.
\end{example}

Finally, analogous to the construction of the connection Laplacian $\Delta_\nabla$, we may construct the Hilbert sheaf Laplacian. Further details, such as self-adjointness, are discussed in Appendix~\ref{sec:cellular-sheaves}.

\begin{definition}[Hilbert Sheaf Laplacian]
\label{Hilbert-Sheaf-Laplacian}
Let $\mathcal{F}_n^t$ be the Hilbert sheaf on the graph $G_n=(\mathcal{X}_n,E)$ induced by Def,~\ref{def:sheaf-from-bundle}. Fix an orientation for each edge $e\in E$. The \textbf{Hilbert sheaf Laplacian} is the bounded linear operator
\begin{equation}
\Delta_{\mathcal{F}_n^t}:
C^0(\mathcal{F}_n^t;G_n)
\to
C^0(\mathcal{F}_n^t;G_n)
\end{equation}
defined, for a signal $S\in C^0(\mathcal{F}_n^t;G_n)$ and at a node $x_i\in\mathcal{X}_n$, by
\begin{equation}
(\Delta_{\mathcal{F}_n^t}S)_{x_i}
=
\sum_{\substack{e\in E:\\ e=\{x_i,x_j\}}}
(\mathcal{F}_n^t)_{x_i\leq e}^{*}
\left(
(\mathcal{F}_n^t)_{x_i\leq e}S_{x_i}
-
(\mathcal{F}_n^t)_{x_j\leq e}S_{x_j}
\right),
\end{equation}
where $x_j$ denotes the other endpoint of $e$, and $(\mathcal{F}_n^t)_{x_i\leq e}^{*}$ is the adjoint of the restriction map $(\mathcal{F}_n^t)_{x_i\leq e}$.
\end{definition}
Intuitively, $\Delta_{\mathcal{F}_n^t}$ measures how much a signal fails to be locally consistent across edges: before comparing $S_{x_i}$ and $S_{x_j}$, both values are mapped into the common edge stalk $\mathcal{F}_n^t(e)$ by the restriction maps. Thus, it is a broad generalization of a graph Laplacian, with restriction maps replacing scalar edge weights. The Hilbert sheaf Laplacian $\Delta_{\mathcal{F}_n^t}$ is a self-adjoint bounded linear operator. Once the manifold is sampled and the induced sheaf Laplacian is computed, space-discretized HilbNets, which are still not implementable due to the infinite-dimensional signals, are simply given by Def,~\ref{HilbNet-defn} with the connection Laplacian of $\mathcal{E}$ replaced by the sheaf Laplacian of $\mathcal{F}_n^t$, i.e., by $\Omega(\mathcal{F}_n^t,\Delta_{\mathcal{F}_n^t},\mathcal{W},\sigma)$.

\textbf{Signal Sampling.} Hilbert cellular sheaves are the structures that arise when we sample our base manifold but faithfully record the potentially infinite-dimensional signal in each fiber. In practice, we typically only have access to a sampled or compressed version of the signal as well. For instance, when considering a timeseries $S \in L^2(\mathbb{R})$, we may use the orthogonal Fourier basis $\overline{\{e^{ik\theta}\}}_{k \in \mathbb{Z}} = L^2(\mathbb{R})$ and then record a compressed representation with respect to this basis i.e. $[\langle S, e^{-id\theta}\rangle, \dots, \langle S, e^{id\theta}\rangle] $ for some $d$. We can consider fiber-wise orthogonal projections with respect to any chosen basis in the Hilbert bundle setting as a principled approach to discretizing Hilbert bundle signals.
\begin{proposition}\label{prop:signal_discretization}
    Let  $(\mathcal{M},\mathcal{E},\nabla)$ be a Hilbert bundle, with strictly infinite-dimensional generic Hilbert-space fiber $\mathcal{H}$. Fix an orthogonal basis $\mathcal{B}=\{e_1, e_2, ...\} $ of $\mathcal{H}$ and let $\mathcal{H}_d:= \mathrm{span}(e_1, e_2, ..., e_d) $. Then there exists a smooth map of bundles 
    \begin{equation}
   \Pi_d: \mathcal{E} \to \mathcal{E}_d
   \end{equation}
   where $\mathcal{E}_d $ is a $d$-dimensional vector bundle with generic fiber $\mathcal{H}_d$ and at each $x\in \mathcal{M}$, $\left .\Pi_d\right |_{\calE_x}: \mathcal{E}_x \to \mathcal{E}_{d,x}$ recovers the usual orthogonal projection map. See Appendix \ref{app:proofs:lemmas-finite-rank} for details. 
\end{proposition}

Applying Proposition \ref{prop:signal_discretization} to $(\mathcal{M},\mathcal{E},\nabla)$, we obtain $(\mathcal{M},\mathcal{E}_d,\nabla_d)$, to which we may apply the spatial discretization of Def.\ref{def:sheaf-from-bundle} to construct the cellular sheaf $\mathcal{F}_{n,d}^t$ with $d$-dimensional stalks. We refer to $\mathcal{F}_{n,d}^t$ as a \textit{network sheaf}. The signals on this sheaf are then sampled Hilbert bundle signals, i.e., we can discretize $S \in L^2(\mathcal{M}, \mathcal{E})$  as a $nd$-dimensional vector $\mathbf{s}_{n,d}:=\Pi_dS  \in C^0(\mathcal{F}_{n,d}^t; G_n) \subseteq \mathbb{R}^{nd}$, stacking the $d$-dimensional orthogonal projections over the $n$ sampled locations, with respect to the chosen basis $\mathcal{B}$. In this case, the restriction maps can be written as matrices, thus the sheaf Laplacian becomes a block matrix
$\Delta_{\mathcal{F}_{n,d}^t}\in\mathbb{R}^{nd\times nd}$ whose $(i,j)$-block maps the discretized stalk over $x_j$ to the discretized stalk over $x_i$ and is given by
\begin{equation}\label{eq:discrete_sheaf_laplacian}
(\Delta_{\mathcal{F}_{n,d}^t})_{ij}
=
\begin{cases}
\displaystyle
\sum_{r:\, e_{ir}\in E}
k_{ir}^t I_d,
& i=j, \\[1.2em]
\displaystyle
-P_{x_j\to x_i}^{(d,e_{ij})},
& i\neq j \text{ and } e_{ij}\in E, \\[0.5em]
0,
& \text{otherwise,}
\end{cases}
\end{equation}
where $P_{x_j\to x_i}^{(d,e_{ij})}:=k_{ij}^t
\left(P_{x_i\to m_{ij}}^{(d)}\right)^*
P_{x_j\to m_{ij}}^{(d)}$, and $P_{x_j\to m_{ij}}^{(d)}$ denotes the restriction of the parallel transport map from $\mathcal{E}_{x_j}$ to $\mathcal{E}_{m_{ij}}$ to the corresponding $d$-dimensional subbundles in the image of $\Pi_d$. We may thus use this Laplacian to build an implementable sheaf convolutional architecture as follows. 
\begin{definition}[$(n,d)$-Hilbert bundle convolutional neural network] \label{def: discretized-HilbNets}  Let  $(\mathcal{M},\mathcal{E},\nabla)$ be a Hilbert bundle, with generic Hilbert-space fiber $\mathcal{H}$ and corresponding basis $\mathcal{B}$. A $(n,d)$-Hilbert bundle convolutional neural network, or $(n,d)$\textbf{-HilbNet}, is specified by a filter bank $\mathcal{W}=\{g^{\ell}_{u,q}\}_{\ell,u,q}$ with $g^{\ell}_{u, q}\in L_c^\infty(\mathbb{R})$, a Lipschitz continuous nonlinear activation $\sigma:\mathbb{R}\to\mathbb{R}$,
the choice of a $d$-dimensional subbasis of $\mathcal{B}$ and sample $\mathcal{X}_n \subset \mathcal{M}$. Given input sampled signals $\mathbf{s}_{n,d,1},..., \mathbf{s}_{n,d,F_0}\in C^0(\mathcal{F}_{n,d}^t; G_n) $, the $L$-layer network output is obtained by the recursion
    \begin{equation}
        \mathbf{s}_{n,d,u}^{\ell+1}
        =
        \sigma\left(\sum_{q=1}^{F_\ell}
        g^\ell_{u,q}(\Delta_{\mathcal{F}_{n,d}^t})\mathbf{s}_{n,d,q}^\ell
        \right),
        \qquad
        \ell=0,\dots,L-1, \qquad  \mathbf{s}_{n,d,q}^0=\mathbf{s}_{n,d,q}.
    \end{equation}
\end{definition}
Discretized HilbNets are fully implementable and can be compactly written using Def.~\ref{HilbNet-defn} with the connection Laplacian of $\mathcal{E}$ replaced by the sheaf Laplacian of $\mathcal{F}_{n,d}^t$, i.e., by $\Omega( \mathcal{F}_{n,d}^t,\Delta_{\mathcal{F}_{n,d}^t},\mathcal{W},\sigma)$.

\begin{example} In the case that we consider our filter bank $\mathcal{W}$ to consist of order $K$ polynomials, the $(n,d)$-HilbNet can be written as a novel variant of sheaf neural networks \cite{hansen2020sheafnn,battiloro2023tangent} given by
\begin{align} \label{eqn:discretized-hilbnet}
    &\mathbf{S}_{n,d}^{\ell+1}=\sigma\left(\sum_{k=0}^{K-1}(\Delta_{\mathcal{F}_{n,d}^t})^k \mathbf{S}^{\ell}_{n,d}\mathbf{W}_{\ell, k} \right) \in \mathbb{R}^{n d}.
\end{align}

where the matrices $\mathbf{S}^\ell_{n,d} \in \mathbb{R}^{nd \times F_\ell}$, and $\{\mathbf{W}_{\ell,k}\}_k$, with  $\mathbf{W}_{\ell,k}\in \mathbb{R}^{F_\ell \times F_{\ell+1}}$ collect the sampled signals and the learnable filter weights at each layer, respectively.
\end{example}

\vspace{-.2cm}\section{Theoretical Convergence Guarantees}\vspace{-.2cm}\label{sec:convergence}

Our main result may be understood as a far-reaching generalization of the convergence result of \citet{BELKIN20081289}. Consider a random sample $\mathcal{X}_n \subset \mathcal{M}$ and the corresponding geometric graph $G_{n} = (\mathcal{X}_n,E)$. It is established in \cite{BELKIN20081289} that as sampling density increases, the weighted graph Laplacian converges to the manifold Laplace-Beltrami operator in probability. We analogously show that the Hilbert sheaf Laplacian $\Delta_{\mathcal{F}_n}$ over $G_n$ converges to $\Delta_\nabla$, thus recovering the results of \cite{BELKIN20081289} as the special case $\mathcal{E}=\mathcal{M} \times \mathbb{R}$. Our proof, presented in Appendix \ref{sec:appendix-proofs}, is inspired by the strategy of \cite{BELKIN20081289} but with the necessary non-trivial modifications to accommodate the simultaneous generalization to cellular sheaves instead of graphs and to infinite-dimensional Hilbert-spaces. In order to state our results, we require the following intermediary operator. 

\begin{definition} (Point-Cloud Extension of Sheaf Laplacian) \label{defn:point-cloud-extension-main} Let $(\mathcal{M}, \mathcal{E}, \nabla)$  be a Hilbert bundle and consider a sample $\calX_n \subset \mathcal{M}$. Then the corresponding Hilbert sheaf Laplacian $\Delta_{\mathcal{F}^t_n}$ may be extended to the \textbf{point-cloud Laplacian}  $\hat{\Delta}_{\calF^t_{n}}$, an operator on $L^2(\mathcal{M}, \mathcal{E})$ via \begin{equation}(\hat{\Delta}_{\mathcal{F}^t_n}S)(x)= \frac{1}{n} \sum_j e^{-d_{\calM}(x,x_j)^2/4t}\big ( S(x) - P_{x_j \to x}S(x_j)\big ) 
\end{equation}
\end{definition}
As such, we are able to consider the sheaf-level and bundle-level Laplacians as operators on the same space through this extension. In this setting, we then have the following convergence result.

\begin{theorem} (Convergence of Hilbert Sheaf Laplacian) \label{thm: main-thm} Let $\mathcal{M}$ be a $m$-dimensional closed Riemannian manifold. Further, let $(\mathcal{M}, \mathcal{E}, \nabla)$  be a Hilbert bundle and associated connection Laplacian $\Delta_\nabla$. Fix a section $S \in  C^3(\mathcal{M},\mathcal{E})$.  Consider a random sample $\mathcal{X}_n = \{x_1, x_2, \cdots, x_n\} \subset \mathcal{M}$. Let $\mathcal{F}^t_{n}$ be the induced Hilbert cellular sheaf with bandwidth $t$. Then we have, for any $x \in \calM$,
 \begin{equation} \label{thm: convergence in probability} \tag{A}
\lim _{n \rightarrow \infty} \frac{1}{t_n\left(4 \pi t_n\right)^{\frac{m}{2}}} {\hat{\Delta}}_{\mathcal{F}^{t_n}_n}S(x)=\frac{1}{\operatorname{vol}(\mathcal{M})} \Delta_\nabla S(x) \quad \text{in probability,}
 \end{equation}
 with bandwidth $t_n = n^{-\frac{1}{m + 2 + \alpha}}$, $\alpha > 0$. Further, if $S \in C^4(\mathcal{M},\mathcal{E})$, we have 
\begin{equation}  
\label{thm: L2-upgrade} \tag{B}
\lim _{n \rightarrow \infty} \E_{\mathcal{X}}\left [  \left \| \frac{1}{t_n\left(4 \pi t_n\right)^{\frac{m}{2}}} \hat{\Delta}_{\mathcal{F}^{t_n}_n} S(x) -\frac{1}{\operatorname{vol}(\mathcal{M})} \Delta_\nabla S(x) \right \|_{L^2}^2 \right ] = 0\quad \text{in  $L^2$-norm}.
\end{equation}
\end{theorem}
Our framework may be seen as concurrently generalizing the convergence results for the weighted graph Laplacian of \cite{BELKIN20081289} as well as the graph connection Laplacian of \cite{singer2017spectral} to  
allow for arbitrary bundles, with potentially infinite-dimensional fibers, equipped with an arbitrary choice of connection. Such convergence results have previously served as the basis of transferability and robustness results in geometric deep learning \cite{wang2022stability, 10508628, levie2019transferability}, as well as to justify the development of numerous Laplacian-based manifold learning techniques \cite{belkin2001laplacian, coifman2006diffusion, yang2024spherical,ValletL08,Rustamov07}. We may likewise develop generalizations of these results for implementable sheaf Laplacians by discretizing in the signal-domain.

\textbf{Finite Rank Convergence.}
Consider a signal $S$ that has been sampled to $\mathbf{s}_{n,d}$ as per Proposition \ref{prop:signal_discretization}. We then have the following theoretical guarantee, which formalizes the intuitive notion that the fully discretized sheaf Laplacian converges to true connection Laplacian as we take an increasingly refined sample of both the underlying manifold and the signal. 

\begin{theorem}[Finite-Rank Approximation]\label{thm: diagonal finite rank convergence}
 Consider the setting of Theorem \ref{thm: main-thm} with a section $S\in C^4(\calM,\calE)$, for $\calE$ a strictly infinite-dimensional Hilbert bundle. Then there exists a sequence of finite rank approximating sheaves $\mathcal{F}^{t_n}_{n,d_n}$ such that
\begin{equation}
 \lim_{n \to \infty} \E_{\mathcal{X}} \left [ 
    \left \| \frac{1}{t_n (4 \pi t_n)^{m/2}} \hat{\Delta}_{\mathcal{F}^{t_n}_{n,\subbundleidx_n}} \mathbf{s}_{n,d_n} - \frac{1}{\mathrm{vol}(\calM)} \Delta_\nabla S \right \|_{L^2}^2 
    \right ] = 0,
\end{equation}
    with bandwidth $t_n = n^{-\frac{1}{m + 2 + \alpha}}$, $\alpha > 0$.
\end{theorem}

While this result, as described in Appendix \ref{app:broader_impact}, can pave the way for new developments of Laplacian-based manifold learning, we here restrict our focus to its consequences for $(n,d)$-HilbNets.

\begin{corollary}[Convergence in Architecture]\label{cor: convergence in architecture}
    Under the hypotheses of Theorem \ref{thm: diagonal finite rank convergence}, let $\{\subbundleidx_n\}_n$ be the required sequence. Fix a   fiber-wise nonlinearity $\sigma$ that is $C_\sigma$-Lipschitz in the corresponding fiber norms and choice of filter bank $\mathcal{W}$. Then, the output of the discrete $(n,d)$-HilbNet converges to the output of the continuous HilbNet architecture in the sense that
\begin{equation}
\Omega( \mathcal{F}^{t_n}_{n,\subbundleidx_n}, \hat{\Delta}_{\mathcal{F}^{t_n}_{n,\subbundleidx_n}},\mathcal{W}, \sigma) \to \Omega(\mathcal{E}, \Delta_{\nabla}, \mathcal{W}, \sigma) \quad \text{in mean squared error},
\end{equation}
as the sampling density $n, d_n \to \infty$. 
\end{corollary}

\begin{corollary}(Transferability)\label{transferability} 
Let $\{\mathcal{X}_n\}_{n=1}^{\infty}$ and $\{\mathcal{Y}_n\}_{n=1}^{\infty}$ be independent sequences of random samples of $\mathcal{M}$. Let $\{\subbundleidx_n\}^{\infty}_{n=1}$ be a sequence such that the conclusion of Theorem \ref{thm: diagonal finite rank convergence} holds for both samplings. For any fiber-wise nonlinearity  $\sigma$ that is $C_\sigma$-Lipschitz in the corresponding fiber norms and any filter bank $\calG$, we then have that, 
\begin{equation}
    \lim_{n \to \infty} \E_{\mathcal{X},\mathcal{Y{}}} \left [ \left \| \Omega(\calF^{t_n}_{\mathcal{X}_n,d_n}, \hat{\Delta}_{\calF^{t_n}_{\mathcal{X}_n,d_n}}, \mathcal{W}, \sigma) - \Omega(\calF^{t_n}_{\mathcal{Y}_n,d_n}, \hat{\Delta}_{\calF^{t_n}_{\mathcal{Y}_n,d_n}},\mathcal{W}, \sigma) \right \|_{L^2}^2 \right ] = 0. 
\end{equation}
    Further, one may derive a sample-independent quantitative bound for the $L^2$ disagreement $\| \Omega(\calF^{t_n}_{\mathcal{X}_n, d_n}, \hat{\Delta}_{\calF^{t_n}_{\mathcal{X}_n, d_n}}, \mathcal{W}, \sigma) - \Omega(\calF^{t_n}_{\mathcal{Y}_n, d_n},\hat{\Delta}_{\calF^{t_n}_{\mathcal{Y}_n, d_n}}, \mathcal{W}, \sigma) \|_{L^2}$. See Appendix \ref{sec:appendix-proofs} for details. 
\end{corollary}

By these results, we may understand the $(n,d)$-HilbNets as the principled discretization of continuous HilbNets. These may also be understood as robustness results for $(n,d)$-HilbNets, as they establish that the architecture is scale-consistent and is stable against resampling of the base manifold. Notably, by the generality of HilbNets, most existing geometric convolutional architectures can be understood as instances of HilbNets for a particular choice of bundle and connection. As such, our results may also be seen as extending the transferability guarantees of \cite{wang2022stability, 10508628, levie2019transferability} to a larger class of architectures and data modalities. See Appendix \ref{sec:existing-architectures} for further discussion.
\begin{table}[h]
\centering
\small
\resizebox{\linewidth}{!}{
\begin{tabular}{ccccccc}
\toprule
& \multicolumn{2}{c}{Free $O(m)$}
& \multicolumn{2}{c}{Circulant}
& \multicolumn{2}{c}{Frozen identity (GCN)} \\
\cmidrule(lr){2-3}
\cmidrule(lr){4-5}
\cmidrule(lr){6-7}
$n$
& \cameraready{Det.} & Theory
& \cameraready{Det.} & Theory
& \cameraready{Det.} & Theory \\
\midrule
$16$   & \cameraready{$(1.71{\pm}0.17){\times}10^{-7}$} & $0$ & \cameraready{$(1.84{\pm}0.29){\times}10^{-2}$} & $1.84{\times}10^{-2}$ & \cameraready{$(2.30{\pm}0.31){\times}10^{-2}$} & $2.30{\times}10^{-2}$ \\
$32$   & \cameraready{$(1.28{\pm}0.57){\times}10^{-7}$} & $0$ & \cameraready{$(1.18{\pm}0.14){\times}10^{-2}$} & $1.18{\times}10^{-2}$ & \cameraready{$(1.46{\pm}0.18){\times}10^{-2}$} & $1.46{\times}10^{-2}$ \\
$64$   & \cameraready{$(1.82{\pm}0.47){\times}10^{-7}$} & $0$ & \cameraready{$(1.03{\pm}0.14){\times}10^{-2}$} & $1.03{\times}10^{-2}$ & \cameraready{$(1.29{\pm}0.16){\times}10^{-2}$} & $1.29{\times}10^{-2}$ \\
$128$  & \cameraready{$(2.22{\pm}0.25){\times}10^{-7}$} & $0$ & \cameraready{$(8.94{\pm}0.71){\times}10^{-3}$} & $8.93{\times}10^{-3}$ & \cameraready{$(1.11{\pm}0.09){\times}10^{-2}$} & $1.11{\times}10^{-2}$ \\
$256$  & \cameraready{$(1.87{\pm}0.35){\times}10^{-7}$} & $0$ & \cameraready{$(7.80{\pm}0.12){\times}10^{-3}$} & \cameraready{$7.80{\times}10^{-3}$} & \cameraready{$(9.67{\pm}0.18){\times}10^{-3}$} & $9.67{\times}10^{-3}$ \\
\bottomrule
\end{tabular}
}
\caption{Synthetic transport recovery. \cameraready{\textbf{Det.}: edge-MSE of the learned transports, bounded below by Theory (Appendix~\ref{app:synthetic}).} \textbf{Theory}: analytical squared Frobenius projection distance of \cameraready{$\tilde P^{LC}_{x_j\to x_i}$} onto the variant's hypothesis class.}
\label{tab:transport_recovery}
\end{table}
\vspace{-.2cm}\section{Experimental Results}\vspace{-.2cm}
A key practical advantage of the $(n,d)$-HilbNets architecture in comparison to existing approaches for processing graph signals is our use of parallel transport, which in practice can be known or learned. For instance, these transport operators allow us to incorporate principled signal-level geometric priors in concert with the spatial priors of existing spatiotemporal GCNs. This is well-aligned with the thesis of geometric deep learning that the principled incorporation of geometric priors improves performance, particularly in the low-data or small-model regimes.  For further discussion on strategies for either hand-crafting or learning parallel transport operators, see Appendix \ref{sec:PT-Algo}. Here, we first validate our setup for a synthetic dataset realized from discretizing a known Hilbert bundle in information geometry, and then consider performance on real-world spatiotemporal graph benchmarks based upon traffic forecasting.  In all the experiments, we use the polynomial $(n,d)$-HilbNet from \eqref{eqn:discretized-hilbnet} and learned transport maps. \cameraready{Transports are parametrized as free $O(d)$ maps (products of $R$ Householder reflections), orthogonal circulant maps, or frozen identities (see Appendix~\ref{app:ptrans_param} for details). For traffic forecasting, each HilbNet variant stacks $L{=}2$ sheaf convolutional layers with polynomial filters of order $K{=}2$ and widths $16$ and $32$, followed by a per-node linear readout over the full horizon. Full architectural details and hyperparameters are in Appendix~\ref{app:experiments} (Tables~\ref{tab:synthetic_hp} and~\ref{tab:traffic_hp}). Code is available at \coderepo.}

\begin{table}[t]
\centering
\small
\resizebox{\linewidth}{!}{
\begin{tabular}{lcccccccccc}
\toprule
Model
& Params
& \multicolumn{3}{c}{Horizon 3}
& \multicolumn{3}{c}{Horizon 6}
& \multicolumn{3}{c}{Horizon 12} \\
\cmidrule(lr){3-5}
\cmidrule(lr){6-8}
\cmidrule(lr){9-11}
&
& MAE & RMSE & MAPE
& MAE & RMSE & MAPE
& MAE & RMSE & MAPE \\
\midrule
FC-LSTM~\citep{li2018dcrnn}
& \cameraready{${>}1.5\text{M}$}
& $3.44$ & $6.30$ & $9.6$
& $3.77$ & $7.23$ & $10.9$
& $4.37$ & $8.69$ & $13.2$ \\
STAEformer~\citep{liu2023staeformer}
& \cameraready{${\sim}1.3\text{M}$}
& $2.65$ & $5.11$ & $6.85$
& $2.97$ & $6.00$ & $8.13$
& $3.34$ & $7.02$ & $9.70$ \\
\midrule
MLP fiber baseline
& $5{,}212$
& $3.131{\pm}0.004$ & $6.074{\pm}0.007$ & $8.271{\pm}0.021$
& $3.775{\pm}0.005$ & $7.496{\pm}0.013$ & $10.626{\pm}0.046$
& $4.690{\pm}0.011$ & $9.184{\pm}0.014$ & $14.341{\pm}0.097$ \\
Spatiotemporal graph baseline
& $8{,}908$
& $3.453{\pm}0.080$ & $6.709{\pm}0.220$ & $9.241{\pm}0.279$
& $4.160{\pm}0.117$ & $8.158{\pm}0.184$ & $11.916{\pm}0.491$
& $5.277{\pm}0.093$ & $10.102{\pm}0.128$ & $16.006{\pm}0.662$ \\
HilbNet, frozen identity (GCN)
& $5{,}756$
& $3.092{\pm}0.007$ & $5.920{\pm}0.013$ & $8.218{\pm}0.065$
& $3.713{\pm}0.010$ & $7.312{\pm}0.020$ & $10.520{\pm}0.061$
& $4.608{\pm}0.034$ & $8.991{\pm}0.043$ & $14.166{\pm}0.240$ \\
HilbNet, circulant
& $11{,}656$
& $2.939{\pm}0.021$ & $5.630{\pm}0.061$ & $7.908{\pm}0.067$
& $3.409{\pm}0.032$ & $6.765{\pm}0.092$ & $9.844{\pm}0.125$
& $4.059{\pm}0.049$ & $8.149{\pm}0.114$ & $12.471{\pm}0.195$ \\
HilbNet, free $O(T)$
& $119{,}036$
& $\mathbf{2.923{\pm}0.013}$ & $\mathbf{5.586{\pm}0.048}$ & $\mathbf{7.808{\pm}0.083}$
& $\mathbf{3.372{\pm}0.023}$ & $\mathbf{6.732{\pm}0.066}$ & $\mathbf{9.507{\pm}0.096}$
& $\mathbf{3.938{\pm}0.030}$ & $\mathbf{8.042{\pm}0.101}$ & $\mathbf{11.642{\pm}0.136}$ \\
\bottomrule
\end{tabular}}
\resizebox{\linewidth}{!}{
\begin{tabular}{lcccccccccc}
\toprule
Model
& Params
& \multicolumn{3}{c}{Horizon 3}
& \multicolumn{3}{c}{Horizon 6}
& \multicolumn{3}{c}{Horizon 12} \\
\cmidrule(lr){3-5}
\cmidrule(lr){6-8}
\cmidrule(lr){9-11}
&
& MAE & RMSE & MAPE
& MAE & RMSE & MAPE
& MAE & RMSE & MAPE \\
\midrule
FC-LSTM~\citep{li2018dcrnn}
& \cameraready{${>}1.5\text{M}$}
& $2.05$ & $4.19$ & $4.8$
& $2.20$ & $4.55$ & $5.2$
& $2.37$ & $4.96$ & $5.7$ \\
STAEformer~\citep{liu2023staeformer}
& \cameraready{${\sim}1.4\text{M}$}
& $1.31$ & $2.78$ & $2.76$
& $1.62$ & $3.68$ & $3.62$
& $1.88$ & $4.34$ & $4.41$ \\
\midrule
MLP fiber baseline
& $5{,}212$
& $1.459{\pm}0.003$ & $3.145{\pm}0.017$ & $3.026{\pm}0.020$
& $1.942{\pm}0.004$ & $4.378{\pm}0.017$ & $4.385{\pm}0.045$
& $2.513{\pm}0.004$ & $5.658{\pm}0.015$ & $6.209{\pm}0.038$ \\
Spatiotemporal graph baseline
& $8{,}908$
& $\mathbf{1.400{\pm}0.002}$ & $2.980{\pm}0.016$ & $\mathbf{2.924{\pm}0.003}$
& $1.850{\pm}0.002$ & $4.162{\pm}0.009$ & $4.222{\pm}0.009$
& $2.388{\pm}0.004$ & $5.395{\pm}0.023$ & $5.924{\pm}0.032$ \\
HilbNet, frozen identity (GCN)
& $5{,}756$
& $1.439{\pm}0.003$ & $3.077{\pm}0.022$ & $2.985{\pm}0.007$
& $1.901{\pm}0.006$ & $4.264{\pm}0.021$ & $4.319{\pm}0.015$
& $2.446{\pm}0.007$ & $5.495{\pm}0.019$ & $6.020{\pm}0.032$ \\
HilbNet, circulant
& $15{,}366$
& $1.413{\pm}0.002$ & $2.971{\pm}0.018$ & $2.982{\pm}0.015$
& $1.806{\pm}0.003$ & $\mathbf{3.950{\pm}0.015}$ & $4.162{\pm}0.037$
& $2.211{\pm}0.014$ & $\mathbf{4.866{\pm}0.048}$ & $5.386{\pm}0.047$ \\
HilbNet, free $O(T)$
& $190{,}268$
& $1.417{\pm}0.002$ & $\mathbf{2.969{\pm}0.015}$ & $3.058{\pm}0.013$
& $\mathbf{1.793{\pm}0.005}$ & $3.958{\pm}0.026$ & $\mathbf{4.127{\pm}0.037}$
& $\mathbf{2.181{\pm}0.003}$ & $4.873{\pm}0.019$ & $\mathbf{5.214{\pm}0.045}$ \\
\bottomrule
\end{tabular}
}
\caption{METR-LA (top) and PEMS (bottom) traffic forecasting results. Bottom block of each table: our experiments (mean $\pm$ standard deviation over five seeds). Top block of each table: external baselines reported as in the cited papers \cameraready{(parameter counts estimated from their published configurations)}. MAPE is in percent. Lower is better for all metrics.}
\label{tab:traffic_results}
\vspace{-1em}
\end{table}

\textbf{Synthetic Experiments.} 
We first consider a task where, for a known Hilbert bundle and connection, we train a discretized HilbNet to predict the true parallel transport operators. Following \cite{Malago2018}, the base manifold is $\mathcal{M}=\mathrm{Sym}^{++}(p)$ equipped with the Otto-Wasserstein metric, each $\Sigma\in\mathcal{M}$ parameterizing a density $\mathcal{N}(0,\Sigma)$. The ambient fiber $\mathcal{H}_\Sigma=L^2(\rho_\Sigma;\mathbb{R}^p)$ is genuinely infinite-dimensional; the computational fiber is the Otto-velocity image of covariance perturbations, a sub-bundle whose fibers are already finite-dimensional with $d=p(p+1)/2$, and on which the Levi-Civita transports $P^{LC}_{x_i\to m_{ij}}$ admit a closed-form. We sample $n$ points, build a $k$NN graph $G_n$ ($k{=}8$) under $W_2$, and assemble the network sheaf $\mathcal{F}_{n,d}^t$ and its Laplacian $\Delta_{\mathcal{F}_{n,d}^t}\in\mathbb{R}^{nd\times nd}$ from Def,~\ref{def:sheaf-from-bundle}. We consider three transport parametrizations from Appendix~\ref{sec:PT-Algo} ( averaged over 3 seeds): free $O(d)$ (Householder), circulant, and frozen identity (a usual GCN \cite{kipf2016semi}), to recover the Levi-Civita transports in Cholesky-rescaled coordinates. As the reader can notice in Table \ref{tab:transport_recovery}, the free class recovers \cameraready{$\tilde P^{LC}_{x_j\to x_i}$} to numerical precision \cameraready{${\approx}2\cdot10^{-7}$}, while each restricted class converges to its analytical Frobenius-projection plateau to within \cameraready{$0.05\%$}. This confirms that the transport hypothesis class constrains the per-edge restriction maps of $\Delta_{\mathcal{F}_{n,d}^t}$ in a quantitatively predictable way. More experiments and details are in Appendix~\ref{app:synthetic}

\textbf{Traffic Forecasting.}
We evaluate HilbNets on real-world spatiotemporal traffic-speed forecasting, where each node of a road-network graph carries a time-series fiber with $d=T$ observed time steps. This is a natural instance of the Hilbert-bundle framework: the base graph encodes spatial proximity among sensors, while the edgewise transports $P_{j\to i}^{e_{ij}}$ from Appendix~\ref{sec:PT-Algo} control how temporal fibers are aligned before filtering. We test on two standard benchmarks, METR-LA~\citep{li2018dcrnn} and PEMS-BAY~\citep{li2018dcrnn}, predicting future speeds at horizons $3$, $6$, and $12$. We compare the same three HilbNet transport classes, frozen identity (a GCN), circulant, and free $O(T)$, against a fiber-only MLP and a spatiotemporal graph baseline obtained by stacking GCN layers with one-dimensional convolutional layers processing the temporal dimension, all sharing the same polynomial sheaf filter order, readout, and forecasting loss. Full experimental details are given in Appendix~\ref{app:traffic}. Table~\ref{tab:traffic_results} reports MAE, RMSE, and MAPE (mean $\pm$ std over five seeds). On both datasets, learning non-trivial transports consistently improves over frozen identity at all horizons, confirming that the sheaf structure helps beyond the usual graph structure. The free $O(T)$ class achieves the best overall accuracy, but the circulant variant is competitive while using roughly one tenth of the transport parameters, supporting the use of structured, physically motivated transport priors for spatiotemporal data. This confirms that geometric inductive biases can lead to better performance in low-data regimes or comparable performance in normal regimes with substantially fewer parameters.

\vspace{-.2cm}\section{Conclusions}\vspace{-.2cm}\label{sec:conclusions}

We introduced a novel convolutional learning framework for infinite-dimensional signals over a manifold using Hilbert bundles,  a setting that concurrently unifies and generalizes existing approaches. It allows us to consider arbitrary connection Laplacians, a more general class of filters via the Borel calculus, and thus applications of the resulting filters to potentially infinite-dimensional signals. We defined HilbNets as stacks of Hilbert bundle filters and pointwise non-linearity. We consequently introduced a practically implementable \textit{(n,d)-HilbNet} via the theory of \textit{Hilbert cellular sheaves}, and proved that this discretized architecture converges to the continuous architecture in the limit. Notably, our convergence in architecture is derived from a novel extension of the Laplacian convergence result of \cite{BELKIN20081289} to the setting of Hilbert sheaves and Hilbert bundles, and we believe this result will be of independent interest to the broader machine learning community. Lastly, we verified the benefits of integrating domain-specific geometric priors through experiments with discretized HilbNets on synthetic and real-world data. Overall, we envision the prospective impact of our contributions as two-fold: the HilbNet framework allows for the principled development of domain-specific architectures through appropriate choices of connection, filter bank, and manifold- and signal-alignment measures, while our Hilbert Laplacian convergence theorem lays the theoretical groundwork for the development of Laplacian-based manifold-theoretic techniques in the setting of infinite-dimensional signals, spanning from mechanistic interpretability to self-supervised learning methods. A more detailed discussion on impact \cameraready{and future directions} is presented in Appendix \ref{app:broader_impact}.

\noindent\cameraready{\textbf{Limitations.} Our theoretical guarantees rest on some assumptions that may fail to hold in practice, as is usually the case. The convergence results (Theorems~\ref{thm: main-thm}--\ref{thm: diagonal finite rank convergence}) require the base manifold $\mathcal{M}$ to be closed (compact without boundary), sections to be sufficiently smooth, and samples to be drawn i.i.d.\ from the uniform distribution. Real-world sensor networks, such as those in our traffic experiments, are neither uniformly sampled nor necessarily supported on compact manifolds. Moreover, measured signals are typically noisy rather than smooth. However, many works, including the foundational results of \cite{BELKIN20081289}, as well as related works on vector diffusion maps \cite{singer2012vdm,singer2017spectral} and manifold-based learning \cite{coifman2006diffusion}, assume compact manifolds and uniform or smooth sampling densities, yet are routinely and successfully applied under weaker conditions. We expect that HilbNets would behave similarly and, in fact, our numerical results confirm that HilbNets likewise remain effective under these standard approximations, generally outperforming matched baselines that lack the principled bundle-geometric structure. On the computational side, the network sheaf Laplacian scales quadratically in the product of spatial and fiber dimensions, which may become prohibitive for very large graphs or high-dimensional signal discretizations without further sparsification or approximation strategies. In practice, this cost can be contained by building the sheaf on a sparse $k$NN graph rather than on a dense sample, and by structured, parameter-efficient transport parametrizations such as circulant maps (Appendix~\ref{app:ptrans_param}). Finally, broader and tailored empirical validation on other infinite-dimensional signal types, such as functional data on manifolds, remains an important direction for future work.}

\begin{ack}
\cameraready{C.B. and F.D. acknowledge support from the National Institutes of Health Grant 1R01ES037156-01.The authors declare no competing interests.}
\end{ack}

\bibliographystyle{tmlr}
\bibliography{tmlr}

\appendix
\onecolumn

\section*{Appendix Contents}
\begin{tabular}{@{}p{0.93\textwidth}r@{}}
\textbf{A \quad Extended Related Works} \dotfill & \pageref{appsec:relwork}\\[0.5em]
\textbf{B \quad \cameraready{Broader Impact and Future Directions}} \dotfill & \pageref{app:broader_impact}\\[0.5em]
\textbf{C \quad Existing Convolutional Architectures as Actualizations of HilbNets} \dotfill & \pageref{sec:existing-architectures}\\[0.5em]
\quad C.1 \quad Universality of HilbNets \dotfill & \pageref{app:univers_hilbnet} \\
\quad\quad C.1.1 \quad CNNs, GNNs, and Sheaf NNs \dotfill & \pageref{app:cnngnnsheaf}\\
\quad\quad C.1.2 \quad Equivariant CNNs and GNNs \dotfill & \pageref{app:equiv}\\
\quad\quad C.1.3 \quad Spatio-Temporal GNNs \dotfill & \pageref{app:stgnn}\\[0.5em]
\textbf{D \quad Practical Implementations of Parallel Transport} \dotfill & \pageref{sec:PT-Algo}\\[0.5em]
\textbf{E \quad Parallel Transport Parametrizations} \dotfill & \pageref{app:ptrans_param}\\[0.5em]
\quad E.1 \quad From Bundle Transports to Network-Sheaf Restrictions \dotfill & \pageref{app:transport:laplacian}\\
\quad E.2 \quad Transport Hypothesis Classes \dotfill & \pageref{app:transport:classes}\\
\quad E.3 \quad Parameter Counts \dotfill & \pageref{app:transport:paramcounts}\\[0.5em]
\textbf{F \quad Additional Experimental Details} \dotfill & \pageref{app:experiments}\\[0.5em]
\quad F.1 \quad Synthetic Experiments: The Statistical Bundle over Centered Gaussians \dotfill & \pageref{app:synthetic}\\
\quad\quad F.1.1 \quad The Bundle \dotfill & \pageref{app:synthetic:bundle}\\
\quad\quad F.1.2 \quad Levi-Civita Connection and Closed-Form Parallel Transport \dotfill & \pageref{app:synthetic:lc}\\
\quad\quad F.1.3 \quad Cholesky Rescaling \dotfill & \pageref{app:synthetic:cholesky}\\
\quad\quad F.1.4 \quad Sample Construction \dotfill & \pageref{app:synthetic:sample}\\
\quad\quad F.1.5 \quad Spectral Stability under Sampling Density Increase \dotfill & \pageref{app:synthetic:spectral}\\
\quad\quad F.1.6 \quad Hyperparameters \dotfill & \pageref{app:synthetic:hyperparameters}\\
\quad F.2 \quad Traffic Forecasting: Experimental Details \dotfill & \pageref{app:traffic}\\
\quad\quad F.2.1 \quad Datasets \dotfill & \pageref{app:traffic:datasets}\\
\quad\quad F.2.2 \quad Task Formulation \dotfill & \pageref{app:traffic:task}\\
\quad\quad F.2.3 \quad Model Variants and Baselines \dotfill & \pageref{app:traffic:models}\\
\quad\quad F.2.4 \quad Detailed Results \dotfill & \pageref{app:traffic:results}\\
\quad\quad F.2.5 \quad Hyperparameters \dotfill & \pageref{app:traffic:hyperparameters}\\[0.5em]
\textbf{G \quad Mathematical Background} \dotfill & \pageref{sec:mathematical-background}\\[0.5em]
\quad G.1 \quad Hilbert Bundles \dotfill & \pageref{sec:hilbert-bundle}\\
\quad G.2 \quad Connection Laplacian \dotfill & \\
\quad G.3 \quad Heat Flow on a Hilbert Bundle \dotfill & \pageref{app:heat_flowhb}\\
\quad G.4 \quad Borel Functional Calculus \dotfill & \pageref{sec:borel-calc}\\
\quad G.5 \quad Cellular Sheaves and Sheaf Laplacians \dotfill & \pageref{sec:cellular-sheaves}\\
\quad G.6 \quad Empirical Laplacians \dotfill & \pageref{app:empirical_lap}\\[0.5em]
\textbf{H \quad Proofs of Results} \dotfill & \pageref{sec:appendix-proofs}\\[0.5em]
\quad H.1 \quad Auxiliary Lemmas for Theorem 1 \dotfill & \\
\quad H.2 \quad Key Lemmas for Theorem 1 \dotfill & \\
\quad H.3 \quad Proof of Theorem 1 \dotfill & \pageref{app:proofth1}\\
\quad H.4 \quad Key Lemmas for Theorem 2 \dotfill & \\
\quad H.5 \quad Proof of Theorem 2 \dotfill & \pageref{app:proofth2}\\
\quad H.6 \quad Key Lemmas for Corollary 1 (Convergence in Architecture) \dotfill & \\
\quad H.7 \quad Proof of Corollary 1 (Convergence in Architecture) \dotfill & \\
\quad H.8 \quad Proof of Corollary 2 (Transferability) \dotfill & \\[0.5em]
\end{tabular}

\section{Extended Related Works}\label{appsec:relwork}
The connection between (possibly) continuous domains (manifolds and bundles) and discrete structure (graphs and cellular sheaves) first emerged in pioneering investigations on the so-called manifold hypothesis. This hypothesis posits that, although data may live in a high-dimensional ambient space, they are effectively generated by sampling from one or several low-dimensional (Riemannian) manifolds \cite{bronstein2017geometric}. The manifold hypothesis underpins several modern spectral graph methods, i.e., nonlinear dimensionality-reduction/clustering/(deep) learning techniques that exploit latent geometric structures. The renowned work  \cite{BELKIN20081289} from Belkin \& Niyogi proved that, assuming access to a finite point cloud (the signals) sampled from the underlying manifold, it is possible to build a weighted undirected graph  whose Laplacian converges to the Laplace-Beltrami operator of the underlying manifold in probability as the number of samples goes to infinity. 

The work in \noindent \cite{BELKIN20081289} and related results, such as \cite{singer2012vdm, singer2017spectral}, have been used (directly or indirectly) to design principled learning systems over manifolds. A consistent fraction of this literature focused on scalar manifold signals, thus the case in which one or more scalar values are attached to each point of a manifold. Notable examples are manifold convolutional neural networks \cite{wang2022convolution,levie2019transferability}, kernel methods and Gaussian processes on manifolds \cite{borovitskiy2020matern, mostowsky2024geometrickernels}, as well as a growing literature of generative manifold models \cite{zaghen2025towards, de2022riemannian}. In a complementary direction, operator-learning methods on manifolds extend neural operators beyond Euclidean domains; these methods handle an infinite-dimensional object globally, but they still assign a finite vector to each point of a manifold. Most of the works in this class are instances of neural manifold operators  \cite{chen2023learning, pmlr-v202-bonev23a,wu2024nmo, jiao2024unknown}, which aim at resolution-independent learning of PDE solution operators. Some works explored vector-valued manifold signals, i.e., multivariate real-valued functions supported on manifolds; in this case, one or more finite vectors are attached to each point of a manifold. Examples are tangent bundle convolutional neural networks \cite{battiloro2023tangent} and vector-field Gaussian processes on manifolds built via gauge-independent projected kernels \cite{borovitskiy2021vector}. Moreover,  especially in the statistics community,  functional observations with manifold structure, i.e., manifold-valued functions supported on the real line \cameraready{or real-valued functions supported on manifolds}, have been long studied \cite{sangalli2016smooth,dai2018pca,shao2022intrinsic}, and recent works have started analyzing autoregressive processes on the sphere \cite{caponera2021asymptotics,spoto2025change}. Finally, learning systems acting on discrete bundles, i.e., bundles whose base space is a finite set/discrete manifold, have been recently investigated \cite{bamberger2025bundle}. Despite their diversity, all the models cited in this paragraph use finite-dimensional fibers and implicitly assume the Levi-Civita connection. As such, they do not allow for the arbitrary connections or the potentially infinite-dimensional signals considered. One of the main reasons behind this gap is the lack of a rigorous generalization of the  \cite{BELKIN20081289}'s convergence result in these settings, which is our main contribution.

\noindent Pioneering works on sheaf theory can be found in \cite{Leray1946,serre1955faisceaux, grothendieck1955general}.\textit{ Cellular sheaves} are combinatorial instances of sheaves that have been introduced in \cite{shepard1985cellular} and later rediscovered in \cite{curry2014sheaves}. In \cite{shepard1985cellular,curry2014sheaves}, these sheaves were first defined over regular cell complexes, hence the term ``cellular'' sheaves. However, as in this work, cellular sheaves are often defined over tamer objects, here graphs. In \cite{hansen2019sheafsp, dinino2025connection}, the authors studied the problem of learning \textit{vector cellular sheaves}, i.e., cellular sheaves over undirected graphs with finite-dimensional node signals. The works in \cite{hansen2020opinion,Hansen2019towardspecsheaf,ghrist2022cellular,riess2022diffusion} introduced a novel class of diffusion dynamics on vector cellular sheaves. In \cite{bodnar2022sheafdiff,hansen2020sheafnn,barbero2022sheafnnconn, fiorini2025sheaves, duta2023sheaf, peng2026sheaf, battiloro2022tangent,battiloro2023tangent}, neural networks operating on vector cellular sheaves over (undirected, directed, hyper) graphs with finite-dimensional signals are presented, generalizing graph neural networks. We again note however, that all these works implicitly or explicitly restrict to consider either the Levi-Civita or flat connections.  
Additionally, the work in \cite{bodnar2022sheafdiff} exploited vector cellular sheaf theory to show that the underlying geometry of the graph gives rise to oversmoothing behavior of GCNs. Also, (vector and general) cellular sheaves recently appeared in causal theory \cite{dacunto2025rck}, control \cite{hanks2025distributed}, and telecommunications \cite{grimaldi2025sheafcomm}. Finally, the works in \cite{battiloro2022tangent,battiloro2023tangent} showed that neural networks for tangent bundle signals can be implemented as certain sheaf neural networks operating on vector cellular sheaves from manifold samples. 

\section{\texorpdfstring{\cameraready{Broader Impact and Future Directions}}{Broader Impact and Future Directions}}\label{app:broader_impact}
The potential impact of this work extends well beyond the effectiveness of the HilbNet architecture. Our convergence result unifies and extends the graph- and vector-diffusion convergence theories of \cite{BELKIN20081289,singer2017spectral}, thereby enabling novel geometric learning systems for genuinely infinite-dimensional manifold-supported  and equipped with arbitrary connections. HilbNets are just a first (principled, transferable) instance of such systems, but our result opens several new avenues.

\textbf{Clustering and Dimensionality Reduction.} Classical Laplacian-based methods for clustering \cite{von2008consistency} and nonlinear dimensionality reduction  \cite{belkin2001laplacian,coifman2006diffusion} all rely, either explicitly or implicitly, on the convergence of the graph Laplacian to the Laplace–Beltrami operator. Our Theorem \ref{thm: main-thm} provides the analogous foundation in the Hilbert bundle setting, immediately suggesting sheaf-spectral generalizations of these techniques. For instance, one may define Hilbert sheaf eigenmaps by computing the leading eigensections of $\Delta_{\mathcal{F}_n^t}$ and using them as coordinates, yielding embeddings that are aware not only of the base manifold geometry but also of the fiber-wise coupling encoded by the connection. This is particularly promising for data such as spatiotemporal fields or distributional signals, where standard spectral methods discard the internal structure of each observation. Similarly, sheaf-spectral clustering would partition data by jointly considering geometric proximity on $\mathcal{M}$ and coherence of the infinite-dimensional signals across fibers, a strictly richer criterion than what scalar graph Laplacians can capture. The finite-rank convergence guarantee of Theorem \ref{thm: diagonal finite rank convergence} ensures that such methods can be implemented with truncated signals while remaining provably consistent with the underlying continuous geometry.

\textbf{Structured Self-Supervised Learning.} Self-supervised learning (SSL) has largely been built around objectives that encourage invariance or equivariance with respect to augmentations of the base domain \cite{you2020graph, thakoor2022largescale, dangovski2021equivariant, yu2024self}. Our framework suggests a more structured family of SSL methods. Because the Hilbert sheaf Laplacian encodes both spatial geometry and fiber-wise transport, one can design contrastive or non-contrastive objectives that encourage learned representations to be sections of an appropriate bundle, i.e., to satisfy local consistency constraints dictated by the restriction maps. Such objectives would yield representations that are not merely invariant to domain augmentations but are geometrically coherent across fibers, a property that is especially desirable when downstream tasks depend on the relational structure between signals at different manifold points, as in multi-sensor forecasting or multi-agent coordination.

\textbf{Generalizability Theory and Mechanistic Interpretability for Transformers.} Another promising direction concerns the connection between our framework and transformer architectures \cite{vaswani2017attention}. In a standard transformer, each token is equipped with a positional encoding, either fixed (e.g., sinusoidal) or learned, that situates it in a continuous geometric space \cite{su2024roformer}. These positional encodings can be viewed as sampled points on a base manifold $\mathcal{M}$, with the encoding scheme implicitly defining the metric structure of the domain. The residual-stream representation at each position, or, in the infinite-width or infinite-context limit, the full distribution over possible activations, then lives in a Hilbert space fibered over this base point, so that the collection of representations across positions constitutes a section of a Hilbert bundle over $\mathcal{M}$. The attention mechanism then defines a data-dependent transport between these fibers. In this view, a self-attention layer is an instance of a single-step diffusion under a learned sheaf Laplacian whose base graph is determined by the sampled positional encodings. This may in some sense be viewed as a more precise incarnation of the recently introduced \textit{geodesic hypothesis} \cite{huang2026semantictubepredictionbeating}, where the autoregressive output of transformers is modeled by a stochastic diffusion PDE in Euclidean space, rather than the proposed manifold-theoretic treatment. 
Making this correspondence precise would allow one to import the convergence and transferability machinery of Theorems \ref{thm: main-thm}--\ref{thm: diagonal finite rank convergence} into the transformer setting. 
We note that there exists a recent line of work that attempts to adapt Laplacian-based GNN generalization and stability results to transformers \cite{porrasvalenzuela2026sizetransferabilitygraphtransformers}, but due to their fundamental reliance on the convegrence result of \cite{BELKIN20081289}, must work in a somewhat simplified setting. In conjunction with our convergence result, the generality of the Hilbert sheaf Laplacian is potentially well-suited for establishing extensions of these generalization theorems for a broader class of transformers.
On the interpretability side, decomposing attention into a positional-affinity component and a fiber-transport component offers a principled lens through which to study what information each head moves and how it is transformed in transit. One could, for example, measure the holonomy of the learned connection around closed loops of attention to detect whether a head implements a nontrivial geometric transformation. While formalizing these connections requires treatment of the data-dependence of the connection and the interplay between positional and content-based attention, the mathematical infrastructure developed in this work provides a strong starting point.

\section{Existing Convolutional Architectures as Actualizations of HilbNets} \label{sec:existing-architectures}

\subsection{Universality of HilbNets} \label{app:univers_hilbnet}
Due to the first-principles approach to the construction of HilbNets, we note that they serve as a sort of universal architecture. That is, several popular variations of convolutional architectures in geometric deep learning, across domains and modalities, can be derived as particular instantiations of HilbNets --- even when their construction does not explicitly invoke cellular sheaves. We consider a few concrete examples of this philosophy. 

\subsubsection{CNNs, GNNs, and Sheaf NNs}\label{app:cnngnnsheaf}  Convolutional neural networks (CNNs) and graph neural networks (GNNs) are both often formalized as acting on signals $f: \mathcal{M} \to \mathbb{R}$. In particular, by the consistency of the discrete Fourier transform, CNNs operating on fixed grid can be viewed as operating on a principled discretization of $\mathcal{M} = [0,1]^2$. We may more generally view the uniform grid on which CNNs act as a particular instantiation of a graph, and thus view CNNs as a special case of GNNs \cite{kipf2016semi}, and the relevant convolutional operator as the graph Laplacian. The GNNs can likewise be viewed as principled discretizations of manifold neural networks \cite{wang2022stability}, precisely by the convergence result of \cite{BELKIN20081289}. Sheaf neural networks \cite{hansen2020sheafnn, bodnar2022sheafdiff} may then be understood as an enrichment that allows for matrix-valued edge weights rather than scalars during convolution. Further, in \cite{battiloro2022tangent}, it was made precise that sheaf neural networks may in particular, be viewed as acting upon tangent bundle signals $s: \mathcal{M} \to T\mathcal{M}$, via the convergence result of \cite{singer2012vdm}. In particular, the convergence result of \cite{singer2012vdm} applies strictly to tangent bundle setting, providing an explanation as to why existing works of sheaf neural networks either explicitly or implicitly restrict to discretizing the tangent bundle with either flat or Levi-Civita connections \cite{barbero2022sheafnnconn, bamberger2025bundle}. As such, existing sheaf neural network architectures can be typically be recovered as HilbNets under the paradigm that $\mathcal{E}=T\mathcal{M}$.

\subsubsection{Equivariant CNNs and GNNs}\label{app:equiv}  The equivariant case of CNNs and GNNs arises when we wish for our architecture to respect some underlying symmetry group $G \curvearrowright \mathcal{M}$. More generally, there may not exist a global representation of the symmetry, but rather only a local representation. In physics, this is known as a \textit{gauge symmetry}, and is formalized by considering our signal $f$ as a section of a bundle with connection $(\mathcal{M},\mathcal{E}, \nabla)$, where the group action is then encoded as a symmetry of the connection $\nabla$. As such, \textit{gauge-equivariant} CNNs and GNNs constitute perhaps the most general equivariant architectures in the literature (see \citet{Weiler2026EquivariantCNNs} for a thorough introduction in the CNN case), and are formulated precisely as convolutions at the level of sections of a frame bundle (although the necessary datum of a connection is often suppressed in the literature). From this perspective, it is natural that gauge-equivariant CNNs and GNNs can be derived as particular cases of cellular sheaf networks (see the \citet{Li2025LearningFromFrustration} for a more in-depth exploration of this perspective). Thus, by noting that these architectures can be equivalently reformulated in the language of sheaves, our Theorem \ref{thm: main-thm}, as well as the consequent transferability result, can be seen to apply to these architectures. In particular, this may be understood as establishing the theoretical bedrock for the intuitive idea that as the underlying mesh or graph is increasingly refined, these architectures indeed approach continuous operators on sections of the underlying bundle, while maintaining equivariance across scales. 

\subsubsection{Spatio-Temporal GNNs}\label{app:stgnn} A formal treatment of signal processing of graphs whose signals at each node are timeseries is still emerging and is an active area of research, and these filtering techniques then serve as the basis for the development of spatiotemporal graph neural networks (STGNNs). In this literature, it is common to consider convolutional operators built via the \textit{joint Laplacian} $L_{J} = L_T \otimes \operatorname{Id}_G + \operatorname{Id}_T \otimes L_G$,  where $L_G$ is the graph-domain Laplacian and $L_T$ is the `time-domain' Laplacian  \cite{Grassi2018TimeNodeSignalProcessing}. Due to this decomposition, the resulting `time' and `space' filters commute, allowing for the development of both \textit{time-and-space} or \textit{time-then-space} STGNNs \cite{marisca2025_2506.15507}. Consider now the continuous setting. Convolutions via $\Delta_\nabla$ intertwine the spatial and temporal domains, and this is precisely encoded by our parallel transport maps. So suppose our bundle is trivial $\mathcal{E} = \mathcal{M} \times L^2(\mathbb{R}^n)$ with trivial connection $\nabla$.  In this case, our parallel transport maps are simply $P_{\gamma} = Id$, and the connection Laplacian collapses to the Laplace-Beltrami operator. On product manifolds $\mathcal{M} = \mathcal{M}_1 \times \mathcal{M}_2$, the Laplace-Beltrami operator decomposes as $\Delta_\mathcal{M} = \Delta_{\mathcal{M}_1} \otimes \operatorname{Id}_{\mathcal{M}_2} + \operatorname{Id}_{\mathcal{M}_1}  \otimes \Delta_{\mathcal{M}_2}$, implying that heat flow is given by $
e^{t \Delta_M}=e^{t \Delta_{M_1}} \otimes e^{t \Delta_{M_2}}
$ (see \citep{Grigoryan2009} for a formal derivation). Expressing $e^{t{\Delta_\mathcal{M}}}$ as an integral operator via the heat kernel, the fact that spatial and temporal filters commute in this case is then simply an application of Fubini's theorem. As such, we see that the type of filtering commonly considered in STGNNs is recovered precisely as the `base case' of HilbNet, and in particular, our robustness guarantees can also be applied to these STGNN architectures.

\section{Practical Implementations of Parallel Transport}
\label{sec:PT-Algo}

As we have established, a key strength of the HilbNets architecture is the ability to encode signal-level geometric priors through the principled incorporation of relevant parallel transport operators. This naturally raises the question as to how these transport operators should be implemented in-practice. We may consider three general classes of use-cases. 

\textbf{Task-Inherent Priors} The most theoretically well-grounded case is when knowledge of the geometry of task itself may be utilized to build our parallel transport operators. For instance, suppose the nodes of our base graph represent cameras and the task is multi-view 3D recognition. Then the relevant transport operators should record the rotation $P_{x_i \to x_j} \in SO(3)$ that aligns views as in \cite{Li2025LearningFromFrustration}, resulting in an appropriately equivariant sheaf Laplacian operator. More generally, whenever the data modality lacks a `global' reference frame or coordinate system, then the appropriate alignments between local reference frames precisely gives rise to a connection and the associated parallel transport. For instance, biomedical timeseries analysis often utilizes algorithms based upon the large deformation diffeomorphic metric mapping (LDDMM) \cite{Beg2005LDDMM}, a core part of which may be understood as extracting the necessary parallel transport from the data using the first-order ODE definition. We may also consider the tangent bundle networks of \cite{battiloro2023tangent} in this category, as they use explicit vector-field data from which they may then compute the necessary sheaf transition maps via local PCA. As such, we see that HilbNets may be applied to any of these settings, where the relevant parallel transport would be completely determined by the task itself and thus, can typically be explicitly pre-computed.  

\textbf{Domain-Inherent Priors} Alternatively, it is often the case that we may not have access to task-specific priors, but rather to general knowledge of the structure of the signal-domain. For instance, in many domains, our generic stalks  may be equipped with the additional structure of a reproducing kernel Hilbert space (RKHS), i.e. $\mathcal{H}_\kappa$. Analogously to the previous case, we may then view parallel transport as operators that that maximize alignment, but now with respect to our kernel. For instance, given a choice of similarity kernel $\kappa$ between timeseries or distributions, then given our initial section data $\{S_u\}_{u=1}^{F_0}$, we may define our parallel transport operator matrices via 
\begin{equation}\label{eqn:kernel_reg}
    P_{x_j\to x_i}^{(d,e_{ij})} := \arg \max_{\mathbf{T} \in \mathcal{C}} \sum^{F_0}_{q=1}\kappa(\mathbf{T}\mathbf{S}_{i,q},\mathbf{S}_{j,q})
\end{equation}for some suitable class of operators $\mathcal{C} \subseteq O(d)$, and force the diagonal blocks of the sheaf Laplacian to be the sum of the scalar edge weights given by the kernel. This is exactly the discretized Sheaf Laplacian from \eqref{eq:discrete_sheaf_laplacian}. In practice, given a choice of similarity kernel on our fibers, we may then either precompute these parallel transport operators using the above optimization objective or learn them end-to-end with the model's learned filters.  In the latter case, \eqref{eqn:kernel_reg} is applied as a regularization to the task loss. The special case in which \eqref{eqn:kernel_reg} is not employed at all, $\mathcal{C} = O(d)$, $K=1$ in \eqref{eqn:discretized-hilbnet},  recovers the sheaf diffusion neural network from \cite{bodnar2022sheafdiff}. As such, we see that the greater generality of the Hilbert sheaf Laplacian consequently lends itself to more flexible and perhaps more broadly applicable design choices than existing sheaf neural networks.

\section{Parallel Transport Parametrizations}\label{app:ptrans_param}

This appendix details the finite-dimensional transport parametrizations used to instantiate the network sheaf Laplacian $\Delta_{\mathcal{F}_{n,d}^t}$ in the experiments. In particular, this particular instantiation of HilbNets may be considered as a paticular of the end-to-end learning paradigm introduced in \ref{sec:PT-Algo} for polynomial filters and a few demonstrative classes of $\mathcal{C}$ and $\kappa$. The discussion should be read as a continuation of the two-stage discretization in Section~\ref{def: discretized-HilbNets}: after sampling the manifold, we obtain a Hilbert cellular sheaf $\mathcal{F}_n^t$; after sampling or projecting the fibers, we obtain a finite-dimensional network sheaf $\mathcal{F}_{n,d}^t$ with $d$-dimensional stalks. 

\subsection{From bundle transports to network-sheaf restrictions}
\label{app:transport:laplacian}

Recall that, before signal discretization, the Hilbert cellular sheaf $\mathcal{F}_n^t$ induced by a sample $\mathcal{X}_n=\{x_1,\dots,x_n\}$ assigns the node stalk
\begin{equation}
\mathcal{F}_n^t(x_i)
=
\mathcal{E}_{x_i}
\end{equation}
and, for an edge $e_{ij}\in $, the edge stalk
\begin{equation}
\mathcal{F}_n^t(e_{ij})
=
\mathcal{E}_{m_{\gamma_{ij}}},
\end{equation}
where $m_{\gamma_{ij}}$ is the midpoint of the chosen geodesic between $x_i$ and $x_j$. Its restriction maps are weighted parallel transports of the form
\begin{equation}
(\mathcal{F}_n^t)_{x_i\leq e_{ij}}
=
\sqrt{k_{ij}^t}\,
P_{x_i\to m_{\gamma_{ij}}},
\qquad
k_{ij}^t
=
\exp
\left(
-\frac{d_\mathcal{M} (x_i,x_j)^2}{4t}
\right). 
\end{equation}
After fiber discretization, the network sheaf $\mathcal{F}_{n,d}^t$ has finite-dimensional stalks, which we identify with $\mathbb{R}^d$ after choosing the first $d$ basis elements of the fiber Hilbert space. The corresponding restriction maps are matrices
\begin{equation}
(\mathcal{F}_{n,d}^t)_{x_i\leq e_{ij}}:
\mathbb{R}^d
\to
\mathbb{R}^d.
\end{equation}

For the pragmatic parametrizations used in the experiments, it is useful to express the same sheaf Laplacian in node-to-node transport coordinates. Fix an orientation convention for each edge $e_{ij}$. After identifying the edge stalk with the coordinate system of one endpoint, we write the restrictions as
\begin{equation}
(\mathcal{F}_{n,d}^t)_{x_i\leq e_{ij}}
=
\sqrt{k_{ij}^t}\,I_d,
\qquad
(\mathcal{F}_{n,d}^t)_{x_j\leq e_{ij}}
=
\sqrt{k_{ij}^t}\,P_{x_j\to x_i}^{(d,e_{ij})},
\label{eq:pragmatic-restrictions}
\end{equation}
where
\begin{equation}
P_{x_j\to x_i}^{(d,e_{ij})}
\in
O(d)
\end{equation}
is the finite-dimensional transport carrying the discretized fiber over $x_j$ into the discretized fiber over $x_i$ along edge $e_{ij}$. When the transport comes from the continuous Hilbert bundle, this matrix is the finite-dimensional representation of the corresponding parallel transport, after the chosen fiber projection and coordinate identification. When the continuous connection is unknown, $P_{x_j\to x_i}^{(d,e_{ij})}$ is instead chosen from a transport hypothesis class.

With the shorthand
\begin{equation}
P_{j\to i}^{e}
:=
P_{x_j\to x_i}^{(d,e_{ij})},
\end{equation}
the action of the network sheaf Laplacian on a sampled signal
\begin{equation}
\mathbf{s}_{n,d}
=
(\mathbf{s}_{x_1},\dots,\mathbf{s}_{x_n})
\in
C^0(\mathcal{F}_{n,d}^t;G_n),
\qquad
\mathbf{s}_{x_i}\in\mathbb{R}^d,
\end{equation}
takes the concrete form
\begin{equation}
(\Delta_{\mathcal{F}_{n,d}^t}\mathbf{s}_{n,d})_{x_i}
=
\sum_{x_j\in\mathcal{N}(x_i)}
k_{ij}^t
\left(
\mathbf{s}_{x_i}
-
P_{j\to i}^{e_{ij}}\mathbf{s}_{x_j}
\right).
\label{eq:pragmatic-laplacian-action}
\end{equation}
Equivalently, $\Delta_{\mathcal{F}_{n,d}^t}\in\mathbb{R}^{nd\times nd}$ is the block matrix with blocks
\begin{equation}
(\Delta_{\mathcal{F}_{n,d}^t})_{ij}
=
\begin{cases}
\displaystyle
\sum_{r:\,e_{ir}\in E}
k_{ir}^t I_d,
& i=j, \\[1.2em]
\displaystyle
-k_{ij}^t P_{j\to i}^{e_{ij}},
& i\neq j \text{ and } e_{ij}\in E, \\[0.5em]
0,
& \text{otherwise.}
\end{cases}
\label{eq:pragmatic-laplacian-blocks}
\end{equation}
This is the same sheaf Laplacian defined in Section~\ref{Hilbert-Sheaf-Laplacian}, specialized to the coordinate convention in~\eqref{eq:pragmatic-restrictions}. The scalar weight $k_{ij}^t$ controls how strongly the two sampled base points interact, while the transport matrix $P_{j\to i}^{e_{ij}}$ controls how the two discretized fibers are aligned before their signals are compared.

\subsection{Transport hypothesis classes}
\label{app:transport:classes}

As mentioned in \ref{sec:PT-Algo}, the true  connection is often unknown and parallel transport maps must be learned. In the experimental results of this work, we always learn the parallel transport maps end-to-end using the task loss regularized with \eqref{eqn:kernel_reg}, and we restrict each edgewise transport to hypothesis classes $\mathcal{C}$ such that
\begin{equation}
P_{j\to i}^{e_{ij}}
\in
\mathcal{C}
\subseteq
O(d).
\end{equation}
We use three transport classes in the experiments: frozen identity, free orthogonal transports, and circulant or time-stationary transports.

\paragraph{Frozen identity.}
The simplest class is
\begin{equation}
\mathcal{C}_{\mathrm{id}}
=
\{I_d\}.
\end{equation}
This recovers the usual assumption that neighboring fibers are canonically identified and that no non-trivial alignment is needed. In this case,
\begin{equation}
P_{j\to i}^{e_{ij}}
=
I_d
\end{equation}
for every edge, and~\eqref{eq:pragmatic-laplacian-action} becomes
\begin{equation}
(\Delta_{\mathcal{F}_{n,d}^t}\mathbf{s}_{n,d})_{x_i}
=
\sum_{x_j\in\mathcal{N}(x_i)}
k_{ij}^t
\left(
\mathbf{s}_{x_i}
-
\mathbf{s}_{x_j}
\right).
\end{equation}
Thus, frozen identity reduces the sheaf Laplacian to a standard weighted graph Laplacian applied independently to each fiber coordinate, and, therefore, HilbNets to standard GCNs. It is a useful baseline: any improvement over frozen identity quantifies the value of learning or imposing non-trivial transports.

\paragraph{Free orthogonal transports.}
The most expressive finite-dimensional class is
\begin{equation}
\mathcal{C}_{\mathrm{free}}
=
O(d),
\end{equation}
or, when the target transports are known to lie in the identity component,
\begin{equation}
\mathcal{C}_{\mathrm{free}}
=
SO(d).
\end{equation}
In the experiments, we parameterize free orthogonal transports by products of Householder reflections.

For a nonzero vector $v\in\mathbb{R}^d$, define the Householder reflection
\begin{equation}
H(v)
=
I_d
-
2\frac{vv^\top}{\|v\|_2^2}.
\end{equation}
Each $H(v)$ is orthogonal and symmetric:
\begin{equation}
H(v)^\top H(v)=I_d,
\qquad
H(v)^\top=H(v).
\end{equation}
For each oriented edge $e_{ij}$, we store $R$ Householder vectors
\begin{equation}
v_{e_{ij},1},\dots,v_{e_{ij},R}
\in
\mathbb{R}^d
\end{equation}
and define
\begin{equation}
P_{j\to i}^{e_{ij}}
=
H(v_{e_{ij},R})
\cdots
H(v_{e_{ij},1}).
\label{eq:householder-transport}
\end{equation}
Therefore $P_{j\to i}^{e_{ij}}$ is exactly orthogonal for every parameter value.

By the Cartan--Dieudonn\'e theorem, every matrix in $O(d)$ can be represented as a product of at most $d$ Householder reflections. In practice, the choice of $R$ is dataset-dependent: in our synthetic experiments we use $R=16$ for $d=m=10$, a modest over-parametrization that aids optimization in our traffic experiments we use $R=8$ for $d=T=12$, which parameterizes a strict Householder subset of $O(T)$ rather than the full orthogonal group, and which we find sufficient for the alignment patterns observed in the data. If one fixes exactly $R$ non-degenerate reflections, the determinant parity is fixed:
\begin{equation}
\det(P_{j\to i}^{e_{ij}})
=
(-1)^R.
\end{equation}
Thus, an even number of reflections parameterizes the identity component $SO(d)$, while an odd number parameterizes the other component. In the synthetic Gaussian experiment, the ground-truth Levi-Civita transports are obtained continuously from the identity along geodesics and, after Cholesky rescaling, lie in the identity component. Hence an even number of reflections is appropriate. If both connected components of $O(d)$ are needed, one may add a fixed final reflection or a discrete sign component. For numerical stability, the implementation uses
\begin{equation}
H(v_{e_{ij},r})
=
I_d
-
2\frac{v_{e_{ij},r}v_{e_{ij},r}^\top}
{\|v_{e_{ij},r}\|_2^2 + \epsilon},
\end{equation}
with a small $\epsilon>0$ to avoid division by zero at degenerate $v_{e_{ij},r}$. This recovers the exact Householder reflection $H(q) = I_d - 2qq^\top$ with $q = v_{e_{ij},r}/\|v_{e_{ij},r}\|_2$ in the limit $\epsilon\to 0$ for $\|v_{e_{ij},r}\|_2>0$, and yields a matrix orthogonal up to numerical precision.

The free class is useful as an expressivity test. If the target transport belongs to $O(d)$ in the chosen coordinates, then the Householder class can represent it. This is precisely the role it plays in the synthetic statistical-bundle experiment, where Cholesky rescaling converts the intrinsic Wasserstein-unitary Levi-Civita transports into Euclidean-orthogonal matrices. In real-data experiments, the free class serves as a high-capacity transport baseline.

\paragraph{Circulant or time-stationary transports.}
For time-series fibers, the discretized fiber dimension is the number of retained time samples, so we write $d=T$. A natural prior is that inter-fiber transport should commute with time shifts. Let
\begin{equation}
\mathsf{S}_T:\mathbb{R}^T\to\mathbb{R}^T
\end{equation}
be the cyclic shift operator. A time-stationary transport is one satisfying
\begin{equation}
P_{j\to i}^{e_{ij}}\mathsf{S}_T
=
\mathsf{S}_T P_{j\to i}^{e_{ij}}.
\end{equation}
The commutant of the cyclic shift is the algebra of circulant matrices. Requiring in addition that $P_{j\to i}^{e_{ij}}$ be orthogonal gives the class of orthogonal circulant transports.

Let $F_T$ denote the unitary discrete Fourier transform matrix. Then every orthogonal circulant transport has the form
\begin{equation}
P_{j\to i}^{e_{ij}}
=
F_T^{*}
\operatorname{diag}(\lambda_{e_{ij}})
F_T,
\qquad
|\lambda_{e_{ij},k}|=1.
\label{eq:circulant-fourier-form}
\end{equation}
For real-valued time-domain signals, the Fourier multipliers must satisfy conjugate symmetry:
\begin{equation}
\lambda_{e_{ij},T-k}
=
\overline{\lambda_{e_{ij},k}}.
\end{equation}
We therefore store only the independent positive-frequency phases. Let
\begin{equation}
m_T
=
\left\lfloor
\frac{T-1}{2}
\right\rfloor.
\end{equation}
The learnable parameter for edge $e_{ij}$ is
\begin{equation}
\varphi_{e_{ij}}
=
(\varphi_{e_{ij},1},\dots,\varphi_{e_{ij},m_T})
\in
\mathbb{R}^{m_T}.
\end{equation}
We define
\begin{equation}
\lambda_{e_{ij},0}=1,
\qquad
\lambda_{e_{ij},k}=e^{i\varphi_{e_{ij},k}},
\qquad
\lambda_{e_{ij},T-k}=e^{-i\varphi_{e_{ij},k}},
\quad
k=1,\dots,m_T.
\end{equation}
If $T$ is even, the Nyquist frequency is self-conjugate and is fixed to
\begin{equation}
\lambda_{e_{ij},T/2}=1
\end{equation}
for the identity-component parametrization. This yields a real orthogonal circulant matrix through~\eqref{eq:circulant-fourier-form}. Equivalently, $P_{j\to i}^{e_{ij}}$ can be constructed in real arithmetic from its first column. For $r=0,\dots,T-1$, define
\begin{equation}
c_{e_{ij}}[r]
=
\frac{1}{T}
\left[
1
+
\mathbb{I}_{\{T\text{ even}\}}(-1)^r
+
2\sum_{k=1}^{m_T}
\cos\left(
\varphi_{e_{ij},k}
+
\frac{2\pi kr}{T}
\right)
\right].
\end{equation}
The full circulant matrix is then
\begin{equation}
(P_{j\to i}^{e_{ij}})_{ab}
=
c_{e_{ij}}[(a-b)\bmod T],
\qquad
a,b=0,\dots,T-1.
\end{equation}
This form is convenient for implementation because it avoids explicitly manipulating complex-valued matrices. The circulant class has only
\begin{equation}
m_T
=
\left\lfloor
\frac{T-1}{2}
\right\rfloor
\end{equation}
parameters per edge, compared with $d(d-1)/2$ degrees of freedom for a general orthogonal matrix. Each phase $\varphi_{e_{ij},k}$ has a direct interpretation as the phase lag at frequency $k$ between the two endpoint fibers. Thus, the transport may advance or delay oscillatory components across an edge, but it cannot arbitrarily mix frequencies or reshape the waveform. This is the intended inductive bias for spatiotemporal signals such as traffic or sensor time series, where neighboring sensors may observe delayed or phase-shifted versions of related temporal patterns. In the synthetic experiment, the same class is used more abstractly as a structured subgroup of $O(d)$ against which the ground-truth transports can be projected.

\subsection{Parameter counts}
\label{app:transport:paramcounts}

For a graph $G_n=(\mathcal{X}_n,E)$ with $|E|$ undirected edges and discretized fiber dimension $d$, the transport parameter counts are summarized in Table~\ref{tab:transport_param_counts}.

\begin{table}[h]
\centering
\small
\begin{tabular}{lll}
\toprule
Transport class & Parameters per edge & Interpretation \\
\midrule
Frozen identity & $0$ & no learned alignment \\
Free Householder & $R d$ & product of $R$ reflections in $O(d)$ \\
Full circulant & $\lfloor(d-1)/2\rfloor$ & one phase per positive frequency \\
\bottomrule
\end{tabular}
\caption{Parameter counts for the finite-dimensional transport classes used in $\mathcal{F}_{n,d}^t$.}
\label{tab:transport_param_counts}
\end{table}

Thus, the free class is maximally expressive but parameter-heavy, while the circulant classes encode a strong time-stationary prior and scale linearly with the number of frequencies or bands.

\section{Additional Experimental Details}\label{app:experiments}
 
\subsection{Synthetic experiments: the statistical bundle over centered Gaussians}
\label{app:synthetic}
 
\subsubsection{The bundle}
\label{app:synthetic:bundle}
 
The base manifold is $\mathcal{M}=\mathrm{Sym}^{++}(p)$, the open cone of $p\times p$ symmetric positive-definite matrices. Each $\Sigma\in\mathcal{M}$ parameterizes a centered Gaussian $\mathcal{N}(0,\Sigma)$ on $\mathbb{R}^p$ with density $\rho_\Sigma$. We equip $\mathcal{M}$ with the Otto-Wasserstein metric, namely the Riemannian metric induced on $\mathrm{Sym}^{++}(p)$ by the optimal-transport distance $W_2$ between centered Gaussian measures.
 
Concretely, the tangent space is naturally identified with symmetric matrices,
\begin{equation}
T_\Sigma\mathcal{M}
\cong
\mathrm{Sym}(p),
\end{equation}
and the Otto-Wasserstein inner product between $U,V\in\mathrm{Sym}(p)$ is
\begin{equation}
W_\Sigma(U,V)
=
\tfrac{1}{2}\,
\mathrm{Tr}\bigl(L_\Sigma[U]\,V\bigr),
\qquad
L_\Sigma[U]\Sigma+\Sigma L_\Sigma[U]=U,
\label{eq:otto_metric_app}
\end{equation}
where $L_\Sigma[U]$ is the unique symmetric solution of the Lyapunov equation.
 
Above each $\Sigma$, the ambient Hilbert fiber is the vector-field space
\begin{equation}
\mathcal{H}_\Sigma
:=
L^2(\rho_\Sigma;\mathbb{R}^p),
\end{equation}
equipped with the inner product
\begin{equation}
\langle a,b\rangle_{\mathcal{H}_\Sigma}
=
\int_{\mathbb{R}^p}
a(x)^\top b(x)\rho_\Sigma(x)\,dx.
\end{equation}
This fiber is genuinely infinite-dimensional. The finite-rank fiber used in the synthetic experiments is the Otto-velocity image of covariance perturbations:
\begin{equation}
\mathcal{E}_\Sigma
:=
\bigl\{
v_V(x)=L_\Sigma[V]x
:
V\in\mathrm{Sym}(p)
\bigr\}
\subset
L^2(\rho_\Sigma;\mathbb{R}^p).
\label{eq:otto_velocity_fiber}
\end{equation}
Thus, the computational fiber $\mathcal{E}_\Sigma$ is a finite-dimensional statistical subspace of the ambient Hilbert fiber. The map $V\mapsto v_V$ is an isometry between $\mathrm{Sym}(p)$ with the Otto-Wasserstein metric and $\mathcal{E}_\Sigma$ with the $L^2(\rho_\Sigma;\mathbb{R}^p)$ inner product. Indeed, for $U,V\in\mathrm{Sym}(p)$,
\begin{equation}
\langle v_U,v_V\rangle_{\mathcal{H}_\Sigma}
=
\mathbb{E}_{x\sim\mathcal{N}(0,\Sigma)}
\bigl[
x^\top L_\Sigma[U]L_\Sigma[V]x
\bigr]
=
\mathrm{Tr}\bigl(L_\Sigma[U]L_\Sigma[V]\Sigma\bigr).
\end{equation}
Using $V=L_\Sigma[V]\Sigma+\Sigma L_\Sigma[V]$ and the fact that $L_\Sigma[U]$, $L_\Sigma[V]$, and $\Sigma$ are symmetric, this equals
\begin{equation}
\tfrac{1}{2}
\mathrm{Tr}\bigl(L_\Sigma[U]V\bigr)
=
W_\Sigma(U,V).
\end{equation}
Therefore,
\begin{equation}
\dim \mathcal{E}_\Sigma
=
\dim \mathrm{Sym}(p)
=
d
=
p(p+1)/2.
\end{equation}
Since the fibers $\mathcal{E}_\Sigma$ are already $m$-dimensional, the fiber discretization of Proposition~\ref{prop:signal_discretization} is exact with $d=m$. Throughout this appendix, we therefore write $d=m$ and use the network sheaf notation $\mathcal{F}_{n,d}^t$ and Laplacian $\Delta_{\mathcal{F}_{n,d}^t}\in\mathbb{R}^{nd\times nd}$ from Section~\ref{def: discretized-HilbNets}.
 
This construction is useful because it gives a faithful but tractable proxy for the Hilbert-bundle settings that motivate HilbNets. It is \emph{faithful} in the sense that the ambient fibers $L^2(\rho_\Sigma;\mathbb{R}^p)$ are infinite-dimensional vector-field Hilbert spaces, and the Levi-Civita connection of $(\mathrm{Sym}^{++}(p),W_\Sigma)$ yields non-trivial, metric-compatible parallel transports. It is \emph{tractable} because the Otto-velocity map selects a finite-rank statistical sub-bundle on which the metric, parallel-transport ODE, and projection of ground-truth transports onto restricted transport classes admit closed-form numerical evaluation. Thus, the experiments probe the finite-rank computational slice used by the implementation. 
\subsubsection{Levi-Civita connection and closed-form parallel transport}
\label{app:synthetic:lc}
 
The Levi-Civita connection on $(\mathrm{Sym}^{++}(p),W_\Sigma)$ is the canonical metric-compatible torsion-free connection associated with the Otto-Wasserstein metric. In covariance coordinates, its Christoffel symbol is
\begin{equation}
\Gamma_\Sigma(U,V)
=
\tfrac{1}{2}
\bigl(
U L_\Sigma[V]
+
V L_\Sigma[U]
\bigr)
\in
\mathrm{Sym}(p),
\label{eq:christoffel_app}
\end{equation}
which is symmetric in $(U,V)$, as required for a torsion-free connection.
 
Let $\Sigma_t$ denote the Wasserstein geodesic from $\Sigma_0$ to $\Sigma_1$. Parallel transport of a tangent vector $V(t)\in\mathrm{Sym}(p)$ along $\Sigma_t$ is governed by
\begin{equation}
\dot V(t)
=
-
\Gamma_{\Sigma_t}
\bigl(
\dot\Sigma_t,V(t)
\bigr),
\qquad
V(0)=V_0.
\label{eq:parallel_transport_ode_app}
\end{equation}
We solve this ODE numerically by Euler integration with $50$ steps. The resulting linear map on the finite-rank fiber is the ground-truth Levi-Civita transport
\begin{equation}
P^{LC}_{x_i\to x_j}:\mathrm{Sym}(p)\to\mathrm{Sym}(p).
\end{equation}
In the notation of Def,~\ref{def:sheaf-from-bundle}, the restriction maps of the induced sheaf $\mathcal{F}_{n,d}^t$ use the midpoint transports $P^{LC}_{x_i\to m_{ij}}$, which play the role of the discretized parallel transport $P_{x_i\to m_{ij}}^{(d)}$ with $d=m$. These midpoint transports define the restriction maps used in the spectral-stability experiments. The transport-recovery experiments instead regress against the Cholesky-rescaled node-to-node transport $\tilde P^{LC}_{x_j\to x_i}$, which adopts the orientation convention of Appendix~\ref{sec:PT-Algo}.
 
Because the connection is metric-compatible, $P^{LC}_{x_i\to x_j}$ is unitary with respect to the Otto-Wasserstein inner products on the source and target fibers:
\begin{equation}
W_{x_j}
\bigl(
P^{LC}_{x_i\to x_j}U,
P^{LC}_{x_i\to x_j}V
\bigr)
=
W_{x_i}(U,V).
\label{eq:wasserstein_unitarity_app}
\end{equation}
We verify this numerically to within $0.5\%$ using $200$ Euler steps. This Wasserstein-unitarity is the geometric invariant that justifies comparing the ground-truth transports to orthogonal parametrizations after metric rescaling.
 
\subsubsection{Cholesky rescaling}
\label{app:synthetic:cholesky}
 
The free-$O(m)$ transport class used in the implementation is Euclidean-orthogonal in vectorized fiber coordinates (cf.\ the hypothesis classes in Appendix~\ref{sec:PT-Algo} with $d=m$). However, the intrinsic fiber metric is $W_\Sigma$, not the raw Frobenius metric on $\mathrm{Sym}(p)$. Therefore, $P^{LC}_{x_i\to x_j}$ is not generally orthogonal in raw coordinates.
 
Let $G_\Sigma$ be the Gram matrix of $W_\Sigma$ in a fixed basis of $\mathrm{Sym}(p)$. We factor
\begin{equation}
G_\Sigma
=
R_\Sigma^\top R_\Sigma
\end{equation}
by Cholesky decomposition and represent a fiber coordinate vector $u$ in the rescaled frame as
\begin{equation}
\tilde u
=
R_\Sigma u.
\end{equation}
In this frame, the rescaled Levi-Civita transport is
\begin{equation}
\tilde P^{LC}_{x_i\to x_j}
=
R_{x_j}
P^{LC}_{x_i\to x_j}
R_{x_i}^{-1}.
\label{eq:rescaled_transport_app}
\end{equation}
By Wasserstein-unitarity, it satisfies
\begin{equation}
\bigl(\tilde P^{LC}_{x_i\to x_j}\bigr)^\top
\tilde P^{LC}_{x_i\to x_j}
=
I_m.
\end{equation}
Hence,
\begin{equation}
\tilde P^{LC}_{x_i\to x_j}\in O(m).
\end{equation}
This is the coordinate system used in the transport-recovery experiments. In these coordinates, the free-$O(m)$ Householder class described in Appendix~\ref{sec:PT-Algo} contains the ground-truth transports. The spectral-stability metrics are computed from the assembled sheaf Laplacian $\Delta_{\mathcal{F}_{n,d}^t}$ and are invariant to this coordinate choice up to similarity transformation.
 
\subsubsection{Sample construction}
\label{app:synthetic:sample}
 
We draw samples
\begin{equation}
\Sigma_i
=
R_iD_iR_i^\top,
\qquad
i=1,\dots,n,
\end{equation}
where $R_i$ is a Haar-random $p\times p$ orthogonal matrix, obtained by QR decomposition of a standard-normal matrix, and $D_i$ is diagonal with log-uniform spectrum on $[\log 0.5,\log 2.0]$. Equivalently, the eigenvalues of $\Sigma_i$ lie in $[0.5,2.0]$ on a log-uniform scale.
 
Although $\mathrm{Sym}^{++}(p)$ is non-compact, this procedure samples a bounded subset of it. This is appropriate for the finite deployment-regime stability tests reported here, but it is not the normalized-volume sampling assumption used in the asymptotic convergence theorem.
 
We build a $k$NN graph $G_n=(\mathcal{X}_n,E)$ with $k=8$ under the Gaussian Wasserstein distance $W_2(\Sigma_i,\Sigma_j)$ and assign Gaussian-kernel weights
\begin{equation}
k_{ij}^{t}
=
\exp
\left(
-\frac{W_2(\Sigma_i,\Sigma_j)^2}{4t}
\right),
\qquad
t=0.5.
\end{equation}
The induced network sheaf $\mathcal{F}_{n,d}^t$ has per-edge restriction maps
\begin{equation}
(\mathcal{F}_{n,d}^t)_{x_i\leq e_{ij}}
=
\sqrt{k_{ij}^{t}}\,
P^{LC}_{x_i\to m_{ij}},
\end{equation}
as in Def,~\ref{def:sheaf-from-bundle}, where $m_{ij}$ is the Wasserstein-geodesic midpoint between $\Sigma_i$ and $\Sigma_j$.
 
\subsubsection{Spectral stability under sampling density increase}
\label{app:synthetic:spectral}

For each Gaussian dimension $p$, sample size $n\in\{50,100,200,400,800\}$, and random seed, we sample
$\mathcal{X}_n\subset\mathrm{Sym}^{++}(p)$, build $\mathcal{F}_{n,d}^t$ using the closed-form Levi-Civita transports, and assemble the sheaf Laplacian
\begin{equation}
\Delta_{\mathcal{F}_{n,d}^t}
\in
\mathbb{R}^{nd\times nd}
\end{equation}
as a sparse block matrix. Since the sheaf Laplacian has size $n_{\max} m \times n_{\max} m$ and we require a dense eigendecomposition, we choose $n_{\max}$ so that $n_{\max}\cdot d\approx 10^4$ remains tractable on a single CPU node, yielding $n_{\max}\in\{4000,2000,1000\}$ for $p\in\{2,3,4\}$ (i.e., $d\in\{3,6,10\}$), respectively.

Let
\begin{equation}
\lambda_1^{(n)}
\le
\dots
\le
\lambda_k^{(n)}
\end{equation}
denote the bottom-$k$ eigenvalues of $\Delta_{\mathcal{F}_{n,d}^t}$, with $k=32$, and define $\lambda_i^{(n_{\max})}$ analogously for the reference operator $\Delta_{\mathcal{F}_{n_{\max},m}^t}$. We measure both the aggregate $\ell_2$ and worst-case relative spectral discrepancy of the bottom-$32$ eigenvalues of $\Delta_{\mathcal{F}_{n,d}^t}$ against a high-resolution reference, sweeping $p\in\{2,3,4\}$ and $n\in\{50,100,200,400,800\}$, and averaging over 5 sampling realizations. Fig.~\ref{fig:op_conv} shows a monotone decrease for both metrics across all dimensions, demonstrating that the sheaf Laplacian stabilizes  as manifold sampling density increases, and faster for higher signal sampling densities.
\begin{figure}
    \centering
    \includegraphics[width=1\linewidth]{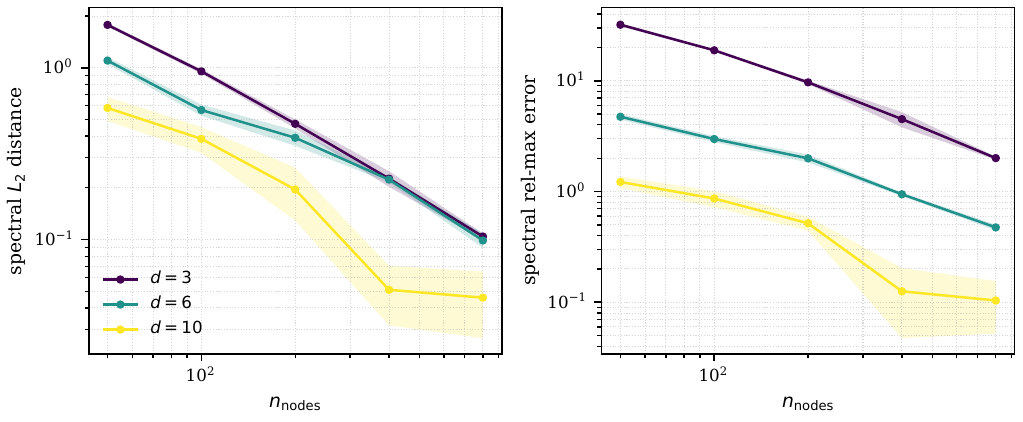}
    \caption{Spectral stability of $\Delta_{\mathcal{F}_{n,d}^t}$ for different signal and manifold sampling densitites. \textbf{Left}: aggregate $\ell_2$ eigenvalues discrepancy. \textbf{Right}: worst-case relative error.}
    \label{fig:op_conv}
\end{figure}
Moreover,

\paragraph{Aggregate discrepancy.} Fig.~\ref{fig:op_conv} (Left) reports the low-frequency spectral $\ell_2$ discrepancy
\begin{equation}
\mathrm{spec\text{-}L_2}(n)
=
\frac{1}{k}
\left(
\sum_{i=1}^{k}
\bigl(
\lambda_i^{(n)}
-
\lambda_i^{(n_{\max})}
\bigr)^2
\right)^{1/2}.
\label{eq:spec_l2_app}
\end{equation}
We average across three seeds and report $\bar x\pm s$ as the shaded band.

\paragraph{Worst-case discrepancy.} Fig.~\ref{fig:op_conv} (Right) reports the relative max error over the bottom-$k$ eigenvalues:
\begin{equation}
\mathrm{spec\text{-}rel\text{-}max}(n)
=
\frac{
\max_{1\le i\le k}
\left|
\lambda_i^{(n)}
-
\lambda_i^{(n_{\max})}
\right|
}{
\lambda_k^{(n_{\max})}
}.
\label{eq:spec_rel_max_app}
\end{equation}
This complements the aggregate metric by capturing worst-case low-frequency spectral error.

Overall, Fig.~\ref{fig:op_conv} shows a monotone decrease for both metrics across all dimensions, demonstrating that the sheaf Laplacian stabilizes  as manifold sampling density increases, and faster for higher signal sampling densities.
 
We train three transport parametrizations from Appendix~\ref{sec:PT-Algo}, free $O(d)$ (Householder), circulant, and frozen identity, to recover the Levi-Civita transports in Cholesky-rescaled coordinates by minimizing the per-edge transport-MSE loss
\begin{equation}
\mathcal{L}
=
\frac{1}{|E|}
\sum_{e_{ij}\in E}
\mathbb{E}_{V\sim\mathcal{N}(0,I_d)}
\left\|
P_{j\to i}^{e_{ij}}V
-
\tilde P^{LC}_{x_j\to x_i}V
\right\|_2^{2}.
\label{eq:free_transport_mse_app}
\end{equation}
Here $P_{j\to i}^{e_{ij}}\in O(d)$ is the rescaled transport produced by the model for edge $e_{ij}$, $\tilde P^{LC}_{x_j\to x_i}$ is the ground-truth Levi-Civita transport from $x_j$ to $x_i$ in Cholesky-rescaled coordinates, and $V$ is a fresh isotropic Gaussian test vector.
 
By the trace identity
\begin{equation}
\mathbb{E}_{V\sim\mathcal{N}(0,I_d)}
\|AV\|_2^2
=
\|A\|_F^2,
\end{equation}
the population loss equals the mean per-edge squared Frobenius distance
\begin{equation}
\mathcal{L}^*
=
\frac{1}{|E|}
\sum_{e_{ij}\in E}
\left\|
P_{j\to i}^{e_{ij}}
-
\tilde P^{LC}_{x_j\to x_i}
\right\|_F^{2}.
\label{eq:free_transport_frob_app}
\end{equation}
 
\paragraph{Free $O(d)$.}
The free transport is parameterized by products of Householder reflections, as detailed in Appendix~\ref{sec:PT-Algo}. Since $\tilde P^{LC}_{x_j\to x_i}\in O(d)$ in the Cholesky-rescaled frame, the free-$O(d)$ class contains the target transport and the minimum population loss is zero. Empirically, we observe \cameraready{$\mathcal{L}^*/d\approx 2\cdot 10^{-7}$} at convergence, confirming recovery to numerical precision.

\paragraph{Restricted classes and analytical projections.}
For each restricted transport hypothesis class $\mathcal{C}\subseteq O(d)$, the population loss minimum is the mean Frobenius distance from each ground-truth transport $\tilde P^{LC}_{x_j\to x_i}$ to its best approximation inside $\mathcal{C}$:
\begin{equation}
\mathcal{L}^*_{\mathcal{C}}
=
\frac{1}{|E|}
\sum_{e_{ij}\in E}
\min_{T\in\mathcal{C}}
\left\|
T-\tilde P^{LC}_{x_j\to x_i}
\right\|_F^{2}
=
\frac{1}{|E|}
\sum_{e_{ij}\in E}
\left\|
\tilde P^{LC}_{x_j\to x_i}
-
\mathrm{proj}_{\mathcal{C}}
\bigl(
\tilde P^{LC}_{x_j\to x_i}
\bigr)
\right\|_F^{2}.
\label{eq:restricted_class_projection_app}
\end{equation}
This\cameraready{, divided by $d$,} is the Theory column in Table~\ref{tab:transport_recovery}.

\emph{Frozen identity.}
For $\mathcal{C}_{\mathrm{id}}=\{I_d\}$, the projection is trivial: $\mathrm{proj}_{\{I\}}(\tilde P^{LC}_{x_j\to x_i})=I_d$, so
\begin{equation}
\mathcal{L}^*_{\mathrm{frozen}}
=
\frac{1}{|E|}
\sum_{e_{ij}\in E}
\left\|
\tilde P^{LC}_{x_j\to x_i}-I_d
\right\|_F^{2}.
\end{equation}
 
\emph{Circulant.}
For the circulant class (Appendix~\ref{sec:PT-Algo} with $d=m$), let $F_d$ be the $d\times d$ DFT matrix and define
\begin{equation}
\mathcal{C}_{\mathrm{circ}}
=
\left\{
F_d^{*}
\mathrm{diag}
(e^{i\varphi_1},\dots,e^{i\varphi_d})
F_m
:
\varphi_k\in\mathbb{R}
\right\}.
\end{equation}
The Frobenius-best projection of any $T\in O(d)$ onto this class is the diagonal-phase Procrustes solution
\begin{equation}
\mathrm{proj}_{\mathrm{circ}}(T)
=
F_m^{*}
\mathrm{diag}
(e^{i\hat\varphi_1},\dots,e^{i\hat\varphi_d})
F_d,
\qquad
\hat\varphi_k
=
\arg
\bigl(
(F_dTF_d^{*})_{kk}
\bigr).
\label{eq:circ_projection_app}
\end{equation}
The zero-frequency phase is pinned to zero to preserve the constant mode, and the self-conjugate Nyquist frequency, when present, is handled according to the real-valued convention of Appendix~\ref{sec:PT-Algo}. In the synthetic experiment the circulant class is used as a structured subgroup of $O(d)$ for testing transport-class projection, rather than as a time-series prior.
 
\paragraph{Empirical vs.\ theoretical plateaus.}
\cameraready{All entries of Table~\ref{tab:transport_recovery} are normalized per fiber coordinate, i.e., divided by $d$. The Det.\ column is the lowest value of $\mathcal{L}^*/d$ in \eqref{eq:free_transport_frob_app}, evaluated at the learned transports, observed over the $5000$-epoch budget, while the Theory column is $\mathcal{L}^*_{\mathcal{C}}/d$ computed on the same edge set used during training. Det.\ reports means $\pm$ standard deviations and Theory reports means across three seeds. Since every transport in $\mathcal{C}$ satisfies $\mathcal{L}^*\geq\mathcal{L}^*_{\mathcal{C}}$, Det.\ is bounded below by Theory by construction: it lies at most $0.05\%$ above it for the circulant class and coincides with it, up to floating-point precision, for frozen identity.} This confirms that restricting the transport class to $\mathcal{C}\subseteq O(d)$ constrains the per-edge restriction maps of $\Delta_{\mathcal{F}_{n,d}^t}$ in a quantitatively predictable way. We do not claim that such an arbitrary edgewise subgroup constraint automatically lifts to a smooth global connection class on the continuous bundle.
 
\subsubsection{Hyperparameters}
\label{app:synthetic:hyperparameters}
See Table \ref{tab:synthetic_hp}.
 
\begin{table}[h]
\centering
\small
\begin{tabular}{lll}
\toprule
& Spectral stability & Transport recovery \\
\midrule
Gaussian dimension $p$    & $\{2,3,4\}$ ($d\in\{3,6,10\}$)                     & $4$ ($d{=}10$) \\
sample-size grid $n$      & $\{50,100,200,400,800\}$                           & $\{16,32,64,128,256\}$ \\
reference $n_{\max}$      & $\{4000,2000,1000\}$ at $p\in\{2,3,4\}$            & --- \\
graph                     & kNN, $k{=}8$, $W_2$ distance                       & kNN, $k{=}8$, $W_2$ distance \\
top-$k$ eigenvalues       & $32$                                               & --- \\
seeds                     & $3$                                                & $3$ \\
epochs / patience         & ---                                                & $5000$ / $600$ \\
optimizer                 & ---                                                & Adam, lr $5\!\cdot\!10^{-3}$, batch $256$ \\
Householder reflections   & ---                                                & $R{=}16$ \\
Euler steps for $P^{LC}$  & $50$                                               & $50$ \\
\bottomrule
\end{tabular}
\caption{Synthetic experiments: hyperparameters for spectral stability and transport recovery.}
\label{tab:synthetic_hp}
\end{table}

\subsection{Traffic forecasting: experimental details}
\label{app:traffic}
 
\subsubsection{Datasets}
\label{app:traffic:datasets}
 
We use two standard traffic-speed benchmarks from~\cite{li2018dcrnn}.
 
\emph{METR-LA.} $|\mathcal{X}_n|=207$ loop-detector sensors on the Los Angeles highway network, recording average traffic speed at $5$-minute intervals. The spatial graph follows the DCRNN convention: edge weights $W_{ij} = \exp(-d_\calM(x_i,x_j)^2/\sigma^2)$ with $\sigma$ set to the standard deviation of the pairwise road-network distances are thresholded to retain $W_{ij} \ge \kappa$ (with $\kappa = 0.1$) and symmetrized via $W \leftarrow \max(W, W^\top)$.
 
\emph{PEMS-BAY.} $|\mathcal{X}_n|=325$ sensors in the San Francisco Bay Area, with the same temporal resolution and graph construction.
 
For both datasets, we use $T=12$ observed time steps as input and forecast at horizons $h\in\{3,6,12\}$. Train/validation/test splits follow the standard $70/10/20$ chronological partition of~\cite{li2018dcrnn}.
 
\subsubsection{Task formulation}
\label{app:traffic:task}
 
At each forecasting instance, the input signal is
\begin{equation}
\mathbf{s}_{n,T}
=
({\mathbf{s}}_{x_1},\dots,{\mathbf{s}}_{x_n})
\in
C^0(\mathcal{F}_{n,T}^t; G_n),
\qquad
{\mathbf{s}}_{x_i}\in\mathbb{R}^T,
\end{equation}
where $n=|\mathcal{X}_n|$ and $d=T$, so the network sheaf $\mathcal{F}_{n,T}^t$ has $T$-dimensional stalks and Laplacian $\Delta_{\mathcal{F}_{n,T}^t}\in\mathbb{R}^{nT\times nT}$. The goal is to predict future speed vectors $\mathbf{y}^{(h)}\in\mathbb{R}^{|\mathcal{X}_n|}$ at each horizon $h$. The prediction loss is the mean absolute forecasting error
\begin{equation}
\mathcal{L}_{\mathrm{pred}}
=
\frac{1}{|\mathcal{D}|}
\sum_{(\mathbf{s}_{n,T},\mathbf{y})\in\mathcal{D}}
\sum_{h\in\{3,6,12\}}
\left\|
\widehat{\mathbf{y}}^{(h)}(\mathbf{s}_{n,T})-\mathbf{y}^{(h)}
\right\|_1.
\end{equation}
For HilbNet variants with learned transports, we add the kernel regularizer of Appendix~\ref{sec:PT-Algo} with weight $\lambda$, giving the full training objective $\mathcal{L}=\mathcal{L}_{\mathrm{pred}}+\lambda\,\mathcal{L}_{\mathrm{kernel-reg}}$.
 
\subsubsection{Model variants and baselines}
\label{app:traffic:models}
 
All HilbNet variants are discretized HilbNets (Def,~\ref{def: discretized-HilbNets}) with polynomial filters of order $K$ and the same per-node linear readout (DCRNN convention) mapping sheaf-filtered features to per-node horizon predictions. The only architectural difference is the admissible class of edgewise transports $P_{j\to i}^{e_{ij}}\in\mathcal{C}\subseteq O(T)$, as described in Appendix~\ref{sec:PT-Algo} with $d=T$. The transport hypothesis classes are the same of the previous section, briefly summarized and contextualized below.
 
\emph{Frozen identity.} $\mathcal{C}_{\mathrm{id}}=\{I_T\}$. Neighboring sensors exchange time windows without temporal alignment. The sheaf Laplacian reduces to a standard weighted graph Laplacian applied independently to each temporal coordinate, recovering a graph convolutional network.
 
\emph{Circulant.} Each transport is a real orthogonal circulant matrix parameterized by $\lfloor(T{-}1)/2\rfloor$ frequency-wise phases per edge. This encodes a time-stationary prior: the transport can advance or delay oscillatory components across an edge but cannot arbitrarily mix temporal coordinates. This is a natural inductive bias for traffic data, where congestion patterns propagate through the road network with local delays and phase shifts.
 
\emph{Free $O(T)$.} Each transport is parameterized by a product of $R{=}8$ Householder reflections, yielding an orthogonal matrix in a strict Householder-defined subset of $O(T)$ (see Appendix~\ref{sec:PT-Algo}). This is the most expressive transport class but uses $\sim\!10{\times}$ more parameters in total than the circulant variant (e.g.,\ $119{,}036$ vs $11{,}656$ on METR-LA; per edge the ratio is $\sim\!20{\times}$, since circulant uses only $\lfloor(T{-}1)/2\rfloor=5$ phases per edge at $T{=}12$).
 
We also compare against two non-transport baselines. A \emph{fiber-only MLP} processes each sensor's time window independently, ignoring graph structure. A \emph{spatiotemporal graph baseline} applies standard graph convolution to the temporally-augmented node features but does not learn sheaf transports. Finally, we include two external baselines from the literature: FC-LSTM~\citep{li2018dcrnn} and STAEformer~\citep{liu2023staeformer}, reported as in the cited papers \cameraready{(parameter counts estimated from their published configurations)}.
 
\subsubsection{Detailed results}
\label{app:traffic:results}
 
Table~\ref{tab:traffic_results} reports MAE, RMSE, and MAPE at horizons $3$, $6$, and $12$ (mean $\pm$ std over five seeds for our experiments).
 
\paragraph{Value of graph structure.} On both datasets, the frozen-identity HilbNet improves over the fiber-only MLP baseline, confirming that sheaf diffusion over the spatial graph is beneficial even without non-trivial transports.
 
\paragraph{Value of learned transports.} Both the free and circulant variants consistently outperform frozen identity at all horizons on both datasets. This confirms that the sheaf structure helpw beyond the usual graph structure.
 
\paragraph{Free vs.\ circulant.} On METR-LA, the free-$O(T)$ model achieves the best absolute accuracy at all horizons (e.g., MAE $3.938$ vs.\ $4.059$ for circulant at horizon $12$), as expected from its larger hypothesis class. On PEMS-BAY, the two variants are nearly tied: free wins MAE and MAPE at horizon $12$ by $\sim\!0.03$ mph (a $2$--$4\sigma$ effect over five seeds), circulant wins MAE/MAPE at horizon $3$, and RMSE is statistically indistinguishable at horizons $6$ and $12$. In both cases, the circulant model uses roughly one tenth of the transport parameters (e.g., $11{,}656$ vs.\ $119{,}036$ on METR-LA), making it the most parameter-efficient HilbNet variant. This supports the central pragmatic-transport message: a structured transport class encoding a physically motivated alignment prior can recover most of the benefit of unconstrained learned transports with substantially fewer degrees of freedom.
 
\paragraph{Comparison with external baselines.} Our HilbNet variants are lightweight models designed to test the value of Hilbert-sheaf structure, not to compete with large-scale spatiotemporal transformers. The external baselines (FC-LSTM, STAEformer) are included for reference and use substantially more parameters and architectural components. Nevertheless, the circulant and free HilbNets outperform FC-LSTM at all horizons on both datasets while using far fewer parameters.
 
\subsubsection{Hyperparameters}
\label{app:traffic:hyperparameters}
See Table~\ref{tab:traffic_hp}. All experiments are run on a single H200 GPU. Hyperparameters are chosen with a sweep. All presented variants are computed using the same codebase, and we made sure they differ only in their transport parametrization.
 
\begin{table}[h]
\centering
\small
\resizebox{\linewidth}{!}{\begin{tabular}{lll}
\toprule
& METR-LA & PEMS-BAY \\
\midrule
sensors $|\mathcal{X}_n|$              & $207$                                                       & $325$ \\
input window $T$           & $12$                                                        & $12$ \\
horizons $h$               & $\{3,6,12\}$                                                & $\{3,6,12\}$ \\
input feature dim          & $2$ (speed + time-of-day)                                   & $2$ (speed + time-of-day) \\
graph                      & thresh.\ kernel $\exp(-d^2/\sigma^2)$, $\kappa{=}0.1$       & thresh.\ kernel $\exp(-d^2/\sigma^2)$, $\kappa{=}0.1$ \\
sheaf-conv layers          & $L{=}2$, channel widths $[2,16,32]$                         & $L{=}2$, channel widths $[2,16,32]$ \\
polynomial order $K$ per layer  & $K=[2,2]$ (HilbNet), $K=[1,1]$ (MLP fiber baseline)             & $K=[2,2]$ (HilbNet), $K=[1,1]$ (MLP fiber baseline) \\
Householder reflections    & $8$                                                         & $8$ \\
readout                    & per-node linear ($T_{\mathrm{in}}\!\cdot\!F_{\mathrm{last}}\to T_{\mathrm{out}}$) & per-node linear ($T_{\mathrm{in}}\!\cdot\!F_{\mathrm{last}}\to T_{\mathrm{out}}$) \\
epochs / patience          & $150$ / $20$                                                & $150$ / $20$ \\
batch size                 & $32$                                                        & $32$ \\
optimizer                  & Adam, $\beta{=}(0.9, 0.999)$                                & Adam, $\beta{=}(0.9, 0.999)$ \\
learning rate              & $5\!\cdot\!10^{-3}$ (HilbNet); $1\!\cdot\!10^{-3}$ (STGNN baseline) & $5\!\cdot\!10^{-3}$ (HilbNet); $1\!\cdot\!10^{-3}$ (STGNN baseline) \\
LR schedule                & cosine annealing, $\eta_{\min}{=}10^{-6}$                   & cosine annealing, $\eta_{\min}{=}10^{-6}$ \\
weight decay / dropout     & $0$ / $0$                                                   & $0$ / $0$ \\
gradient clipping          & $\|g\|_2{\le}5$ (DCRNN convention)                          & $\|g\|_2{\le}5$ (DCRNN convention) \\
kernel regularizer $\lambda$   & $0.01$                                                      & $0.01$ \\
seeds                      & $5$                                                         & $5$ \\
\bottomrule
\end{tabular}}
\caption{Traffic forecasting hyperparameters.}
\label{tab:traffic_hp}
\end{table}

\section{Mathematical Background} \label{sec:mathematical-background}

\subsection{Hilbert Bundles}\label{sec:hilbert-bundle}
In this section, we provide relevant background on the theory of Hilbert bundles. In particular, we define the notions of Banach and Hilbert manifolds, as well as introduce the appropriate notions of connection, parallel transport, and heat flow for bundles in this setting.

\subsubsection{Banach and Hilbert manifolds}

To study heat kernels for smooth Hilbert bundles, we must examine manifolds modeled on generic Banach spaces. We will assume all Banach spaces and Hilbert spaces are defined over the field of real numbers $\mathbb{R}$, unless otherwise stated. 

\begin{definition}
    A second-countable topological space $\mathcal{M}$ is a \textbf{topological Banach manifold} if there is a Banach space $\mathcal{V}$ and an atlas $\{(U_i\: ,\: \phi_i:U_i \rightarrow \mathcal{V})\}_{i \in I}$ such that the following conditions hold:
    \begin{enumerate}
        \item each $U_i$ is an open subset of $\mathcal{M}$;
        \item each $\phi_i : U_i \rightarrow \mathcal{V}$ is a homeomorphism onto an open subset of $\mathcal{V}$;
        \item for all $i,j$, $\phi_i(U_i \cap U_j)$ is an open subset of $\mathcal{V}$;
        \item the transition map $\phi_j \phi_{i}^{-1}: \phi_i(U_i \cap U_j) \rightarrow \phi_j(U_i \cap U_j)$ is a homeomorphism.
    \end{enumerate}

    When the Banach space $\mathcal{V}$ is specified, we say that $\mathcal{M}$ is an $\mathcal{V}$-manifold, or a Banach manifold modeled on $\mathcal{V}$. If each map $\phi_i$ and transition map $\phi_j \phi_{i}^{-1}$ is $k$-times Fr\'echet differentiable, we say that $\mathcal{M}$ is a $C^k$-Banach manifold. If these maps are smooth i.e. $C^\infty$, we say that $\mathcal{M}$ is a smooth Banach manifold.
\end{definition}

\begin{definition}
    A topological (resp. $C^k$ / smooth) Banach manifold $\mathcal{M}$ is a topological (resp. $C^k$ / smooth) \textbf{Hilbert manifold} if it can be modeled on a Banach space $\mathcal{V}$ which admits the structure of a Hilbert space. 
\end{definition}

\begin{remark}
    We make a few observations about this definition.
    \begin{enumerate}
        \item Since every $n$-dimensional real Banach space is isomorphic to $\mathbb{R}^n$, a finite dimensional Banach manifold is exactly a real manifold in the usual sense.
        \item Like ordinary manifolds, we require Banach manifolds to be second countable, and hence to have a countable dense subset. It follows that if $\mathcal{M}$ is a manifold modeled on a Banach space $\mathcal{V}$, then $\mathcal{V}$ itself must be separable. This condition could be be removed, but it will generally make our lives easier. 
        \item The definition of a Hilbert manifold does not directly require the transition maps $\tau_{ij} := \phi_j \circ \phi_i^{-1}: \mathcal{V} \to \mathcal{V}$ to respect inner product structure on the modeling Hilbert space $\mathcal{V}$. Hence, it is often better to think of a Hilbert manifold as a special case of a Banach manifold, instead of as a manifold that respects the Hilbert space structure per se.
    \end{enumerate}
\end{remark}

The usual differential geometric constructions on manifolds extend naturally to Hilbert and Banach manifolds. For example, tangent spaces generalize naturally. Given a $C^k$-Banach manifold $\mathcal{M}$ with $k \geq 1$, for each $x \in \mathcal{M}$, one may form the \textbf{tangent space} $T_x\mathcal{M}$ at a point $x \in \mathcal{M}$ as equivalence classes of triples $(U, \phi, v)$ of a chart $\phi: U \to \mathcal{V}$ and a vector $v \in \mathcal{V}$, under the relation:
$$(U_1, \phi, v) \sim (U_2, \psi, w) \iff (D_{\phi(x)}(\psi \phi)^{-1})(v) = w.$$
Such equivalence classes are easily seen to form a real vector space isomorphic to $\mathcal{V}$. 

\subsubsection{Smooth bundles}

\begin{definition}[Smooth Banach and Hilbert  bundles]\label{def:BanachBundle}
Let \(\mathcal{M}\) be a smooth finite–dimensional manifold and let \(\mathcal{V}\) be a fixed separable Banach space.
A \textbf{smooth Banach bundle} with model space \(\mathcal{V}\) consists of a smooth Banach manifold $\calE$ equipped with a smooth surjective submersion
\[ \pi \colon \calE \longrightarrow \mathcal{M} \, , \]
that satisfies the following conditions.

\begin{enumerate}
  \item \textbf{Local triviality.}  
    For every \(p\in \calM\) there exists an open neighborhood \(U\subset \calM\) and a diffeomorphism
    \[
      \phi_U \colon \pi^{-1}(U) \;\xrightarrow{\;\cong\;}\; U \times \mathcal{V}
    \]
    satisfying \(\pi = \text{proj}_1 \!\circ \phi_U\), where \(\text{proj}_1: U \times \mathcal{V} \to U\) is the canonical projection, and such that, for each \(q\in U\),
    the restriction \(\phi_U|_{\calE_q}\colon \calE_q \to \{q\}\times \mathcal{V}\) is a bounded linear isomorphism. We call the pair \((U,\phi_U)\) a \textbf{trivializing chart}.

  \item \textbf{Smooth transition functions.}  
    Whenever \((U, \phi_U)\) and \((V, \phi_V)\) are trivializing charts, the \textbf{transition map}
    \[
      \tau_{UV}(q,-) \;:=\;
      \phi_V \circ \phi_U^{-1}\big (q,- \big )
      \;\colon\; \mathcal{V} \longrightarrow \mathcal{V},
      \qquad q \in U\cap V ,
    \]
    is a bounded isomorphism and depends smoothly on \(q\); that is,
    \(\tau_{UV}\colon U\cap V \to \mathrm{GL}(\mathcal{V})\) is a smooth map, where
    \(\mathrm{GL}(\mathcal{V})\) denotes the Banach–Lie group of bounded invertible operators on \(\mathcal{V}\) with the operator‐norm topology.

  \item \textbf{Smooth norm.} There is a smooth map $N: \calE \to \R$ such that the trivializing charts $(U, \phi_U)$ can be chosen with the additional property that for each $x \in \calE_p$,
  $$N(x) = \left \| \phi_U \big \rvert_{\calE_p}(x) \right \|_\mathcal{V} \, .$$

  \item \textbf{Smooth fiberwise operations.}  
    The fiberwise addition and scalar‐multiplication maps
    \[
      + \;\colon\; \calE \times_{\mathcal{M}} \calE \longrightarrow \calE ,
      \qquad
      \cdot \;\colon\; \mathbb{R} \times \calE \longrightarrow \calE ,
    \]
    are smooth Banach‐manifold maps.
\end{enumerate}

When the Banach space $\mathcal{V}$ is a separable Hilbert space, we say $\pi:\calE \to B$ is a \textbf{Hilbert bundle}. We denote the inner product on the fiber $\calE_p$ by $\langle-,-\rangle_p$.

\end{definition}

\begin{remark}
    We make a few remarks about the definition of a Hilbert bundle above.
    \begin{enumerate}
        \item For convenience, we restrict our attention to Hilbert bundles over closed finite-dimensional manifolds with separable fibers. None of these restrictions are essential for the general theory of Banach and Hilbert bundles. However, these restrictions are necessary for our approach to constructing heat kernels in this setting.
        \item The intuitive idea is the following: a smooth Banach bundle is a smooth vector bundle where the fibers are allowed to be infinite-dimensional and come equipped with a complete norm. The smooth norm condition enforces that the Banach space fibers are stitched together in such a way that the fiber-wise norm varies smoothly. In the case of a Hilbert bundle, the smooth norm condition also enforces that the fiber-wise inner product $\langle -,- \rangle_p$ varies smoothly.
        \item In light of the previous remarks, the definition presented here is not minimal. We make the choice to include redundant information in our definition for clarity, with the understanding that some conditions are superfluous. For instance, smooth fiberwise operations are implied by local triviality \cite{M_ller_2009}. We also make the choice to include the smooth norm condition, often called a \textbf{smooth orthogonal/Hermitian metric}, in the definition of the bundle itself.
        \item In the case of a finite dimensional model space $\mathcal{V}$, this definition recovers the usual smooth vector bundle, with the additional data of a smooth orthogonal/Hermitian metric. 
        \item Suppose $\pi:\calE \to \calM$ is a smooth Hilbert bundle modeled on a Hilbert space $\mathcal{H}$. While the transition maps $\tau_{UV}$ must respect the topological structure of the Hilbert space $\mathcal{H}$, it need not respect the inner product structure. When each transition map $\tau_{UV}$ is a unitary isomorphism, we say the bundle is a \textbf{smooth unitary Hilbert bundle}.
    \end{enumerate}
\end{remark}

\begin{definition}[Smooth sections of a Banach bundle]\label{def:SmoothSections}
Let $\pi : \calE \to \calM $ be a smooth Banach bundle (resp. Hilbert bundle) over a finite–dimensional manifold \(\calM\) with model Banach space (resp. Hilbert space) \(B\).  A \textbf{section} of $\calE$ is a map $S: \calM \to \calE$ such that $\pi \circ S = \mathrm{id}_\calM$. We denote the collection of all smooth sections by $\Gamma(\calE) := C^\infty(\calM, \calE)$. Note that this is a module over the commutative algebra $C^\infty(\calM)$ with point-wise addition and multiplication. If the section $S$ is only $k$-times continuously differentiable, we write $S \in C^k(\calM, \calE)$.
\end{definition}

\begin{definition}[$L^2$-Sections of a Banach bundle] Suppose that the manifold $\calM$ is endowed with a measure $\mu$. We say that $S$ is an \textbf{$L^2$-section} if $\|S\|_2 := \left ( \int_\calM ||S(x)||^2_{\calE_x} \, d \mu(x) \right)^{1/2} < \infty$. We may similarly form a \textbf{space of $L^2$-sections}, denoted $L^2(\calM,\calE;\mu)$, or simply $L^2(\calM,\calE)$ when the measure is implied by context. When $\calE$ is modeled on a separable Hilbert space $\mathcal{H}$, the space of $L^2$-sections $L^2(\calM,\calE;\mu)$ is a real separable Hilbert space with inner product $\langle S, S' \rangle_\calE := \int_\calM \langle S(x), S'(x)\rangle_{\calE_x} \, d \mu(x)$.
\end{definition}

\begin{remark}
   In the Riemannian setting, we have a natural candidate for $\mu$ via the (pseudo) volume form on $\mathcal{M}$, or its normalized variant. 
\end{remark}

\subsubsection{Connections} \label{sec:connections}

We now introduce connections on smooth Banach and Hilbert bundles. The smooth Banach manifold structure on the bundle $\calE$ provides no way to directly compare vectors in different fibers $\calE_x$ and $\calE_y$ with respect to their Banach space structures. Instead, as in the finite-dimensional case, we use a connection to link the fibers though the geometry of the base manifold $\calM$.

\begin{definition}
    Let $\pi:\calE \to \calM$ be a smooth Banach bundle over a compact manifold $\mathcal{M}$, and let $T^*\mathcal{M}$ denote the cotangent bundle on $\mathcal{M}$. A \textbf{connection} on $\calE$ is any of the following three equivalent structures:
    \begin{enumerate}
        \item A connection is an $\R$-linear map:
    $$\nabla : \Gamma(\mathcal{E}) \rightarrow \Gamma(T^*\mathcal{M} \otimes \mathcal{E})$$
    such that the product rule:
    $$\nabla(fS) = Df \otimes S + f \nabla S$$
    holds for all smooth function $f:\mathcal{M} \rightarrow \R$ and smooth sections $s \in \Gamma(\mathcal{E})$.
    \item A connection is a map $\nabla: \Gamma(T\calM) \times \Gamma(\calE) \to \Gamma(\calE)$ which is $C^{\infty}(\calM,\R)$-linear in its vector-field input, and satisfies the Leibniz rule:
    $$\nabla_X(fS) = X[f] S + f \nabla_X S$$
    where $\nabla_XS := \nabla(X,S)$ and $X[f](x) := (D_xf)(X_x)$ is the directional derivative of $f$ along the vector field $X$.
    \item A connection $\nabla$ is the data of an $\R$-linear map $\nabla_x S  : T_x \calM \rightarrow \calE_x$ for each $x \in \calM$ and $s \in \Gamma(\calE)$ that satisfies the following conditions:
    \begin{enumerate}
        \item $\nabla_{(-)}S$ depends smoothly on $x$
        \item $\nabla_x(a_1S_1 + a_2S_2) = a_1 \nabla_x S_1 +  a_2 \nabla_x S_2$ for all $x \in M$, $S_1,S_2 \in \Gamma(\calE)$, and $a_1,a_2 \in \R$;
        \item for every smooth $f: \calM \to \R$ and section $S \in \Gamma(\calE)$, let $fS \in \Gamma(\calE)$ denote the section $(fS)(x) := f(x) s(x)$. For each $x \in \calM$, the maps $\nabla_x(fS)$ and $\nabla_x s$ are related by
        $$\nabla_x(fS)(v) = D_xf(v)S(x) + f(x) (\nabla_x S)(v)$$
        for all $v \in T_x\calM$.
    \end{enumerate}
    \end{enumerate}  
    
\end{definition}

The following proposition provides a standard representation theorem for a smooth connection $\nabla$.

\begin{proposition}\label{pr: local form of connection}
    Let $\nabla$ be a smooth connection on a trivial Hilbert bundle $\pi: \calM \times \calH \to \calM$. There is a map $$A : \Gamma (\calM \times \mathcal{H}) \to \Gamma(T^*\calM \otimes (\calM \times \mathcal{H}))$$
    such that for every section $S \in \Gamma(\calM \times \mathcal{H})$, we have:
    $$\nabla S = dS + AS$$
    where $d$ is the Fréchet derivative.
    Equivalently, for each $x \in \calM$, and $v \in T_xM$, there is a bounded linear operator $A_{x,v} : \mathcal{H} \to \mathcal{H} $, varying smoothly in $(x,v)$, such that:
    $$(\nabla_x S)(v) = (D_xS)(v) + (A_{x,v}S)(x) \, . $$
    Moreover, the assignment $v \mapsto A_{x,v}$ is linear for each $x$.
\end{proposition}

\begin{remark}
    While this proposition is stated for trivial bundles, it can be applied to any Hilbert bundle through the choice of a trivialization, or through Kuiper's theorem \citep{Kuiper1965}. 
\end{remark}

\subsubsection{Parallel Transport}
Connections allow us to relate the geometries of the fibers over nearby points in $\calM$. For example, we may define parallel transport.

Given a smooth curve $\gamma:[0,1] \to \calM$, say that $S:[0,1] \to \calE$ is a \textbf{section over $\gamma$} if $\pi \circ S = \gamma$.

\begin{definition}
    Let $\pi:\mathcal{E} \to \calM$ be a smooth Banach bundle with model space $X$, let $\nabla$ be a connection on $\calE$, and let $\gamma:[0,1] \to \calM$ be a smooth path. A map $S:[0,1] \to \calE$ is a \textbf{section over} $\gamma$ if $\pi \circ S = \gamma$. A section over $\gamma$ is \textbf{parallel} if
    $$\nabla_{\dot{\gamma}}S = 0.$$ 
\end{definition}

\begin{proposition}
    Let $\gamma$ be a smooth path in $\calM$. For every vector $v \in \calE_{\gamma(0)}$, there is a unique parallel section $S_v$ over $\gamma$ such that $S_v(0) = v$. Moreover, the dependence on $v$ is smooth and linear. 
\end{proposition}

\begin{proof}
    The existence of the parallel section $S_v$ can be restated as an initial value problem for a linear ordinary differential equation in a Banach space. Standard existence and uniqueness theorems apply. For details, see \cite{Lang}.
\end{proof}

By the existence and uniqueness of parallel sections, we may define corresponding \textbf{parallel transport maps}. Given a path $\gamma:[0,1] \to \calM$, there is an induced \textbf{parallel transport} operator $P_\gamma^{\nabla}: \calE_{\gamma(0)} \to \calE_{\gamma(1)}$ defined by
$$P^{\nabla}_{\gamma}(v) := S_v(1).$$
It is straightforward to see that $P^{\nabla}_\gamma$ is a linear bijection, with inverse given by $(P^{\nabla}_\gamma)^{-1} = P^{\nabla}_{\gamma^{\text{rev}}}$, where $\gamma^{\text{rev}}$ is the path obtained by reversing $\gamma$. By the closed graph theorem, it follows that $P^{\nabla}_\gamma$ is a bounded linear isomorphism. 

\begin{definition}
    Let $\pi: \calE \rightarrow \calM$ be a Hilbert bundle equipped with a connection $\nabla$. We say the connection $\nabla$ is \textbf{compatible} with the Hilbert bundle structure if it satisfies:
    $$X[\langle S_0(x), S_1(x) \rangle_x] = \langle \nabla_X S_0(x), S_1(x) \rangle_x + \langle S_0(x), \nabla_X S_1(x) \rangle_x$$
    for every smooth vector field $X \in \Gamma (T\mathcal{M})$, and sections $S_0,S_1 \in \Gamma(\calE)$. 
\end{definition}

\begin{proposition} \label{prop:connection-equivalent}
    Let $\pi: \calE \to \calM$ be a Hilbert bundle with connection $\nabla$. The following are equivalent:
    \begin{enumerate}
        \item $\nabla$ is compatible with the Hilbert bundle structure;
        \item Every parallel transport map $P^\nabla_{\gamma}$ is unitary;
    \end{enumerate}
\end{proposition}

\begin{proof}
    First suppose that $\nabla$ is compatible with the Hilbert bundle structure, and $\gamma$ a smooth path in $\calM$. Let $u,v \in \calE_{\gamma(0)}$. By compatibility, we may check that:
    \begin{align*}
        \frac{d}{dt} \big \langle S_u(t), S_v(t) \big \rangle_{\gamma(t)} & = \big \langle \nabla_{\dot{\gamma}}S_u(t), S_v(t) \big \rangle_{\gamma(t)}  +  \big \langle S_u(t), \nabla_{\dot{\gamma}}S_v(t) \big \rangle_{\gamma(t)}   = 0 \, . 
    \end{align*}
    It immediately follows that $P^\nabla_{\gamma}$ is unitary.

    Conversely, suppose that every parallel transport map $P^\nabla_{\gamma}$ is unitary. Let $X$ be a smooth vector field, and $S_0,S_1$ smooth sections of $\calE$. Let $x \in \calM$, and let $\gamma$ be a smooth path such that $\gamma(0) = x$. For $j \in \{0,1\}$ let $u_j$ be a parallel section over $\gamma$ such that $u_j(0) := S_j(x)$. 
    For $j \in \{0,1\}$, let $w_j(t) := S_j(\gamma(t)) - u_j(t)$.
    Finally, let $w_j(t) := S_j(\gamma(t)) - u_j(t)$.
     Since $u_j$ is parallel over $\gamma$, we have that
    $\nabla_{\dot{\gamma}} S_j(\gamma(t))  = \nabla_{\dot{\gamma}} w_j(t).$
    Moreover, $w_j(0) = 0$.
    We may use these facts to compute:
    \begin{align*}
        X[\big \langle S_0(x), S_1(x) \big \rangle_{x}] 
        & = \frac{d}{dt} \bigg \rvert_{t = 0} \big \langle S_0(\gamma(t)), S_1(\gamma(t)) \big \rangle_{\gamma(t)} \\
        & = \frac{d}{dt} \bigg \rvert_{t = 0} \big \langle u_0(t) + w_0(t), u_1(t) + w_1(t) \big \rangle_{\gamma(t)} \\
        & = \frac{d}{dt} \bigg \rvert_{t=0} \langle u_0(t), u_1(t) \rangle_{\gamma(t)} + \langle \nabla_{\dot{\gamma}}S_0(x), S_1(x) \rangle_x + \langle S_0(x), \nabla_{\dot{\gamma}} S_1(x) \rangle_x
    \end{align*}
    Since parallel transport maps are unitary, it follows that the quantity $\big \langle u_0(t), u_1(t) \big \rangle_{\gamma(t)}$
    is constant in $t$. Hence 
    \begin{align*}
        X[\big \langle S_0(x), S_1(x) \big \rangle_{x}] 
        & = \frac{d}{dt} \bigg \rvert_{t=0} \langle u_0(t), u_1(t) \rangle_{\gamma(t)} + \langle \nabla_{\dot{\gamma}}S_0(x), S_1(x) \rangle_x + \langle S_0(x), \nabla_{\dot{\gamma}} S_1(x) \rangle_x \\
        & = \langle \nabla_{X}S_0(x), S_1(x) \rangle_x + \langle S_0(x), \nabla_{X} S_1(x) \rangle_x 
    \end{align*}
    proving that $\nabla$ is a compatible connection.
\end{proof}

\begin{remark}
    We note that by the Proposition~\ref{prop:connection-equivalent}, as we assumed our parallel transport operators to be unitary in the main text, this may equivalently be understood as a metric-compatibility condition. 
\end{remark}

\subsection{Connection Laplacian}
Let $\pi : \calE \to \calM$ be a Hilbert bundle on a closed Riemannian manifold, equipped with a compatible connection $\nabla : \Gamma(\calM;\calE) \to \Gamma (\calM; T^*\calM \otimes \calE)$. Moreover, $T^*\calM \otimes \calE$ inherits the structure of a Hilbert bundle, with fiber-wise inner products induced from the metric $g$ and the fiber-wise inner products of $\calE$. Since $\nabla$ is a linear differential operator, it has a formal adjoint $\nabla^*: \Gamma (\calM; T^*\calM \otimes \calE) \to \Gamma(\calM;\calE)$, defined implicitly by the formula:
$$\int_\calM \langle \nabla S_0(x), S_1(x) \rangle_{x} \, d \mu(x) \: = \: \int_{\calM} \langle S_0(x), \nabla^*S_1(x) \rangle_{x} \, d\mu(x)$$
where $S_0 \in \Gamma(\calM;\calE)$, $S_1 \in \Gamma(\calM; T^*\calM \otimes \calE)$, and $\mu$ is the pseudo-volume form on $\calM$. Using this adjoint, we may define the connection Laplacian.

\begin{definition}
    The \textbf{connection Laplacian} is the linear operator:
    $$\Delta_{\nabla}:= \nabla^* \nabla : \Gamma(\calE) \to \Gamma(\calE) \, . $$
\end{definition}

\begin{remark}
    The connection $\nabla: \Gamma(\calE) \to \Gamma(T^*\calM \otimes \calE)$ can be extended to a closed, densely-defined unbounded operator $\nabla_{L^2}: L^2(\calE) \to L^2(T^*\calM \otimes \calE)$. This extended operator has an adjoint $\nabla^*_{L^2}$ as a Hilbert space operator, from which we may define a composite $\Delta_{\nabla_{L^2}} := \nabla_{L^2}^* \nabla_{L^2}$. The formal adjoint $\nabla^*$ and connection Laplacian $\Delta_{\nabla}$ can be found by restricting the domains of $\nabla_{L^2}^*$ and $\Delta_{\nabla_{L^2}}$ to linear subspaces of smooth sections. From this perspective, it becomes clear that the connection Laplacian $\Delta_\nabla$ is well defined for all $C^2$ sections, as $C^2$ is contained in the Sobolev space $H^2$. 
\end{remark}

The connection Laplacian also admits a characterization in terms of covariant derivatives. Let $\nabla^2_{X,Y}$ denote the second covariant derivative with respect to vector fields $X,Y$. 

\begin{lemma}(Connection Laplacian in Coordinates)\label{lemma:coord-expression}
Let $(\mathcal{M}, \mathcal{E}, \nabla)$  be a Hilbert bundle over a closed Riemannian manifold equipped with a compatible Fréchet connection. As operators,
$$ \Delta_\nabla S = - \mathrm{tr} \nabla^2 S \, . $$
Moreover, with respect to a local orthonormal frame $\{e_i\}_{i=1}^m$ that is synchronous at $p \in \calM$, we have equality:
$$\Delta_\nabla S(p) = - \sum_{i=1}^m \nabla_{e_i} \nabla_{e_i} S(p) \, . $$
\end{lemma}

\begin{proof}
    We adapt the proof of \cite{Petersen2006} to the Hilbert bundle setting, using the synchronous frame technique of \cite{nicolaescu2007geometry}. By a partition of unity argument, it suffices to show that for every pair of smooth sections $S_1,S_2$ supported inside the domain of a local orthonormal frame $\{e_i\}_{i=1}^m$, we have an equality:
    $$\int_{\calM} \langle \Delta_{\nabla} S_1(x), S_2(x) \rangle_x \, d \mu(x) = -  \int_{\calM}\sum_{i=1}^m \langle \nabla^2_{e_i, e_i} S_1(x), S_2(x) \rangle_x \, d \mu(x) \: .  $$
    Using the formal adjoint of $\nabla$, the integral on the left can be rewritten as
    $$\int_{\calM} \langle \Delta_{\nabla} S_1(x), S_2(x) \rangle_x \, d \mu(x) = \sum_{i=1}^m \int_{\calM} \langle \nabla_{e_i} S_1(x), \nabla_{e_i} S_2(x) \rangle_x \, d \mu(x) \: . $$
    To analyze the right hand side, simply note that $\nabla_{e_i, e_i}^2 S_1 = \nabla_{e_i} \nabla_{e_i} S_1 - \nabla_{\nabla_{e_i}e_i} S_1$, and by compatibility, that:
    $$e_i [\langle \nabla_{e_i} S_1(x), S_2(x) \rangle_x] = \langle \nabla_{e_i} \nabla_{e_i} S_1(x), S_2(x) \rangle_x + \langle \nabla_{e_i}S_1(x), \nabla_{e_i} S_2(x) \rangle_x \: . $$
    Rearranging and summing over $i$ yields:
    \begin{align*}
        \sum_{i} \langle \nabla^2_{e_i, e_i} S_1(x), S_2(x) \rangle_x = & - \sum_{i} \langle \nabla_{e_i}S_1(x), \nabla_{e_i}S_2(x) \rangle_x \\ &+ \sum_{i} \left ( e_i[\langle \nabla_{e_i} S_1(x), S_2(x) \rangle_x ] - \langle \nabla_{\nabla_{e_i}e_i} S_1(x), S_2(x) \rangle_x  \right ) \: .
    \end{align*}
    The second sum on the right-hand side may be identified as the divergence of the vector field $\langle \nabla_vS_1(x), S_2(x) \rangle_x$, and hence integrates to zero by Stokes' theorem. Therefore
    $$ -\int_{\calM}\sum_i \langle  \nabla^2_{e_i, e_i} S_1(x), S_2(x) \rangle_x \, d \mu(x) = \sum_{i} \int_{\calM} \langle \nabla_{e_i} S_1(x), \nabla_{e_i} S_2(x) \rangle_x \, d \mu(x) $$
    as well. Therefore $\Delta_\nabla (S) = - \mathrm{tr}\nabla^2(S)$.
    Under the additional hypothesis that $\{e_i\}_{i=1}^m$ is synchronous at $p$, we have $\nabla_{e_i}e_i = 0$ for each $1 \leq i \leq m$. At such a point, the trace reduces to $\mathrm{tr} \nabla^2 S(p) = \sum_i \nabla_{e_i} \nabla_{e_i} S(p)$. 
\end{proof}

\subsection{Heat Flow on a Hilbert Bundle}\label{app:heat_flowhb}
In this section, we fix a closed finite-dimensional Riemannian manifold $(\calM,g)$ with canonical volume pseudo-form $\mu$, a smooth Hilbert bundle $\pi: \calE \to \calM$ with fiber $\calE_x \cong \calH$, and a connection $\nabla$. Our goal in this section is three-fold:
\begin{enumerate}
    \item Demonstrate that the heat equation with respect to the connection Laplacian $\Delta_\nabla$ has a unique solution;
    \item Show that the heat-flow admits a Heat Kernel;
    \item Provide asymptotic estimates for the heat flow that relate to the geometry of the underlying manifold $\calM$.
\end{enumerate}
Our approach is to adapt the methods of Berline, Getzler, and Vergne \cite{Berline1992HeatKA} for finite-rank bundles to the Hilbert bundle setting. A key subtlety that arises in this generalization is in the definition of tensor-products of bundles. While the algebraic and topological tensor products of finite-rank Hilbert spaces agree, they need not coincide for general Hilbert spaces. This complicates, for instance, the necessary tensor-hom adjunction. In order to keep track of the appropriate tensor, we adopt the following convention.  Let $\pi_0: \calE_0 \to \calM$ and $\pi_1 : \calE_1 \to \calM$ be smooth Hilbert bundles over a common manifold $\calM$. One may form the \textbf{hom-bundle} $\hom(\calE_0, \calE_1) \to \calM$, whose fiber $\hom(\calE_0, \calE_1)_x$ is the Banach space of bounded linear operators from $\calE_0$ to $\calE_1$. This is a Banach bundle when $\hom(\calE_0, \calE_1)_x$ is topologized with the operator norm.

\begin{definition}
    Let $\Delta_\nabla$ be the connection Laplacian for a compatible connection on a smooth Hilbert bundle $\pi : \calE \to \calM$ over a compact orientable manifold $\calM$. A \textbf{heat kernel} for $\Delta_\nabla$ is a continuous section $K^\nabla(x,y,t)$ of the Banach bundle $\text{hom} \big ( \text{pr}_2^* \calE, \text{pr}_1^* \calE \big ) $ over $\calM \times \calM \times (0,\infty)$ that satisfies the following conditions.
    \begin{enumerate}
        \item $K^\nabla(x,y,t)$ is $C^1$ with respect to $t$, and is $C^2$ with respect to $x$. 
        \item $K^\nabla(x,y,t)$ satisfies the \textbf{heat equation} $\partial_t K^\nabla(x,y,t) = - (\Delta_\nabla)_x K^\nabla(x,y,t)$, where $(\Delta_\nabla)_x$ means applying the Laplacian to $x$.
        \item $K^\nabla(x,y,t)$ satisfies the boundary condition $\lim_{t \to 0} K_t S = S$ for every smooth section $s \in \Gamma(\calE)$, where
        $(K_tS)(x) = \int K^\nabla(x,y,t) S(y) \, d \mu(y).$
    \end{enumerate}
\end{definition}

\begin{lemma}(Heat Kernel for Hilbert Bundles) \label{Heat-Kernel-Existence}
      Let $\Delta_\nabla$ be the connection Laplacian associated to a Hilbert bundle, $(\mathcal{M}, \mathcal{E}, \nabla)$.  Let $n := \dim(\mathcal{M})/2$, $\psi$ be a cutoff function, and 
    $$C_n(x,y) := \frac{\exp \left ( \frac{-d_\calM(x,y)^2}{4t} \right ) }{(4\pi t)^{n/2}}. $$ 
    The following hold:
    \begin{enumerate}
        \item The Laplacian $\Delta_\nabla$ admits a unique heat kernel $K^\nabla(x,y,t)$.
        \item There exist smooth sections $\Phi_i \in \Gamma \left (\mathcal{M} \times \mathcal{M}, \text{hom} \big ( \text{pr}_2^* \calE, \text{pr}_1^* \calE \big ) \right )$ such that for every $N > n$, the kernel
        $$K^N(x,y,t) := C_n(x,y) \psi(d_\calM(x,y)^2) \sum_{i=0}^N t^i \Phi_i(x,y) j(x,y)^{-1/2} |dx|^{1/2}$$
        is asymptotic to $K^\nabla(x,y,t)$, in the sense that
        $$\|\partial^k_t (K^\nabla(x,y,t) - K^N(x,y,t))\|_{\ell} = O(t^{N- \frac{n + \ell + 2k}{2}}).$$
        \item The leading term $\Phi_0(x,y)$ is equal to the parallel transport $P_{y \to x}: \mathcal{E}_y \to \mathcal{E}_x$ with respect to the Fréchet connection associated to $\calH$ along the unique length-minimizing geodesic joining $y$ and $x$. 
    \end{enumerate}
\end{lemma}

\begin{proof} \textit{(Sketch)} We note that parametrix-based approach of Berline, Getzler, and Vergne \cite{Berline1992HeatKA} extends to our setting with minor but judicious modifications. As the original parametrix argument is quite lengthy, we simply make note of the necessary modifications to their argument. First, all integration must be understood as Bochner integration. Second, to avoid ambiguities surrounding the algebraic and geometric tensor bundle, the hom-bundle is used instead of tensor bundles.  \\
Now note that the parametrix argument is fundamentally local in nature. Consider a smooth Hilbert bundle $\calE \to \calM$, and note that we require an associated Fréchet connection $\nabla$ to be metric-compatible. Then, in a a local trivialization $\calE \big \rvert_U \cong U \times \calH$, with $\calH$ a separable Hilbert space, the connection has the form $\Delta = d + A$ with $A$ a smooth $\hom(\calH, \calH)$ valued 1-form by Proposition \ref{pr: local form of connection}. The connection Laplacian is a second-order elliptic operator with scalar principal symbol and lower-order coefficients in the Banach algebra $\hom(\calH, \calH)$. The parametrix argument then proceeds via solving transport equations along geodesics and then correcting the resulting approximated kernel by a Volterra series. These steps do not rely in any essential way on finite-dimensionality of the fiber, but only on the fact that the coefficient algebra admits the usual smooth calculus and operator-norm estimates. Thus, replacing matrix-valued coefficients by $\hom(\calH, \calH)$-valued ones, one obtains in the same way a smooth kernel $K^\nabla(x,y,t)$, and the usual energy argument gives uniqueness.
\end{proof}

\begin{remark}
    The details of the necessary parametrix argument may be found in 2.1 -- 2.5 of \cite{Berline1992HeatKA}. Note that while they assume a finite-rank hypothesis through the entirety of chapter two, the hypothesis is actually unused until section 2.6 of their work, when the operator $K_t$ is required to be Hilbert-Schmidt. 
\end{remark}

\subsection{Borel Functional Calculus} \label{sec:borel-calc}

Given a linear map $T: \R^n \to \R^n$ and a suitably well-behaved function $g: \mathbb{R} \to \mathbb{C}$, one may ``apply $g$'' to $T$ to get a new linear map $g(T): \R^n \to \R^n$. In particular, whenever $g$ is analytic with a globally defined Maclaurin expansion $g(x) = \sum_j a_j x^j$, one may define $g(T) = \sum_j a_jT^{j}$, where $T^j$ is interpreted as the  $j$-fold composition $T \circ \cdots \circ T$. When $T$ is a bounded linear endo-operator on a Hilbert space, one may similarly define $g(T)$ via series expansion. However, when $T$ is unbounded, more care must be taken to handle series convergence. This difficulty in the unbounded case is pertinent for the HilbNet architecture, where the convolution filter $g$ must be applied to the unbounded connection Laplacian $\Delta_\nabla$. The \emph{Borel functional calculus} provides an elegant solution. While traditionally formulated through the spectral theorem and projection-valued measures, for the purpose of the HilbNet architecture, the following version (Theorem VIII.5 of \cite{reed1972functional}) will be sufficient.



\begin{theorem} (Spectral Theorem - Functional Calculus Form)  \label{thm:spectral-thm} Let $A$ be a self-adjoint operator on a Hilbert space $\mathcal{H}$. Then there is a unique map $\hat{\phi}$ from the bounded Borel measurable functions on $\mathbb{R}$ into the space of bounded linear operators on $\calH$, $\mathcal{L}(\mathcal{H})$, so that \begin{itemize}
    \item $\hat{\phi}$ is an algebraic *-homomorphism.
    \item $\hat{\phi}$ is norm continuous, that is, $\|\hat{\phi}(h)\|_{\mathcal{L}(\mathcal{\calH})} \leq\|h\|_{\infty}$.
    \item Let $g_n(x)$ be a sequence of bounded Borel functions with $g_n(x) \xrightarrow[n \rightarrow \infty]{ } x$ for each $x$ and $\left|h_n(x)\right| \leq|x|$ for all $x$ and $n$. Then, for any $\psi \in \mathrm{dom}(A)$, $\lim _{n \rightarrow \infty} \hat{\phi}\left(h_n\right) \psi=A \psi$.
    \item If $g_n(x) \to h(x)$ pointwise and if the sequence $\|h_n \|_\infty$ is bounded, then $\hat{\phi}(h_n) \to \hat{\phi}(h)$ strongly.
\end{itemize}

In addition:

\begin{itemize}
    \item  If $A \psi=\lambda \psi$, then $\hat{\phi}(h) \psi=h(\lambda) \psi$.
    \item If $g \geq 0$, then $\hat{\phi}(h) \geq 0$.
\end{itemize}

\end{theorem}

\subsection{Cellular Sheaves and Sheaf Laplacians} \label{sec:cellular-sheaves}

Cellular sheaves on graphs are a data structure that generalizes weighted graphs. We take our exposition of cellular sheaves and their Laplacians primarily from \cite{Hansen2019towardspecsheaf}. See \cite{Hansen2019towardspecsheaf, curry2014sheaves, hansen2020opinion} for more details. 

\begin{definition}[Cellular Sheaf on a Graph]
    Let $G = (V,E)$ be an undirected multi-graph without self-loops, and finitely many vertices and edges. Let $v \leq e$ denote that node $v$ is incident to the edge $e$.
    A \textbf{cellular sheaf}, or equivalently \textbf{network sheaf}, $\calF$ on $G$ consists of the following data.
    \begin{itemize}
        \item A vector space $\calF(\sigma)$ for each $\sigma \in V \amalg E$, called the \textbf{stalk} over $\sigma$. 
        \item A linear map $ \calF_{v \leq e} : \calF(v) \to \calF(e)$ for each incident pair $v \leq e$, called the \textbf{restriction map} of $v$ into $e$.
    \end{itemize}
\end{definition}

\begin{remark}At the level of category theory, a cellular sheaf is a functor $\calF : G \to \mathbf{Vect}$, where the graph $G = (V,E)$ is viewed as a posetal category with objects $\mathrm{Ob}(G) = V \amalg E$, and a unique homomorphism from $v \to e$ whenever $v \leq e$. In this light, we adopt the notation $\calF : G \to \mathbf{Vect}$ for a cellular sheaf $\calF$ on a graph $G$. 
\end{remark}

Traditionally, to add geometric content to a cellular sheaf, one passes to \emph{weighted cellular sheaves}: a cellular sheaf $\calF: G \to \mathbf{Vect}$ where each stalk $\calF_\sigma$ is a finite dimensional vector space endowed with an inner product $\langle - , - \rangle_\sigma$. To accommodate infinite-dimensional Hilbert space stalks, we instead follow the approach of \cite{gould25}.

\begin{definition}(Hilbert Cellular Sheaf on a Graph)
    A \textbf{Hilbert cellular sheaf} $\mathcal{F}$ on a finite graph $G = (V,E)$ consists of the following data.
    \begin{itemize}
        \item A Hilbert space $\calF(v)$ for each $v \in V $, referred to as the \textbf{node stalk} over $v$. 
        \item A Hilbert space $\calF(e)$ for each $e \in E $, referred to as the \textbf{edge stalk} over $v$. 
        \item For each edge $e \in E$ with bounding vertices $u,v$, a pair of bounded linear \textbf{restriction maps} $\calF_{u \leq e} : \calF(u) \to \calF(e)$ and $\calF_{v \leq e} : \calF(v) \to \calF(e)$.
    \end{itemize}
\end{definition}


\begin{remark}
    A bounded Hilbert sheaf $\calF$ can again be viewed as a functor $\calF: G \to \mathbf{Hilb}_{\mathbb{R}}$, where $G$ is the graph $G$ viewed as an acyclic category, and $\mathbf{Hilb}_{\mathbb{R}}$ is the category of real Hilbert spaces and bounded globally-defined linear operators. 
\end{remark}

\begin{remark}
    In order to better differentiate between the usual finite-rank cellular sheaf on a graph and the potentially infinite-rank Hilbert cellular sheaves considered above, we use the terminology of network sheaves when we wish to emphasize the finite-rank consideration. 
\end{remark}

\begin{definition}
    Let $\calF: G \to \mathbf{Hilb}_{\mathbb{R}}$ be a bounded Hilbert sheaf on a graph $G = (V,E)$. The spaces of \textbf{0-cochains} and \textbf{1-cochains} are defined by:
    \begin{align*}
        C^0(G;\calF) & := \bigoplus_{v \in V} \calF(v) \, , \\ 
        C^1(G;\calF) & := \bigoplus_{e \in E} \calF(e) \, . 
    \end{align*}
    where $\oplus$ denotes the direct sum of Hilbert spaces. For a 0-cochain $\xx \in C^0(G;\calF)$, we denote the component of $\xx$ in the stalk over the node $v$ by $\xx_v$, with a similar notation for components of 1-cochains. 
\end{definition}

\begin{definition}
    Let $G = (V,E)$ be a graph. A \textbf{signed incidence relation} on $G$ is a pairing $[-:-]: V \times E \to \{-1,0,1\}$ which satisfies the following conditions:
    \begin{enumerate}
        \item $[v:e] \neq 0$ if and only if $v \leq e$.
        \item For each $e$, $\sum_{v \leq e} [v:e
    ] =0$.
    \end{enumerate}
\end{definition}

\begin{remark}
    The data of a signed incidence structure on $G = (V,E)$ is equivalent to the choice of a source $s(e)$ and target $t(e)$ for each edge $e$. In particular, the total set of incidences can be put into two-to-one correspondence with edges, counting the two distinct ``boundings'' of each $e \in E$.
\end{remark}

\begin{definition}
    Let $G = (V,E)$ be a graph equipped with a signed incidence relation. Let $\calF: G \to \mathbf{Hilb}_{\R}$ be a bounded Hilbert sheaf on $G$. The \textbf{coboundary operator} $\delta: C^0(G;\calF) \to C^1(G;\calF)$ is the operator with image on the each edge stalk:
    $$\left ( \delta \xx \right )_e := \sum_{\substack{ v \leq e}} [v:e] \calF_{v \leq e}(\xx_v) \, .$$
\end{definition}

\begin{remark}
    The coboundary map $\delta$ depends on the choice of the signed incidence relation. However, given two signed incidence relations $[-:-]_1, [-:-]_2$, the corresponding coboundary operators $\delta_1, \delta_2$ differ on the stalk $\calF(e)$ by at most a sign difference $(\delta_1 \xx)_{e} = \pm (\delta_2 \xx)_e$. In particular, $\ker(\delta)$ does not depend on the choice of $\epsilon$.
\end{remark}

\begin{definition}
    Let $\calF: G \to \mathbf{Hilb}_{\mathbb{R}}$ be a bounded Hilbert network sheaf on a graph $G = (V,E, \epsilon)$ equipped with a signed incidence relation. Let $\delta^*: C^1(G;\calF) \to C^0(G;\calF)$ denote the linear adjoint of the corresponding coboundary operator with respect to the inner product structures on the spaces of cochains as product spaces. The \textbf{Hilbert sheaf Laplacian} is the operator $\Delta_\calF : C^0(G;\calF) \to C^0(G;\calF)$ defined by the composition:
    $$\Delta_\calF = \delta^* \circ \delta\, .$$
\end{definition}

\begin{proposition}
    The Hilbert sheaf Laplacian $\Delta_{\calF}$ has the following properties.
    \begin{enumerate}
        \item The Laplacian $\Delta_{\calF}$ is a self-adjoint globally-defined bounded linear operator.
        \item When $C^0(G;\calF) \xrightarrow{\delta} C^1(G;\calF)$ is viewed as a Hilbert complex (in the sense of \cite{Bruning92}) the kernel $\ker(\Delta_{\calF})$ recovers the space of harmonic $0$-cochains.
        \item The negative Laplacian $-\Delta_\calF$ is the infinitesimal generator of a strongly continuous semigroup $e^{- t\Delta_{\calF}}$ on $C^{0}(G;\calF)$. For a choice of initial cochain $\xx_0 \in C^0(G;\calF)$, the resulting flow 
        $\xx_t := e^{- t\Delta_{\calF}} \xx_0$ is a solution to the sheaf heat equation 
        $\frac{d}{dt} \xx_t = - \Delta_{\calF} \xx_t$. 
        Moreover, the flow has limiting behavior $\lim_{t \to \infty} \xx_t = \Pi_{\ker(\Delta_{\calF})} \xx_0$, where $\Pi_{\ker(\Delta_{\calF})}$ denotes the orthogonal projection onto $\ker(\Delta_{\calF})$.
    \end{enumerate}
\end{proposition}

\begin{proof}
    See \cite{gould25}.
\end{proof}

We now recall our construction of Hilbert cellular sheaf from a spatially-discretized Hilbert bundle.

\begin{definition*}[Hilbert Cellular Sheaf from a Hilbert Bundle]
For a given Hilbert bundle $(\mathcal{M},\mathcal{E},\nabla)$ with sampled points $\mathcal{X}_n = \{x_1, \dots, x_n \} \subset \mathcal{M}$, fix a geodesic $\gamma_{ij}$ between $x_i$ and $x_j$, for all $i < j$. Further, let $m_{\gamma_{ij}}$ denote the midpoint of this geodesic. Consider the graph $G_{n} = (\mathcal{X}_n,E)$ with an undirected edge $e_{ij}$ between $x_i$ and $x_j$, for each $i < j$. The associated \textbf{Hilbert cellular sheaf} $\mathcal{F}_n^{t}$ on $G_n$ with bandwidth parameter $t$ is given by the following assignments:
    \begin{itemize}[leftmargin=*]
    \item The Hilbert space $\mathcal{F}_n^t(x_i) := \mathcal{E}_{x_i}$ for each $x_i \in \mathcal{X}_n$, referred to as the node stalk over $x_i \in \mathcal{X}_n$.
    \item The Hilbert space $\mathcal{F}_n^t(e_{ij}) := \mathcal{E}_{m_{\gamma_{ij}}}$ for each $e_{ij} \in E$, referred to as the edge stalk over $e_{i,j} \in E$.
    \item For each edge $e_{ij} \in E$ with bounding vertices $x_i,x_j$, a pair of bounded linear restriction maps
    \begin{align}
    &(\mathcal{F}_n^t)_{x_i \leq e_{ij}}
    :=
    \sqrt{k_{ij}^t} \,
    P_{x_i \to m_{\gamma_{ij}}}:
    \mathcal{F}_n^t(x_i)
    \to
    \mathcal{F}_n^t(e_{ij}), \nonumber \\
    &(\mathcal{F}_n^t)_{x_j \leq e_{ij}}
    :=
    \sqrt{k_{ij}^t} \,
    P_{x_j \to m_{\gamma_{ij}}}:
    \mathcal{F}_n^t(x_j)
    \to
    \mathcal{F}_n^t(e_{ij}),
    \end{align}
    where $k_{ij}^t = e^{-d_{\calM}(x_i,x_j)^2/4t}$, with $d_{\calM}$ the geodesic distance on $\calM$, and $P_{x_i \to m_{\gamma_{ij}}}$ denotes the unitary parallel transport map on $\mathcal{E}$ between $x_i$ and $m_{\gamma_{ij}}$. 
    \end{itemize}
\end{definition*}

\begin{remark}
    We make a few clarifying remarks on this construction. 
    \begin{itemize}
        \item Note that geodesics exist in this setting by the Hopf-Rinow theorem \cite{docarmo1992} by compactness of $\mathcal{M}$, and we should further choose length-minimizing geodesics. 
        \item For simplicity, we use the geodesic distance to weight our restriction maps. However, we could also use the Euclidean heat kernel ala \cite{BELKIN20081289}, and this would result in a reweighted sheaf Laplacian but would ultimately converge to the same connection Laplacian. In practical implementations, it is thus well-justified to work with the Euclidean heat kernel rather than geodesic distance based weights.
        \item While this particular construction is chosen to emphasize the relationship to \cite{BELKIN20081289} and allow for the necessary analytical arguments, there exist alternative constructions that are geodesic choice-independent that generate the same sheaf Laplacian but emphasize functoriality. 
    \end{itemize}
\end{remark}

\subsection{Empirical Laplacians}\label{app:empirical_lap}
Analogously to \cite{BELKIN20081289}, we introduce two intermediary notions of Laplacian that interpolate between the Hilbert sheaf Laplacian and the Laplacian on a Hilbert bundle. For this section, fix a Hilbert bundle with compatible Fréchet connection $(\calM, \calE, \nabla)$ over a closed manifold $\calM$. 

\begin{definition} \label{def:uniform-distribution}
    Consider the unique normalized volume pseudo-form on $\calM$ (or the usual volume form if $\calM$ is orientable), denoted $d\mu$. Thus, $d\mu$ equips $\calM$ with a probability measure and we may refer to the resulting distribution as the \textbf{uniform distribution} on $\calM$.
\end{definition}

Henceforth, let $\mathcal{X}_n = \{x_1, \ldots, x_n\}$ denote the realization of an iid random sample drawn from the uniform distribution on $\mathcal{M}$. We then recall the following construction.

\begin{definition*} (Point-Cloud Extension of Sheaf Laplacian) \label{defn:point-cloud-extension} Let $(\mathcal{M}, \mathcal{E}, \nabla)$  be a Hilbert bundle and consider a sample $\calX_n \subset \mathcal{M}$. Then the corresponding Hilbert sheaf Laplacian $\Delta_{\mathcal{F}^t_n}$ may be extended to the \textbf{point-cloud Laplacian}  $\hat{\Delta}_{\calF^t_{n}}$, an operator on $L^2(\mathcal{M}, \mathcal{E})$ via \begin{equation}(\hat{\Delta}_{\mathcal{F}^t_n}S)(x):= \frac{1}{n} \sum_j e^{-d_{\calM}(x,x_j)^2/4t}\big ( S(x) - P_{x_j \to x}S(x_j)\big ) 
\end{equation}
\end{definition*}

\begin{remark}
    We make the following remarks about the point-cloud Laplacian.
    \begin{enumerate}
        \item The point-cloud Laplacian is the extension of the Hilbert sheaf Laplacian $\Delta_{\mathcal{F}^t_n} :C^0(G;\mathcal{F}_n^t) \to C^0(G;\mathcal{F}^t_{n})$ to an operator acting on sections of the Hilbert bundle $\calE \to \calM$, normalized by a factor of $1/n$. In particular, when evaluated at a sample point $x_i \in \mathcal{X}$, the point cloud Laplacian $\hat{\Delta}_{\mathcal{F}^t_n} S(x_i)$ is exactly the normalized $x_i$ component of the Hilbert sheaf Laplacian $\Delta_{\mathcal{F}^t_n}$ evaluated at the cochain $(S(x_1), \ldots, S(x_n))^T \in C^0(G;\mathcal{F}^t_{n})$.
        \item The point-cloud Laplacian is well defined for any section $S : \calM \to \calE$, regardless of regularity.
    \end{enumerate}
\end{remark}

\begin{definition}[Functional Approximation Laplacian] \label{def:functional-approx}
     For a section $S \in C^1(\calE)$, we define the \textbf{functional approximation} to the connection Laplacian
    $$(\hat{\Delta}^tS)(x) :=  \int_\mathcal{M} \big ( S(x) -P_{y \to x}S(y) \big ) \exp \left ( - \frac{d_\calM(x,y)^2}{4t}\right ) \, dy $$
    where $\int (-) \, dy$ denotes Bochner integration with respect to the canonical normalized volume pseudo-form on $\mathcal{M}$.
\end{definition}

\begin{remark}
    We make the following remarks about the functional approximation Laplacian.
    \begin{enumerate}
        \item The functional approximation Laplacian has no dependence on a sample of points from the underlying manifold. Instead, the functional approximation may be treated as the limiting operator where all points on the manifold have been sampled, and contribute uniformly via parallel transport. 
        \item The geometric data of the connection $\nabla$ impacts the functional approximation Laplacian through the parallel transport maps $P_{y \to x}: \mathcal{E}_y \to \mathcal{E}_x$, which links the fibers of $\calE$.
        \item Viewing the sample $\mathcal{X}_n = \{x_1, \ldots, x_n\} \subseteq \calM$ as having been drawn iid from the uniform probability distribution on $\calM$, the functional approximation Laplacian can be identified pointwise on a section $S$ as the expected value of the point cloud Laplacian. That is, for any $S \in C^1(\calM,\calE)$ and $x \in \calM$, we have:
        $$\frac{1}{\mathrm{vol}(\calM)}(\hat{\Delta}^tS)(x) = \mathbb{E}_{\mathcal{X}} \left [ (\hat{\Delta}_{\mathcal{F}^t_n} S)(x) \right ] $$
    \end{enumerate}
\end{remark}



\section{Proofs of Results}\label{sec:appendix-proofs}

\subsection{Auxiliary Lemmas for Theorem 1}

\begin{lemma}
    For $x = (x_1, \ldots, x_m) \in \R^m$, $k \in \mathbb{N}$, and $a,t > 0$, the following Gaussian identities hold:
    \begin{align}
        \frac{1}{(2 \pi at)^{m/2}} \int x_i \exp \left ( - \frac{ \|x\|^2}{2at} \right ) \, dx &= 0, \label{eq: gauss 1}\\
        \frac{1}{(2\pi at)^{m/2}} \int x_ix_j\exp \left ( - \frac{ \|x\|^2}{2at} \right ) \, dx &= at\delta_{ij}, \label{eq: gauss 2}\\
        \frac{1}{(2 \pi at)^{m/2}} \int \|x\|^{2k + 1} \exp \left ( - \frac{ \|x\|^2}{2at} \right ) \, dx &= O(t^{k + \frac{1}{2}}) \qquad \text{as $t \to 0$}, \label{eq: gauss 3}
    \end{align}
    where $\delta_{ij}$ is the Kronecker delta.
\end{lemma}

\begin{proof}
    Notice that $\frac{1}{(2 \pi at)^{m/2}}\exp \left ( - \frac{ \|x\|^2}{2at} \right )$ is the density function of a multivariate normal random variable $X \sim \mathcal{N}(0, a t I)$. Equations \eqref{eq: gauss 1} and \eqref{eq: gauss 2} are simply the values of the coordinate-mean $\mathbb{E}[X_i] = 0$ and covariance $\Cov(X_i,X_j) = at \delta_{ij}$. Finally, we may write $X = \sqrt{at} Z$, where $Z \sim \mathcal{N}(0,I)$ is a standard multivariate normal (zero-mean, uncorrelated, unit variance) in $m$ dimensions. We may write $\mathbb{E}[\|X\|^{2k + 1}] = (at)^{k + \frac{1}{2}} \cdot \mathbb{E}[||Z||^{2k+1}]$, which confirms \eqref{eq: gauss 3}. 
\end{proof}

\begin{remark}
    If instead of integrating over the entire domain $\R^m$, we integrate over a symmetric ball $B := B(R,0)$ centered at zero, we recover the following augmented Gaussian identities:
    \begin{align}
        \frac{1}{(2 \pi at)^{m/2}} \int_B x_i \exp \left ( - \frac{ \|x\|^2}{2at} \right ) \, dx &= 0, \label{eq: gauss 1 aug}\\
        \frac{1}{(2\pi at)^{m/2}} \int_B x_ix_j\exp \left ( - \frac{ \|x\|^2}{2at} \right ) \, dx &=  \bigg ( at + O(e^{-R^2/2at})  \bigg)\,\delta_{ij} \qquad \text{as $t \to 0$}, \label{eq: gauss 2 aug}\\
        \frac{1}{(2 \pi a t)^{m/2}} \int_B \|x\|^{2k + 1} \exp \left ( - \frac{ \|x\|^2}{2at} \right ) \, dx &= O(t^{k + \frac{1}{2}}) \qquad \text{as $t \to 0$}. \label{eq: gauss 3 aug}
    \end{align}
    The restriction to the ball $B$ leaves all odd-degree symmetries unchanged, and augments the even symmetries by a factor of the form $O(e^{-c/t})$, capturing the exponential decay on probability mass far away from the origin.
\end{remark}

\begin{lemma}(Banach Mean Value Theorem)\label{Banach-MVT}
Consider Banach spaces $\mathcal{B}_1, \mathcal{B}_2$ and some open $\mathcal{U} \subseteq \mathcal{B}_1,$ Then if $S: \mathcal{U} \rightarrow \mathcal{B}_2$ is Gateaux differentiable, then the mean value theorem holds in the sense that

$$
\|S(x)-S(y)\|_{\mathcal{B}_2} \leq\|x-y\|_{\mathcal{B}_1} \sup _{0 \leq t \leq 1}\|D f(t x+(1-t) y)\|
$$

whenever the convex hull $[x, y]$ lies in $\mathcal{U}$.
    
\end{lemma}

\begin{proof}
    The proof is a standard functional analysis argument but we recall it here for completeness. Let $\mathbb{L}$ be the one-dimensional subspace spanned by some fixed nonzero $u \in \mathcal{U}$. Consider $\varphi(cu)=c\|u \|$ as a continuous linear functional on $\mathbb{L}$ with norm $1$. By the strong form of the Hahn-Banach, as stated in \cite{ambrosetti1995primer}, for instance, we may extend this functional to the whole domain. Then note that $\varphi(u)=\|u \|$, so the result follows by considering $u = S(x)-S(y)$.
\end{proof}

\begin{remark}
    We recall that Fréchet differentiability in particular implies Gateaux differentiable, which will be sufficient for our purposes.
\end{remark}

\begin{lemma}(Banach Weak Law of Large Numbers) \label{Hilbert-LLN}  Let $\{X_j\}_{j \in \mathbb{N}}$ denote an independent identically distributed collection of random variables $X_j \in L^1(\Omega, B)$, where $\Omega := (\Omega, \Sigma, \mathbb{P})$ is a probability space and $B$ is a Banach space. Let $S_n := \sum_{j=1}^n X_j$ denote the partial sum of the first $n$ random variables. As $n \to \infty$, the normalized sequence $\frac{1}{n}S_n$ converges in probability to the mean $\mu := \mathbb{E}[X_j]$. That is for all $\epsilon > 0$, 
    $$\lim_{n \to \infty} \mathbb{P} \left [ \left | \frac{1}{n}S_n - \mu \right | > \epsilon \right ] = 0.$$
\end{lemma}

\begin{proof}
    See \citet{Pinelis1992} or \citet{LedouxTalagrand1991}. 
\end{proof}

\subsection{Key Lemmas for Theorem 1}

\begin{lemma}(Taylor Series for Hilbert Signals) \label{Taylor-Series}
 Let  $(\mathcal{M}, \mathcal{E}, \nabla)$ be a Hilbert bundle equipped with a Fréchet connection. For a given signal $S \in C^{n+1}(\mathcal{M},\mathcal{E})$, the space of $(n+1)$-times continuously Fréchet-differentiable sections, and $p \in \mathcal{M}$, consider any $q \in \mathcal{M}$ in a geodesic ball of $p$ and fix a length-minimizing curve $\gamma(t)$ from $p$ to $q$. Let $S^*(t) := P_{\gamma(t) \to p} S(\gamma(t))$ denote the parallel transport of $S$ from $\gamma(t)$ back to $p$ along $\gamma$. As $t \to 0$, we have that
   $$S^*(t) = \left [ \sum_{j=0}^{n} \frac{t^j}{j!} \big (\nabla_{\dot{\gamma}}^{(j)} S \big )(p) \right ] + O(t^{n+1})$$ 
   
\end{lemma}

\begin{proof}
We first establish for a section $V(t)$ of $\calE$ along $\gamma$, that:
$$\frac{d}{dt}\left(P_{\gamma(t)\to p} V(t)\right) = P_{\gamma(t)\to p}\left(\nabla_{\dot{\gamma}} V\right)(\gamma(t))$$
We compute directly from the definition:
\begin{align}
\frac{d}{dt}\left(P_{\gamma(t)\to p} V(t)\right)
&= \lim_{h \to 0} \frac{P_{\gamma(t+h) \to p} V(\gamma(t+h)) - P_{\gamma(t) \to p} V(\gamma(t))}{h} \\
&= \lim_{h \to 0} P_{\gamma(t) \to p} \left( \frac{P_{\gamma(t+h) \to \gamma(t)} V(\gamma(t+h)) - V(\gamma(t))}{h} \right) \\
&= P_{\gamma(t) \to p} \left(\nabla_{\dot{\gamma}} V\right)(\gamma(t))
\end{align}
where in the second line we used the composition law for parallel transport,
$P^{-1}_{\gamma(t) \to p} \circ P_{\gamma(t+h) \to p} = P_{\gamma(t+h) \to \gamma(t)}$.
which follows from uniqueness of the parallel transport ODE, together with the fact
 that $P_{\gamma(t) \to p}$ is a bounded linear operator and hence may be factored
 outside the limit.

Recall that $S^*(t) := P_{\gamma(t) \to p} S(\gamma(t))$ is a curve in the fixed
Hilbert space $\mathcal{E}_p$. 
Iteratively applying the previous derivative computation to $V(t) = \big ( \nabla_{\dot{\gamma}}^{(n)}S \big ) \big (\gamma(t) \big )$, where $(-)^{(n)}$ denotes $n$-fold composition, yields:
$$\frac{d^n}{dt^n} S^*(t) = P_{\gamma(t) \to p}\left(\nabla_{\dot{\gamma}}^{(n)} S\right)(\gamma(t))$$
Evaluating at $t = 0$, where $P_{p \to p} = \mathrm{id}$:
$$\frac{d^n}{dt^n}S^*\bigg|_{t=0} = \left(\nabla_{\dot{\gamma}}^{(n)} S\right)_p \, . $$
Finally, applying Taylor's theorem for Banach spaces \cite{Cartan1967} yields the desired asymptotic statement. 
\end{proof}


\begin{lemma}\label{lem: Hoeffding}
    Let $\hat{\Delta}_{\mathcal{F}^t_n}$ and $\hat{\Delta}^t$ denote the point-cloud and functional Laplacian operators with bandwidth $t$. We have the concentration inequality: 
     $$\mathbb{P} \left [ \frac{1}{t (4 \pi t)^{m/2}} \left | \hat{\Delta}_{\mathcal{F}^t_n}S(x)  - \hat{\Delta}^tS(x)  \right |  > \epsilon \right ] \leq 2 \exp \left ( - \frac{t^2 (4 \pi t)^{m} \epsilon^2 n }{2 K^2} \right )  $$ for some $K > 0$ which depends only on the choice of section $s$. 
     Consequently, we have the following limit in probability as $n \to \infty$:
    $$\frac{1}{t(4\pi t)^{m/2}} \hat{\Delta}_{\mathcal{F}^t_n} S(x) \xrightarrow{\mathbb{P}} \hat{\Delta}^t S(x) \, . $$    
\end{lemma}

\begin{proof}
    The point cloud Laplacian $\Delta_{\mathcal{F}^t_n}S(x)$ may be viewed as the sample average of $n$ iid Hilbert-space valued random variables: $$X_i :=  \exp \left ( - \frac{d_\calM(x,x_i)^2}{4t}\right ) S(x) - \exp \left ( - \frac{d_\calM(x,x_i)^2} {4t}\right ) P_{x_i \to x}S(x_i)$$
    Moreover, the functional approximation $\hat{\Delta}^t$  may be viewed as the expectation $\mathbb{E}\Delta_{\mathcal{F}^t_n}$ with respect to the uniform probability measure on $\calM$. The bundle $\calE \to \calM$ has separable fibers, so the results of \cite{Pinelis1992} apply. We may recover a Hoeffding inequality:
    $$\mathbb{P} \left [ \frac{1}{\delta} \left | \hat{\Delta}_{\mathcal{F}^t_n} S(x)  - \mathbb{E}\hat{\Delta}_{\mathcal{F}^t_n} S(x)  \right |  > \epsilon \right ] \leq 2 \exp \left ( - \frac{(\delta \epsilon)^2n }{2 K^2} \right )  $$
    where $K$ is the maximum norm of the section $S$ over the compact manifold $\mathcal{M}$. Setting $\delta = t(4 \pi t)^{m/2}$ and identifying $\mathbb{E}\hat{\Delta}_{\mathcal{F}^t_n} S(x) = \hat{\Delta}^t S(x)$ yields the desired concentration inequality. Convergence in probability follows immediately. 
\end{proof}

\begin{lemma}\label{lem: ball reduction}
    Let $(\calM, \calE, \nabla)$ be a Hilbert bundle on a closed Riemannian manifold equipped with a Fréchet connection. Let $B \subseteq \calM$ be open, and $p \in B$. Fix a section $S \in C^{3}(\calM, \calE)$. For each $x \in \calM$, let $F(x) := P_{x \to p} S(x)$ denote the parallel transport of $S(x)$ along the designated geodesic connecting $x$ to $p$. For any real $a \in \mathbb{R}^{>0}$, the following asymptotic bound holds as $t \to 0$:
    $$\left | \int_\calM e^{- \frac{ \|y - p\|}{4t}} F(y) \, d \mu(y) - \int_B e^{- \frac{ \|y - p\|}{4t}} F(y) \, d \mu(y) \right | = o(t^a).$$
\end{lemma}

\begin{proof}
    We note this is a modified version of Lemma 4.1 of \cite{BELKIN20081289}. Let $d := \inf_{x \not \in B} \|p - x\|$, $K := \sup_{x \in \calM} \|S(x)\|$, and $M := \mu(\calM \setminus B)$, where $\mu$ is the canonical measure with respect to the volume pseudo-form. Since $p$ is compact and $\calM \setminus B$ is closed, the infimum distance $d > 0$. Recalling that $P_{x \to p}: \calE_x \to \calE_p$ is unitary, and hence $\|F(x)\| = \|S(x)\|$ for all $x \in \calM$, we may bound:
    \begin{align*}
        \left | \int_\calM e^{- \frac{ \|y - p\|}{4t}} F(y) \, d \mu(y) - \int_B e^{- \frac{ \|y - p\|}{4t}} F(y) \, d \mu(y) \right | & \leq \int_{\calM \setminus B} \left \| e^{- \frac{ \|y - p\|}{4t}} F(y) \right \| \, d \mu(y) \\
        & \leq \int_{\calM \setminus B}e^{- \frac{ \|y - p\|}{4t}} \left \|  S(y) \right \| \, d \mu(y) \\
        & \leq MK \exp \left (-\frac{d}{4t} \right) \\
        & = o(t^a).
    \end{align*}
\end{proof}

\begin{lemma} \label{lem: functional approximation almost sure convergence}
    Let  $(\mathcal{M}, \mathcal{E}, \nabla)$  be a Hilbert bundle equipped with a compatible Fréchet connection, with associated Laplacian $\Delta_\nabla$, and functional approximation $\hat{\Delta}^t$ with bandwidth $t$ . Fix a section $S \in  C^3(\mathcal{M},\mathcal{E})$.  For any $x \in \mathcal{M}$ as the bandwidth $t \to 0$, we have pointwise convergence:
 $$ 
\lim _{t \rightarrow 0} \frac{1}{t\left(4 \pi t \right)^{\frac{m}{2}}} \hat{\Delta}^{t} S(x)=\frac{1}{\operatorname{vol}(\mathcal{M})} \Delta_\nabla S(x) \, . 
$$
\end{lemma}

\begin{proof}
 Let $\gamma_t := \frac{1}{t} \frac{1}{(4 \pi t)^{m/2}}$, and consider the scaled functional approximation $\gamma_t \hat{\Delta}^t$ which acts on a section $S$ at point $p$ by:
 $$\left ( \gamma_t \hat{\Delta}^t S \right ) (p) = \gamma_t \int_\calM e^{ - \frac{d_\calM(p,x)^2}{4t}}( S(x) - P_{x \to p}S(p) )\, d \mu(x) \, . $$
 Let $B \subseteq \calM$ denote a sufficiently small ball containing $p$. By Lemma \ref{lem: ball reduction},
 $$ \lim_{t \to 0} \left [ \left ( \gamma_t \hat{\Delta}^tS \right ) (p) \right ] = \lim_{t \to 0}  \gamma_t \int_B e^{ - \frac{d_\calM(p,x)^2}{4t}}( S(p) - P_{p \to x}S(x) )\, d \mu(x).$$
 Parameterize $B$ via geodesic coordinates such that $p = 0$. Let $F(x) := P_{x \to p} S(x)$. Let $\tilde{S}: \mathbb{R}^k \to \mathcal{H}$ and $\tilde{B} \subseteq \mathbb{R}^k$ denote the section $S$ and ball $B$ in coordinates. In these coordinates, we may write:
$$\lim_{t \to 0} \left [ \left ( \gamma_t \hat{\Delta}^t_{\calF_n} S \right ) (p) \right ] = \lim_{t \to 0} \frac{1}{\mathrm{vol}(\calM)} \gamma_t \int_{\tilde{B}} e^{- \frac{d_\calM(\exp_p(x) , 0)^2}{4t}} (\tilde{F}(0) - \tilde{F}(x)) \sqrt{\det(g_{ij})} \, d x \, .$$
In geodesic coordinates, since the closed manifold $\calM$ has bounded Ricci curvature, the metric tensor has an asymptotic expansion  given by (as in e.g. \cite{docarmo1992})
$$\det(g_{ij}) = 1 + O(\|x\|^2) \, . $$
This approximation and the identification $d_\calM(\exp_p(x),p) = \|x\|$ in coordinates allows us to express:
$$\lim_{t \to 0} \left [ \left ( \gamma_t \hat{\Delta}^t_{\calF_n} S \right ) (p) \right ] = \lim_{t \to 0} \frac{1}{\mathrm{vol}(\calM)} \gamma_t \int_{\tilde{B}} e^{- \frac{\|x\|^2}{4t}} (\tilde{F}(0) - \tilde{F}(x)) \big (1 + O(\|x\|^2) \big) \, d x \, . $$

    Let $\mathbb{L}_t := \frac{1}{\mathrm{vol}(\calM)} \gamma_t \int_{\tilde{B}} \exp \left (- \frac{\|x\|^2}{4t} \right ) (\tilde{F}(0) - \tilde{F}(x)) \big (1 + O(\|x\|^2) \big) \, d x$. This expression splits as $\mathbb{L}_t = A_t + B_t$, where:
    \begin{align*}
        A_t &:= \frac{1}{\mathrm{vol}(\calM)} \gamma_t \int_{\tilde{B}} \exp \left (- \frac{\|x\|^2}{4t} \right ) (\tilde{F}(0) - \tilde{F}(x)) \, dx \, ,  \\
        B_t & := \frac{1}{\mathrm{vol}(\calM)} \gamma_t \int_{\tilde{B}} \exp \left (- \frac{\|x\|^2}{4t} \right ) (\tilde{F}(0) - \tilde{F}(x))\left [ O(\|x\|^2) \right] \, dx \, .
    \end{align*}
    We first analyze the limiting behavior of $A_t$ as $t \to 0$. Within our geodesic coordinates centered at $p$, we may further work with a local synchronous frame of $\mathcal{E}$ along these coordinates $x=(x_1,\dots,x_k)$ such that it is parallel along all radial geodesics. Consequently, within this frame, ordinary derivatives of $\tilde{F}$
coincide with covariant derivatives of $S$ at the basepoint:
\begin{align*}
    \partial_i\tilde{F}(p) & =\nabla_{e_i}S(p) \\
    \partial_i\partial_j\tilde{F}(p) & =\nabla_{e_i}\nabla_{e_j}S(p)
+ \mathsf{curv}_{ij} \, ,
\end{align*}
where $\mathsf{curv}_{ij} := - \frac{1}{2} R^{\calE}(e_i, e_j)S(p)$ is half the bundle curvature of $\calE$ arising from the connection $\nabla$. Hence by Lemma \ref{Taylor-Series}, the Taylor expansion of $\tilde{F}$ at $p$ is
$$\tilde{F}(x) = \tilde{F}(p) + \sum_i x_i\nabla_{e_i}S(p) +\frac12\sum_{i,j}x_ix_j \big(\nabla_{e_i}\nabla_{e_j}S(p)+\mathsf{curv}_{ij}\big)+O(\|x\|^3),$$
with $\mathsf{curv}_{ij}=-\mathsf{curv}_{ji}$.

Using the augmented Gaussian identities \eqref{eq: gauss 1 aug}, \eqref{eq: gauss 2 aug}, and \eqref{eq: gauss 3 aug}, we may compute:
\begin{align*}
I_t & := \frac{1}{(4 \pi t)^{m/2}} \int_{\tilde{B}} (\tilde{F}(p) - \tilde{F}(x)) \exp\left ( -\frac{\|x\|^2}{4t}\right)\, dx\\
& = -\frac{1}{2} \sum_{i,j} \left [ \big(\nabla_{e_i}\nabla_{e_j}S(p)+\mathsf{curv}_{ij} \big ) \left (2t + O \left ( e^{-c/t} \right )  \right )   \delta_{ij} \right ]+O \left (t^{3/2} \right ) \\ 
&= -t\sum_i\nabla_{e_i}\nabla_{e_i}S(p)
- t\sum_i\mathsf{curv}_{ii}
+O \left (t^{3/2} \right )
\end{align*}
where $\delta_{ij}$ is the Kronecker delta. Since $\mathsf{curv}_{ij}$ is antisymmetric, $\mathsf{curv}_{ii}=0$, hence
$$I_t = -t\sum_i\nabla_{e_i}\nabla_{e_i}S(p) +O(t^{3/2}).$$
Inserting into the definition of $A_t$ yields
$$ \lim_{t\to 0}A_t = -\frac{1}{\operatorname{vol}(\mathcal M)} \sum_i\nabla_{e_i}\nabla_{e_i}S(p).
$$
Thus we recover the connection Laplacian by Lemma \ref{lemma:coord-expression}
. 

To analyze the quantity $B_t$, we first observe that inside of $\tilde{B}$, the parallel transport map $P_{\exp_p(x) \to p}$ varies smoothly in $x$, with bounded derivatives. Since the section $S$ is $C^3$, the mean value theorem (Lemma \ref{Banach-MVT}) ensures there is a $K>0$ such that $\|F(x)-F(p)\| \leq K \|x - p\|$ for all $x \in \tilde{B}$.
Utilizing this Lipschitz bound and the augmented Gaussian identities, we may compute:
\begin{align*}
    \|B_t\| & \leq \frac{1}{\mathrm{vol}(\calM)} \gamma_t \int_{\tilde{B}} \exp \left (- \frac{\|x\|^2}{4t} \right )  \|\tilde{F}(0) - \tilde{F}(x) \| \left [ O(\|x\|^2) \right] \, dx \\
    & \leq \frac{K / \mathrm{vol}(\calM)}{t} \frac{1}{(4 \pi t)^{m/2}} \int_{\tilde{B}} \, \|x\|^3 \exp \left (- \frac{\|x\|^2}{4t} \right )  \\
    & = \frac{1}{t}O \left (t^{3/2} \right )  \\
    & = O \left (\sqrt{t} \right ) \, . 
\end{align*}
Hence $B_t \to 0$ as $t \to 0$. Combining with the analysis of $A_t$ yields:
\begin{align*}
    \lim_{t \to 0} \left [ \left ( \gamma_t \hat{\Delta}^t_{\calF_n} S \right ) (p) \right ] & = \lim_{t \to 0} [A_t+B_t] \\
    & = \Delta_{\nabla} S(p)
\end{align*}
\end{proof}

\begin{lemma}[Variance asymptotics]\label{lem: variance asymptotics}
    Let  $(\mathcal{M}, \mathcal{E}, \nabla)$  be a Hilbert bundle equipped with a compatible Fréchet connection, with associated Laplacian $\Delta_\nabla$. Fix a section $S \in  C^4(\mathcal{M},\mathcal{E})$.  Consider a random sample of $n$-points with respect to the normalized volume pseudo-form, $\mathcal{X}_n = \{x_1, x_2, \cdots, x_n\} \subset \mathcal{M}$. For each bandwidth $t$, let $\gamma := (t (4 \pi t)^{m/2})^{-1}$,  where $m := \dim(\calM)$. Define an error term:
    $$ \mathcal{R}_n(S) := \gamma \hat{\Delta}_{\calF^t_n}S - \frac{1}{\mathrm{vol}(\calM)} \Delta_\nabla S \, . $$
    There is a constant $C_{\mathrm{var}} > 0$, which depends on the section $S$, such that the following asymptotic estimate holds as $t \to 0^+$:
    $$\E_{\mathcal{X}_n} \left [ \left \| \mathcal{R}_n(S) - \mathbb{E}_{\mathcal{X}_n} \mathcal{R}_n(S)\right \|_{L^2}^2 \right ] \leq \frac{C_{\mathrm{var}}}{n t^{2 + \frac{m}{2}}} \, .  $$
\end{lemma}

\begin{proof}
    Let $Y_j$ be the random section given by
    $$Y_j(x) := \gamma (S(x) - P_{x_j \to x} S(x_j)) \exp \left ( - \frac{d_\calM(x,x_j)^2}{4t} \right ).$$
    Note that $Y_j$ is stochastic only through the sample point $x_j$, and that $Y_i,Y_j$ are independent when $i \neq j$. It is straightforward to verify by Funbini's theorem that 
    \begin{align*}
        \E_{\mathcal{X}_n}\left [ \left \| \mathcal{R}_n(S) - \E_{\mathcal{X}_n} \mathcal{R}_n(S)  \right \|^2_{L^2} \right ] & = \frac{1}{n} \left ( \E_{x_1} \left [ \|Y_1\|^2_{L^2} \right ] - \| \E_{x_1} Y_1 \|^2_{L^2}  \right ) \\
        & \leq \frac{1}{n} \E_{x_1} \left [ \|Y_1\|^2_{L^2} \right ] \, . 
    \end{align*}
    Set $K := \max_{x \in \calM} \| S(x) \|_{\calE_x}$.
    By Fubini, we may exchange the order of integration and find
    $$\E_{x_1}\left [ \|Y_1\|^2_{L^2} \right ]  =   \int_{\calM}  \E_{x_1} \left [ \left \| Y_j(x)\right \|_{\calE_x}^2 \right] dx  $$
    We may compute:
    \begin{align*}
        \E_{x_1}\left [ \|Y_1(x)\|^2_{\calE_x} \right ] & = \frac{1}{\mathrm{vol}(\calM)}\int_{\calM } \left \| \gamma (S(x) - P_{x_1 \to x} S(x_1)) \exp \left (- \frac{d_\calM(x,x_1)^2}{4t} \right ) \right \|_{\calE_x}^2 dx_1 \\
        & \leq \frac{4 \gamma^2 K^2}{\mathrm{vol}(\calM)}\int_\calM \exp \left (- \frac{d_\calM(x,x_1)^2}{2t} \right ) dx_1 \, . 
    \end{align*}
    By standard Gaussian identities, the remaining Gaussian integral is $O(t^{m/2})$ as $t \to0^+$, with constant independent of $x$. Recalling the definition of $\gamma$ in terms of $t$, we recover that
    $$\E_{x_1}\left [ \|Y_1(x)\|^2_{\calE_x} \right ] = O(t^{-(2 + m)})  \cdot O(t^{m/2}) = O(t^{- (2 + m/2)}) . $$
    Since the constants are all independent of $x$, we may integrate over $\calM$ and find that 
    \begin{align*}
        \E_{x_1}\left [ \|Y_1\|^2_{L^2} \right ]  &=  O(t^{- (2 + m/2)})
    \end{align*}
    as well. The result immediately follows. 
\end{proof}

\begin{lemma}[Bias asymptotics]\label{lem: bias asymptotics}
Let  $(\mathcal{M}, \mathcal{E}, \nabla)$  be a Hilbert bundle equipped with a compatible Fréchet connection, with associated Laplacian $\Delta_\nabla$. Fix a section $S \in  C^4(\mathcal{M},\mathcal{E})$.  Consider a random sample of $n$-points with respect to the normalized volume pseudo-form, $\mathcal{X}_n = \{x_1, x_2, \cdots, x_n\} \subset \mathcal{M}$. For each bandwidth $t$, let $\gamma := (t (4 \pi t)^{m/2})^{-1}$,  where $m := \dim(\calM)$. Define an error term:
    $$ \mathcal{R}_n(S) := \gamma \hat{\Delta}_{\calF_n}^tS - \frac{1}{\mathrm{vol}(\calM)} \Delta_\nabla S$$
    There is a constant $C_{\mathrm{bias}} > 0$, which depends on the section $S$, such that the following asymptotic estimate holds as $t \to 0^+$:
    $$\left [ \left \| \E_{\mathcal{X}_n}  \mathcal{R}_n(S)\right \|_{L^2}^2 \right ] \leq \big ( C_{\mathrm{bias}}t \big )^2 \, .  $$
\end{lemma}

\begin{proof}
    This follows from essentially repeating the analysis in the proof of Lemma \ref{lem: functional approximation almost sure convergence} using the fourth order Taylor expansion $F(y) = P_{y \to x} S(y)$ in terms of the covariant derivative. In particular, after accounting for the fourth-order Taylor remainder, we find that:
    $$ \E_{\mathcal{X}_n } [ \gamma \hat{\Delta}_{\calF^t_n} (x)] = \frac{1}{\mathrm{vol}(\calM)} \Delta_\nabla S(x) + O(t) \, . $$
    The result follows. 
\end{proof}

\subsection{Proof of Theorem \ref{thm: main-thm}}
\subsubsection{Proof of Theorem \ref{thm: main-thm}{\ref{thm: convergence in probability}}}\label{app:proofth1}

\begin{theorem*}Let  $(\mathcal{M}, \mathcal{E}, \nabla)$  be a Hilbert bundle equipped with a compatible Fréchet connection, with associated Laplacian $\Delta_\nabla$. Fix a section $S \in  C^3(\mathcal{M},\mathcal{E})$.  Consider a random sample of $n$-points with respect to the normalized volume form, $\mathcal{X}_n = \{x_1, x_2, \cdots, x_n\} \subset \mathcal{M}$. Let $\mathcal{F}^t_{\mathcal{X}_n}$ be the associated Hilbert cellular sheaf with bandwidth $t$ and associated Point cloud Laplacian $\hat{\Delta}_{\mathcal{F}^t_n}$. Then, we have that in probability, for any $x \in \calM$,
 $$ 
\lim _{n \rightarrow \infty} \frac{1}{t_n\left(4 \pi t_n\right)^{\frac{m}{2}}} \hat{\Delta}_{\mathcal{F}^{t_n}_n} S(x)=\frac{1}{\operatorname{vol}(\mathcal{M})} \Delta_\nabla S(x)
$$ 
with bandwidth $t_n = n^{-\frac{1}{m + 2 + \alpha}}$, $\alpha > 0$.
\end{theorem*}

 \begin{proof}
Let $\gamma_n := \frac{1}{t_n} \frac{1}{(4 \pi t_n)^{m/2}}$, and consider the scaled functional approximation $\gamma_n \hat{\Delta}^{t_n}$ which acts on a section $S$ at point $p$ by:
 $$\left ( \gamma_{n} \hat{\Delta}^{t_n} S \right ) (p) = \gamma_n \int_\calM e^{ - \frac{d_\calM(p,x)^2}{4t_n}}(P_{x \to p} S(x) - S(p) )\, d \mu(x) \, . $$
We may bound:
\begin{align*}
    \mathbb{P} \left [ \left \| \frac{1}{\gamma_n} \hat{\Delta}_{\mathcal{F}^{t_n}_n} S(x) - \frac{1}{\mathrm{vol}(\mathcal{M})} \Delta_\nabla S(x) \right \| > 2 \epsilon \right ] \leq& \quad \: \mathbb{P} \left [ \frac{1}{\gamma_n}\left \|  \hat{\Delta}_{\mathcal{F}^{t_n}_n} S(x) - \hat{\Delta}^{t_n} S(x) \right \| >  \epsilon \right ] \\
    & + \mathbb{P} \left [ \left \|  \frac{1}{\gamma_n}\hat{\Delta}^{t_n} S(x) - \frac{1}{\mathrm{vol}(\mathcal{M})} \Delta_\nabla S(x) \right \| >  \epsilon \right ] 
\end{align*}
as $n \to \infty$ we have $t_n \to 0$. Hence the second quantity on the right hand side goes to zero by Lemma \ref{lem: functional approximation almost sure convergence}. On the other hand, the first quantity on the right hand side can be bound by the concentration inequality of Lemma \ref{lem: Hoeffding}, yielding:
$$\mathbb{P} \left [ \frac{1}{\gamma_n}\left \|  \hat{\Delta}_{\mathcal{F}^{t_n}_n} S(x) - \hat{\Delta}^{t_n} S(x) \right \| >  \epsilon \right ]  \leq 2 \exp \left ( \frac{t_n^{2+m} n (4 \pi t)^m \epsilon^2}{2K^2} \right ) \, ,$$
where $K$ is a constant depending on the section $S$. Since $nt_n^{2 + m} \to \infty$ as $n \to \infty$, the concentration upper bound goes to zero as $n \to \infty$ as well. This completes the proof of the of the main theorem. 
\end{proof}
 
\subsubsection{Proof of theorem
\ref{thm: main-thm}{\ref{thm: L2-upgrade}}}
\begin{theorem*}
    Let  $(\mathcal{M}, \mathcal{E}, \nabla)$  be a Hilbert bundle equipped with a compatible Fréchet connection, with associated Laplacian $\Delta_\nabla$. Fix a section $S \in  C^4(\mathcal{M},\mathcal{E})$.  Consider a random sample of $n$-points with respect to the normalized volume form, $\mathcal{X}_n = \{x_1, x_2, \cdots, x_n\} \subset \mathcal{M}$. Let $\mathcal{F}^t_{\mathcal{X}_n}$ be the associated Hilbert cellular sheaf with bandwidth $t$ and associated Point cloud Laplacian $\hat{\Delta}_{\mathcal{F}^t_n}$. Then, we have the following convergence in expectation:
 $$ 
\lim _{n \rightarrow \infty} \E_{\mathcal{X}}\left [  \left \| \frac{1}{t_n\left(4 \pi t_n\right)^{\frac{m}{2}}} \hat{\Delta}_{\mathcal{F}^{t_n}_n} S(x) -\frac{1}{\operatorname{vol}(\mathcal{M})} \Delta_\nabla S(x) \right \|_{L^2}^2 \right ] = 0
$$ 
with bandwidth $t_n = n^{-\frac{1}{m + 2 + \alpha}}$, $\alpha > 0$.
\end{theorem*} 

\begin{proof}
    Let $\gamma_n := (t_n (4 \pi t_n)^{m/2})^{-1} $. Define an error term:
    $$ \mathcal{R}_n(S) := \gamma \hat{\Delta}_{\calF^t_n}S - \frac{1}{\mathrm{vol}(\calM)} \Delta_\nabla S \, . $$
    We may decompose the error into a bias and variance term as:
    $$ \mathbb{E}_{\mathcal{X}_n}\left[\|\mathcal{R}_n\|_{L^2}^2\right ] =\mathbb{E}_{\mathcal{X}_n} \left [ \left \| \mathcal{R}_n(S) - \mathbb{E}_{\mathcal{X}_n} \mathcal{R}_n(S) \right \|_{L^2}^2 \right ]  + \left \| \mathbb{E}_{\mathcal{X}_n} \mathcal{R}_n(S) \right \|_{L^2}^2 \, . $$
    By the asymptotic results of Lemmas \ref{lem: variance asymptotics} and \ref{lem: bias asymptotics}, there are positive constants $C_{\mathrm{var}}$ and $C_{\mathrm{bias}}$ such that
    $$\mathbb{E}_{\mathcal{X}_n}\left[\|\mathcal{R}_n\|_{L^2}^2\right ]  \leq \frac{C_{\mathrm{var}}}{nt_n^{2 + \frac{m}{2}}} + C_{\mathrm{bias}}^2 t_n^2 $$
    Since $t_n=n^{-1/(m+2+\alpha)}$, we have $t_n^2 \to 0$ and $n t_n^{2 + \frac{m}{2}} \to \infty$ as $n \to \infty$. 
    Hence both terms vanish in limit, proving the desired convergence. 
\end{proof}

\subsection{Key Lemmas for Theorem \ref{thm: diagonal finite rank convergence}} \label{app:proofs:lemmas-finite-rank}
\begin{definition}
    Let $\calE$ be a smooth Hilbert bundle over a manifold $\calM$. A \textbf{finite rank approximating sequence} for $\calE$ is a sequence of smooth sub-bundles $\{\calE_\subbundleidx \}_{\subbundleidx \geq 1}$ with the following properties:
    \begin{enumerate}
        \item For each $\subbundleidx$, $\calE_\subbundleidx$ has finite rank;
        \item For each $\subbundleidx$, the bundle $\calE_\subbundleidx$ is a sub-bundle of $\calE_{\subbundleidx+1}$;
        \item For each $x \in \calM$, we have that $\mathrm{cl} \left ( \mathrm{span} \left ( \bigcup_\subbundleidx (\calE_\subbundleidx)_x\right ) \right ) = \calE_x $. 
    \end{enumerate}
\end{definition}

\begin{lemma} \label{app_lem:approximating sequence}
    Let $\calE$ be a smooth Hilbert bundle with infinite-dimensional fibers over a compact manifold $\calM$. A finite rank approximating sequence $\{\calE_\subbundleidx\}_{\subbundleidx}$ exists.
\end{lemma}

\begin{proof}
    By Kuiper's Theorem \citep{Kuiper1965}, the unitary group of the typical fiber $\mathcal{H}$ is contractible, implying that there exists an isomorphism of $\calE$ with bundle $\calM \times \mathcal{H}$ at the level of purely topological bundles. Now note that every Hilbert bundle $\calM \times \calH$ admits a finite rank approximating sequence, by considering a Hilbert space basis $\{e_1, e_2, \ldots \}$ for $\calH$, and defining $\calH_n := \mathrm{span}(e_1, \ldots, e_n)$. The sequence $\{\calM \times \calH_\subbundleidx\}_\subbundleidx $ can then be seen to be a finite rank approximating sequence. Furthermore, because the base space is a finite-dimensional manifold, this topological trivialization can be upgraded to a smooth global trivialization \cite{M_ller_2009}. Let $\Phi : \calE \to \calM \times \mathcal{H}$ be such a smooth isomorphism. Thus, the finite rank approximating sequence $\{\calM \times \calH_\subbundleidx\}_\subbundleidx $ pulls back to a finite rank approximating sequence on $\calE$ by $\calE_\subbundleidx := \Phi^{-1}( \calM \times \calH_\subbundleidx)$, as desired.
\end{proof}

\begin{lemma} \label{app_lem:projection-operators}
    Let $(\calE, \calM, \nabla)$ be a smooth infinite-dimensional Hilbert bundle over a compact manifold $\calM$ equipped with a compatible connection $\nabla$. Let $\{\calE_\subbundleidx \}_{\subbundleidx \geq 1}$ be a finite rank approximating sequence for $\calE$. The data of $\nabla$ induces a compatible connection $\nabla_\subbundleidx := \Pi_\subbundleidx \nabla$ on $\calE_\subbundleidx$, where $\Pi_\subbundleidx : \calE \to \calE_\subbundleidx$ denotes the fiber-wise orthogonal projection onto $\calE_\subbundleidx$. 
\end{lemma}

\begin{proof}
    This follows immediately from the fact that $\nabla$ is compatible and orthogonal projections are self-adjoint. 
\end{proof}
\begin{remark}
    Let  $(\calM, \calE, \nabla)$ be an  infinite-dimensional Hilbert bundle equipped with a compatible connection. By the previous lemmas, we may always find a finite rank approximating sequence $(\calE_\subbundleidx, \calM, \nabla_\subbundleidx)$, each with compatible connection. These compatible connections induce connection Laplacians $\Delta_{\nabla_\subbundleidx}$ on each sub-bundle. Moreover, Theorem \ref{thm: main-thm}{\ref{thm: L2-upgrade}} applies to each Laplacian $\Delta_{\nabla_\subbundleidx}$. 
\end{remark}

We restate and prove Proposition~\ref{prop:signal_discretization}.
\begin{proposition*}
    Let  $(\mathcal{M},\mathcal{E},\nabla)$ be a Hilbert bundle, with strictly infinite-dimensional generic Hilbert-space fiber $\mathcal{H}$. Fix an orthogonal basis $\mathcal{B}=\{e_1, e_2, ...\} $ of $\mathcal{H}$ and let $\mathcal{H}_d:= \mathrm{span}(e_1, e_2, ..., e_d) $. Then there exists a smooth map of bundles 
    \begin{equation}
   \Pi_d: \mathcal{E} \to \mathcal{E}_d
   \end{equation}
   where $\mathcal{E}_d $ is a $d$-dimensional vector bundle with generic fiber $\mathcal{H}_d$ and at each $x\in \mathcal{M}$, $\left .\Pi_d\right |_{\calE_x}: \mathcal{E}_x \to \mathcal{E}_{d,x}$ recovers the usual orthogonal projection map. 
\end{proposition*}
\begin{proof}
    One we fix a basis $\mathcal{B}$, the conclusion follows from the proof of Lemma \ref{app_lem:approximating sequence} and an application of Lemma \ref{app_lem:projection-operators}.
\end{proof}
\begin{definition}
     Let $(\calM, \calE, \nabla)$ be a smooth Hilbert bundle over a closed manifold $\calM$ of dimension $m$ equipped with a compatible connection $\nabla$. Fix a section $S \in C^4(\calM, \calE)$. Let $\{\calE_\subbundleidx\}_{\subbundleidx}$ be a finite rank approximating sequence for $\calE$ with induced connections $\nabla_\subbundleidx$, connection Laplacians $\Delta_{\nabla_\subbundleidx}$, and bandwidth $t$ point cloud Laplacians $\hat{\Delta}_{n,\subbundleidx}^t $ associated to an iid sampling $\mathcal{X} = \{x_1, x_2, \ldots\}$. Let $\Pi_\subbundleidx : \calE \to \calE_\subbundleidx$ denote the fiber-wise orthogonal projection map onto $\calE_\subbundleidx$. The \textbf{discretization error} $\mathtt{D}(n,\subbundleidx)$ and the \textbf{continuous geometry error} $\mathtt{E}(\subbundleidx)$ are the quantities:
    \begin{align*}
        \ttD(n,\subbundleidx) & := \left \| \gamma_n \hat{\Delta}^{t_n}_{n,\subbundleidx} (\Pi_\subbundleidx S) - \frac{1}{\mathrm{vol}(\calM)} \Delta_{\nabla_\subbundleidx} (\Pi_\subbundleidx S)  \right \|_{L^2}^2  \\
        \ttE(\subbundleidx) & := \left \| \frac{1}{\mathrm{vol}(\calM)} \left ( \Delta_{\nabla_\subbundleidx}( \Pi_\subbundleidx S) - \Delta_\nabla S\right ) \right \|_{L^2}^2 \, . 
    \end{align*}
\end{definition}

\begin{remark}
    The discretization error $\ttD(n,\subbundleidx)$ captures the error introduced by approximating the connection Laplacian $\Delta_{\nabla_\subbundleidx}$ by a point-cloud Laplacian on $n$ points. The continuous geometry error $\ttE(\subbundleidx)$ is a deterministic quantity that captures the error introduced by moving to the sub-bundle $\calE_\subbundleidx$.
\end{remark}

\begin{lemma}\label{lem: continuous geometry error goes to zero}
    The continuous geometry error $\ttE(\subbundleidx)$ converges to zero as $\subbundleidx \to \infty$.
\end{lemma}

\begin{proof}
    Without loss of generality, by pushing through the global trivialization of Kuiper's theorem \citep{Kuiper1965}, we may assume without loss of generality that $\calE = \calH \times \calM$ as topological bundles. 
    First, note that the orthogonal projection $\Pi_\subbundleidx$ commutes with the Fr\'echet derivative $D$ on sections, in the sense that:
    $$D ( \Pi_\subbundleidx S) = \Pi_\subbundleidx (DS) . $$
    By Proposition \ref{pr: local form of connection}, for any vector field $X$, we may write:
    $$\nabla_XS = DS(X) + A(X)S$$
    where $A(X)$ is a globally bounded linear operator acting on the fiber $\calH$. We may now compute:
    \begin{align*}
        (\nabla_\subbundleidx)_X(\Pi_\subbundleidx S) & = \Pi_\subbundleidx \nabla_X(\Pi_\subbundleidx S)\\
        & = \Pi_\subbundleidx \left ( D \Pi_\subbundleidx S(X) + A(X)\Pi_\subbundleidx S \right ) \\
        & = \Pi_\subbundleidx DS(X) + \Pi_\subbundleidx A(X)\Pi_\subbundleidx S) 
    \end{align*}
    Since $\Pi_\subbundleidx \to \text{id}$ in the strong operator topology and $DS(X) \in \calH$, we have that $\Pi_\subbundleidx DS(X) \to DS(X)$ as $\subbundleidx \to \infty$. Similarly since $A(X): \calH \to \calH$ is bounded, we find that for each $x \in \calM$, we have $\Pi_\subbundleidx A(X)\Pi_\subbundleidx S(x) \to A(X)S(x)$. Therefore $\nabla_\subbundleidx (\Pi_\subbundleidx S)(x) \to \nabla S(x)$ . By a similar argument, we may conclude that for a pair of vector fields $X,Y$, that
    $(\nabla_\subbundleidx)_X(\nabla_\subbundleidx)_Y(\Pi_\subbundleidx S)(x) \to \nabla_X \nabla_Y S(x).$
    Hence using the coordinate form of the connection Laplacian of  Lemma \ref{lemma:coord-expression}, we recover that:
    $$\lim_{\subbundleidx \to \infty} \left [ \Delta_{\nabla_\subbundleidx } S(x) \right ] = \Delta_\nabla S(x)$$
    for all $x \in \calM$. 

    We now upgrade to this statement to $L^2$ convergence by the dominated convergence theorem. If there is a global bound $K$ such that $\| \Delta_{\nabla_\subbundleidx}( \Pi_\subbundleidx S)(x) - \Delta_\nabla S(x) \|_{\calH} \leq K$ for all $x \in \calM$, we may apply the dominated convergence theorem and conclude that $E(\subbundleidx) \to 0$. To find such a $K$, we first observe that $\Delta_\nabla S$ is a continuous section, and hence $\|\Delta_\nabla S(x)\|_{\calH} \leq K_1$ is bounded on the compact manifold $\calM$. Next, for a pair of vector fields $X,Y$, we may use the fact that $\|\Pi_\subbundleidx \|_{\text{op}} = 1$ and the representation $\nabla = D + A$ to find a $\subbundleidx$-independent bound on  $\left \| \big( (\nabla_{\subbundleidx})_X (\nabla_{\subbundleidx})_Y (\Pi_\subbundleidx S) \big)(x) \right \|_{\calH}$. The local coordinate form the connection Laplacian $\Delta_{\nabla_\subbundleidx}$ again allows us to conclude that there is a bound  $\|\Delta_{\nabla_\subbundleidx} (\Pi_\subbundleidx S)(x)\|_\calH \leq K_2$ for all $x \in \calM$ and $\subbundleidx \geq 1$. The triangle inequality finally allows us to bound:
    $$\| \Delta_{\nabla_\subbundleidx}( \Pi_\subbundleidx S)(x) - \Delta_\nabla S(x) \|_{\calH} \leq  K_1 + K_2 \, . $$
    This completes the proof.
\end{proof}

\subsection{Proof of Theorem \ref{thm: diagonal finite rank convergence}}\label{app:proofth2}
\begin{theorem*}
    Let $(\calE, \calM, \nabla)$ be an  infinite-dimensional Hilbert bundle over a closed manifold $\calM$ of dimension $m$ equipped with a compatible connection $\nabla$. Fix a section $S \in C^4(\calM, \calE)$. Let $\{\calE_\subbundleidx\}_{\subbundleidx}$ be a finite rank approximating sequence for $\calE$ with induced connections $\nabla_\subbundleidx$, connection Laplacians $\Delta_{\nabla_\subbundleidx}$, and bandwidth $t$ point cloud Laplacians $\hat{\Delta}_{\calF^t_{n,\subbundleidx}}$ associated to an iid sampling $\mathcal{X} = \{x_1, x_2, \ldots\}$. Let $\Pi_\subbundleidx : \calE \to \calE_\subbundleidx$ denote the fiber-wise orthogonal projection map onto $\calE_\subbundleidx$. There exists a deterministic increasing sequence $\subbundleidx_n$, depending on the section $S$, such that
    $$ \lim_{n \to \infty} \E_{\mathcal{X}} \left [ 
    \left \| \frac{1}{t_n (4 \pi t_n)^{m/2}} \hat{\Delta}_{\calF_{n,\subbundleidx_n}^{t_n}} (\Pi_{\subbundleidx_n}S) - \frac{1}{\mathrm{vol}(\calM)} \Delta_\nabla S \right \|_{L^2}^2 
    \right ] = 0$$
    with bandwidth $t_n = n^{-\frac{1}{m + 2 + \alpha}}$, $\alpha > 0$.
\end{theorem*}

\begin{proof}
    Let $\gamma_n := \frac{1}{t_n (4 \pi t_n)^{m/2}}$. We may easily bound the expected global error in terms of the continuous geometry error and the expected discretization error:
    $$\E_{\mathcal{X}} \left [ 
    \left \| \gamma_n \hat{\Delta}_{\calF_{n, \subbundleidx_n}}^{t_n} (\Pi_{\subbundleidx_n}S) - \frac{1}{\mathrm{vol}(\calM)} \Delta_\nabla S \right \|_{L^2}^2 
    \right ] \leq 2\E_{\mathcal{X}_n}[\ttD(n,\subbundleidx)] + 2 \ttE(\subbundleidx) \, . $$
    By applying Theorem \ref{thm: main-thm} to the bundle $\calE_\subbundleidx$, the expected discretization error $\E_{\mathcal{X}_n}[\ttD(n,\subbundleidx)] \to 0$ as $n \to \infty$. On the other hand, Lemma \ref{lem: continuous geometry error goes to zero} ensures that $\ttE(\subbundleidx) \to 0$ as $\subbundleidx \to \infty$. 
    
{ \color{black}    To construct the diagonal sequence, first choose an increasing sequence $p_i$ such that $\ttE(\subbundleidx) < \frac{1}{i}$ for all $\subbundleidx \geq p_i$. Next, choose an increasing sequence $N_i$ such that $\E_{\mathcal{X}_n}[\ttD(n,p_i)] < \frac{1}{i}$ for each $n \geq N_i$. For each $n \geq 1$, set $\phi(n) := \max \{i \: \mid \: N_i \leq n \}$. Observe that both $\ttE(p_{\phi(n)})  <\frac{1}{\phi(n)}$ and $\E_{\mathcal{X}_n}[\ttD(n,p_{\phi(n)})] < \frac{1}{\phi(n)}$. Therefore setting $\subbundleidx_n := p_{\phi(n)}$ yields a diagonal sequence with the property that
    $$\E_{\mathcal{X}} \left [ 
    \left \| \gamma_n \hat{\Delta}_{\calF^{t_n}_{n,\subbundleidx_n}} (\Pi_{\subbundleidx_n}S) - \frac{1}{\mathrm{vol}(\calM)} \Delta_\nabla S \right \|_{L^2}^2 
    \right ] < \frac{4}{\phi(n)} \,,$$
    which tends to zero since $\phi(n) \to \infty$. This diagonal subsequence is deterministic as both the continuous geometry error and the \emph{expected} discretization error are deterministic.
}
\end{proof}

\subsection{Key Lemmas for Corollary \ref{cor: convergence in architecture}}

\begin{lemma}\label{lem: filter level MSE convergence}
    Let $(\calE, \calM, \nabla)$ be a smooth infinite-dimensional Hilbert bundle over a closed manifold $\calM$ of dimension $m$ equipped with a compatible connection $\nabla$. Fix a section $S \in C^4(\calM, \calE)$. Let $\{\calE_\subbundleidx\}_{\subbundleidx}$ be a finite rank approximating sequence for $\calE$ with induced connections $\nabla_\subbundleidx$, connection Laplacians $\Delta_{\nabla_\subbundleidx}$, and bandwidth $t$ point cloud Laplacians $\hat{\Delta}_{\calF^t_{n,\subbundleidx}}$ associated to an iid sampling $\mathcal{X} = \{x_1, x_2, \ldots\}$. Let $\Pi_\subbundleidx : \calE \to \calE_\subbundleidx$ denote the fiber-wise orthogonal projection map onto $\calE_\subbundleidx$. Let $\{\subbundleidx_n\}$ be the diagonal sequence induced by Theorem \ref{thm: diagonal finite rank convergence}. Let $\tilde{\Delta}_n := \frac{1}{t_n (4 \pi t_n)^{m/2}}\hat{\Delta}_{\calF^{t_n}_{n,\subbundleidx_n}} \Pi_{\subbundleidx_n}$ with bandwidth $t_n = n^{-\frac{1}{m + 2 + \alpha}}$, $\alpha > 0$. Similarly, let $\tilde{\Delta} := \frac{1}{\mathrm{vol}(\calM)} \Delta_\nabla$. Finally, let $g: \mathbb{R} \to \mathbb{R}$ be a bounded continuous function. 

    Under the Borel functional calculus (Appendix \ref{sec:borel-calc}), we have MSE convergence:
    $$ \E \left [ \left \| g(\tilde{\Delta}_{n})S - g(\tilde{\Delta})S \right \|_{L^2}^2 \right ] \to 0  \, . $$
\end{lemma}

\begin{proof}
    Observe that $\tilde{\Delta}$ and each $\tilde{\Delta}_n$ are  self-adjoint unbounded operators on $L^2(\calE;\calM)$. We begin by showing 
    there is a common core $\mathcal{D} \subseteq L^2(\calE;\calM)$ and a subsequence $n_i$ for which $\tilde{\Delta}_{n_i} S \to \tilde{\Delta} S$ in $L^2(\calE;\calM)$. 

    Take $\mathcal{D}$ to be any countable dense subset of $C^\infty(\calM,\calE)$. Since $\calE$ has separable fibers, such a countable dense subset necessarily exists. Moreover, $\tilde{\Delta}$ and $\tilde{\Delta}_n$ are all defined on $\mathcal{D}$. This $\mathcal{D}$ shall be our common core.

    Treat $\tilde{\Delta}_n S$ as an $L^2(\calM,\calE)$-valued random variable. For each $S \in \mathcal{D}$, Theorem \ref{thm: diagonal finite rank convergence} ensures that $\tilde{\Delta}_n S \to \tilde{\Delta} S$ in mean square error. It immediately follows that $\tilde{\Delta}_n S \xrightarrow{\mathbb{P}_{\mathcal{X}}} \tilde{\Delta} S$ in probability with respect to the measure from which the sampling $\mathcal{X}$ is drawn. 

    Enumerate $\mathcal{D} = \{S_1, S_2, \ldots\}$. Since convergence in probability implies almost sure convergence on a subsequence, we may inductively construct a doubly-indexed sequence of indices $N^a_b$ such that the following properties hold:
    \begin{enumerate}
        \item For each $a$, the sequence $\{N_b^{a+1}\}_b$ is a subsequence of $\{N_b^{a}\}_b$;
        \item Along the sequence $\{N^a_b\}_b$, we have almost sure convergence $\tilde{\Delta}_{N^a_b} S_a\xrightarrow{\mathrm{a.s.}} \tilde{\Delta} S_a$ as $b \to \infty$. 
    \end{enumerate}
    Take the diagonal sequence $n_i := N^i_i$. Along $n_i$, we have that $\tilde{\Delta}_{n_i} S\xrightarrow{\mathrm{a.s.}} \tilde{\Delta} S$ as $i \to \infty$ for all $S \in \mathcal{D}$. Now applying Theorems VIII.25(a) and VIII.20(b) of \cite{reed1972functional}, we may conclude that:
    $$\| g(\tilde{\Delta}_{n_i})S - g(\tilde{\Delta})S \|_{L^2} \xrightarrow{\mathrm{a.s.}} 0$$
    almost surely for each $S$. 

    Notice that the previous argument can not only be applied to the sequence $n = \{1,2,3, ...\}$, but also along any subsequence of $\{n\}_n$. Since any subsequence of $\{n\}_n$ therefore has an almost surely convergent sub-subsequence, we may conclude that for the original sequence $n$, for each section $S$ (not necessarily in $\mathcal{D}$), we have convergence in probability:
    $$\| g(\tilde{\Delta}_{n})S - g(\tilde{\Delta})S \|_{L^2} \xrightarrow{\mathbb{P}_{\mathcal{X}}} 0$$
    in probability with respect to the sampling $\mathcal{X}$ as $n \to \infty$. 

    Finally, by the spectral calculus, since $g : \R \to \R$ is bounded by some $B$, we may bound $\|g(\tilde{\Delta}_{n})S \|_{L^2} \leq B \|S\|_{L^2}$. Similarly, $|g(\tilde{\Delta})S \|_{L^2} \leq B \|S\|_{L^2}$. It follows that for each $n$,
    $$\| g(\tilde{\Delta}_{n})S - g(\tilde{\Delta})S \|_{L^2}^2 \leq 4 B^2$$
    for all $n$. Hence the dominated convergence theorem admits an MSE upgrade to the desired conclusion:
    $$ \E \left [ \left \| g(\tilde{\Delta}_{n})S - g(\tilde{\Delta})S \right \|_{L^2}^2 \right ] \to 0  \, . $$
\end{proof}

\subsection{Proof of Corollary \ref{cor: convergence in architecture}}
\begin{corollary*}\label{cor:architecture-convergence-on-sections}
    Under the hypotheses of Theorem \ref{thm: diagonal finite rank convergence}, let $\{\subbundleidx_n\}_n$ be the constructed deterministic diagonal sequence, and $L \in \mathbb{N}$. Let $\sigma$ be a fiber-wise nonlinearity that is $C_\sigma$-Lipschitz in the corresponding fiber norms. For bounded continuous filters $g^0, \ldots, g^{L-1} \in \mathcal{W}$, consider the continuous and sampled architectures:
\begin{align*}
    S^{\ell + 1} & := \sigma \left (g^\ell \left (\tilde{\Delta} \right ) S^\ell \right ) \\
    S_n^{\ell + 1} & :=  \sigma \left (g^\ell \left (\tilde{\Delta}_n \right ) S_n^\ell \right ) \, , 
\end{align*}
with initializations $S^0 := S$ and $S_n^0 := \Pi_{\subbundleidx_n} S$, where $\tilde{\Delta} := \frac{1}{\mathrm{vol}(\calM)} \Delta_{\nabla}$ and $\tilde{\Delta}_n := \frac{1}{t_n (4 \pi t_n)^{m/2} }\hat{\Delta}_{\calF^{t_n}_{n,\subbundleidx_n}}$.

Then, the output of the discrete architecture converges in mean square to the output of the continuous architecture:
$$\lim_{n \to \infty} \mathbb{E} \left[ \left\| S_n^L - S^L \right\|_{L^2(\calM; \calE)}^2 \right] = 0 \, .$$
\end{corollary*}
\begin{proof}
We proceed by induction on the layers. Let $e_{n,\ell}$ and $\delta_{n,\ell}$ denote the signal error and spectral filter error at layer $\ell$ respectively:
    \begin{align*}
        e_{n, \ell} & := \left \| S_n^{\ell} - S^\ell \right \|_{L^2} \\
        \delta_{n, \ell} & := \left \| g_\ell(\tilde{\Delta}_n)(\Pi_{\subbundleidx_n} S^\ell) - g_\ell(\tilde{\Delta}) S^\ell \right \|_{L^2} \, . 
    \end{align*}
    Let $M_\ell := \sup_{x \in \mathbb{R}}\|g_\ell(x)\|$. By the Lipschitz continuity of the nonlinearity $\sigma$, the triangle inequality yields the pathwise recursive bound:
$$e_{n, \ell + 1} \leq C_\sigma (M_\ell e_{n, \ell} + \delta_{n, \ell})\, . $$
Iterating this inequality over $L$ layers expands to:
$$e_{n,L} \leq \left ( \prod_{r=0}^{L-1} C_\sigma M_r \right ) e_{n,0} + \sum_{q=0}^{L-1} \left ( \prod_{r = q+1}^{L-1} C_\sigma M_r\right ) C_\sigma \delta_{n,q} \, . $$

By Theorem \ref{thm: diagonal finite rank convergence}, the initialized signal error $e_{n,0} \to 0$ in mean square. Furthermore, by Lemma \ref{lem: filter level MSE convergence}, $\delta_{n,q} \to 0$ in mean square as well for each $q$. 

Taking the expectation of the squared recursive bound and applying the Cauchy-Schwarz inequality to the finite sum isolates the individual mean square limits.  Hence, as $n \to 0$, we have that the total error satisfies:
$$\lim_{n \to \infty} \E[e_{n,L}^2] = 0 \, . $$
In particular, we have that as sampling density goes to infinity, in MSE ,
\begin{equation*}
    \Omega( \mathcal{F}_{n,\subbundleidx_n}, \hat{\Delta}_{\mathcal{F}_{n,\subbundleidx_n}},\mathcal{W}, \sigma) \to \Omega(\mathcal{E}, \Delta_{\nabla}, \mathcal{W}, \sigma)\, . 
\end{equation*}

\end{proof}

\subsection{Proof of Corollary \ref{transferability}}
\begin{corollary*}
    Let $(\calE, \calM, \nabla)$ be a smooth Hilbert bundle over a closed manifold $\calM$ of dimension $m$ equipped with a compatible connection $\nabla$. Fix a section $S \in C^4(\calM, \calE)$. Let $\{\calE_\subbundleidx\}_{\subbundleidx}$ be a finite rank approximating sequence for $\calE$ with induced connections $\nabla_\subbundleidx$, and connection Laplacians $\Delta_{\nabla_\subbundleidx}$.  Let $\Pi_\subbundleidx : \calE \to \calE_\subbundleidx$ denote the fiber-wise orthogonal projection map onto $\calE_\subbundleidx$. Let $\mathcal{X} = \{x_1, x_2, \ldots\}$ and $\mathcal{Y} = \{y_1, y_2, \ldots\}$ be a pair of independent iid samplings of points on the manifold. Denote the bandwidth $t$ point cloud Laplacians associated to these distinct samplings by $\hat{\Delta}_{\calF^t_{\mathcal{X}_n,d}}$ and 
    $\hat{\Delta}_{\calF^t_{\mathcal{Y}_n,d}}$ respectively. Let $\{\subbundleidx_n\}$ be a diagonal sequence such that the conclusion of Theorem \ref{thm: diagonal finite rank convergence} holds for both samplings. Let $\tilde{\Delta}_n^{\mathcal{X}} := \frac{1}{t_n (4 \pi t_n)^{m/2}}\hat{\Delta}_{\calF^{t_n}_{\mathcal{X}_{n,\subbundleidx_n}}} \Pi_{\subbundleidx_n}$ with bandwidth $t_n = n^{-\frac{1}{m + 2 + \alpha}}$, $\alpha > 0$, and similar for $\tilde{\Delta}_n^{\mathcal{Y}}$. 

    Let $L \in \mathbb{N}$ be a network depth, and $\sigma$ be a fiber-wise nonlinearity that is $C_\sigma$-Lipschitz in the corresponding fiber norms. For bounded continuous filters $g_0, \ldots, g_{L-1} \in \mathcal{W}$, consider the continuous and sampled architectures:
\begin{align*}
    S^{\ell + 1} & := \sigma\left (g_\ell \left (\tilde{\Delta} \right ) S^\ell \right )\\
    S_n^{\mathcal{X}, \ell + 1} & :=  \sigma \left (g_\ell \left (\tilde{\Delta}^{\mathcal{X}}_n \right ) S_n^{\mathcal{X}, \ell} \right ) \\
    S_n^{\mathcal{Y}, \ell + 1}  &:=  \sigma \left (g_\ell \left (\tilde{\Delta}_n^{\mathcal{Y}} \right ) S_n^{\mathcal{Y},\ell} \right ) \, , 
\end{align*}
with initializations $S^{0}:= S$ and $S_n^{\mathcal{X},0},S_n^{\mathcal{Y},0}  := \Pi_{\subbundleidx_n} S$.
    Under these hypotheses and notation, one may obtain a MSE convergence result:
    $$ \lim_{n \to \infty} \E \left [ \left \| S_n^{\mathcal{X},L} - S_n^{\mathcal{Y},L} \right \|_{L^2}^2 \right ] = 0 \, . $$
    Further, one may derive a quantitative bound for the $L^2$ disagreement $\| S_n^{\mathcal{X},L} - S_n^{{\mathcal{Y}},L} \|_{L^2}$ in terms of sample-indpendent quantities.
\end{corollary*}

\begin{proof}
    One may bound:
    $$ \ \left \| S_n^{\mathcal{X},L} - S_n^{\mathcal{Y},L} \right \|_{L^2}^2 \leq \ 2 \left \| S_n^{\mathcal{X},L} - S^{L} \right \|_{L^2}^2 + 2 \left \| S_n^{\mathcal{Y},L} - S^{L} \right \|_{L^2}^2$$
    Applying Corollary \ref{cor: convergence in architecture} to each sampling separately yields the MSE convergence.      In particular, we have that
\begin{equation*}
    \lim_{n \to \infty} \E_{\mathcal{X},\mathcal{Y{}}} \left [ \left \| \Omega(\calF^{t_n}_{\mathcal{X}_n,d_n}, \hat{\Delta}_{\calF^{t_n}_{\mathcal{X}_n,d_n}}, \mathcal{W}, \sigma) - \Omega(\calF^{t_n}_{\mathcal{Y}_n,d_n}, \hat{\Delta}_{\calF^{t_n}_{\mathcal{Y}_n,d_n}},\mathcal{W}, \sigma) \right \|_{L^2}^2 \right ] = 0 \, . 
\end{equation*}

    To derive a quantitative bound, introduce the per-sampling signal error and spectral filter error at level $\ell$ by:
    \begin{align*}
        e^{(-)}_{n, \ell} & := \left \| S_n^{(-),\ell} - S^\ell \right \|_{L^2} \\
        \delta^{(-)}_{n, \ell} & := \left \| h_\ell(\tilde{\Delta}^{(-)}_n)(\Pi_{\subbundleidx_n} S^\ell) - h_\ell(\tilde{\Delta}) S^\ell \right \|_{L^2} \, . 
    \end{align*}
    Apply the triangle inequality and the layer-wise recursive bounds of the proof of Corollary \ref{cor: convergence in architecture} to establish:
    $$ \left \|S_n^{\mathcal{X},L}-S_n^{\mathcal{Y},L} \right \|_{L^2}
    \le
    \left(\prod_{r=0}^{L-1} C_\sigma M_r\right)(e^{\mathcal{X}}_{n,0}+e^{\mathcal{Y}}_{n,0})
    +
    \sum_{q=0}^{L-1}
    \left(\prod_{r=q+1}^{L-1} C_\sigma M_r\right) C_\sigma(\delta^{\mathcal{X}}_{n,q}+\delta^{\mathcal{Y}}_{n,q}) \, . $$
    The level-zero signal error is sampling independent, and bounded above by $\|S\|_{L^2}$. On the other hand, we may further apply the Borel functional calculus to bound each spectral filter error by:
    \begin{align*}
        \delta_{n,q}^{(-)} & \leq 2 M_q \| S^  q \|_{L^2} \, .
    \end{align*}
    We hence conclude a sample-independent bound:
    $$ \left \|S_n^{\mathcal{X},L}-S_n^{\mathcal{Y},L} \right \|_{L^2}
    \le
    2\left(\prod_{r=0}^{L-1} C_\sigma M_r\right)\|S \|_{L^2}
    +
    2\sum_{q=0}^{L-1}
    \left(\prod_{r=q+1}^{L-1} C_\sigma M_r\right) C_\sigma M \|S^q\|_{L^2}\, . $$
\end{proof}
\end{document}